\documentclass{article}
\usepackage[utf8]{inputenc}
\usepackage[T1]{fontenc}
\usepackage{times}
\usepackage{helvet}
\usepackage{courier}

\usepackage{amsmath}
\usepackage{amssymb}
\usepackage{amsfonts}

\usepackage{booktabs}
\usepackage{multirow}
\usepackage{makecell}
\usepackage{array}
\usepackage{tabularx}

\usepackage{graphicx}
\usepackage{xcolor}
\usepackage{caption}
\usepackage{algpseudocode}
\usepackage{algorithm}

\usepackage{enumitem}
\usepackage{tcolorbox}
\usepackage{float}
\usepackage{nicefrac}
\usepackage{microtype}
\usepackage{url}
\usepackage[colorlinks=true, linkcolor=blue, citecolor=blue, urlcolor=blue]{hyperref}

\usepackage{tikz}

\usepackage[numbers,square]{natbib}
\usepackage{pifont}
\usepackage{arxiv}

\usepackage{fancyvrb}  
\usepackage{tcolorbox} 
\usepackage{amsthm}

\theoremstyle{definition}  
\newtheorem{theorem}{Theorem}[section]      % 
\newtheorem{proposition}[theorem]{Proposition}  % 
\newtheorem{lemma}[theorem]{Lemma}

\theoremstyle{remark}  % 

\DefineVerbatimEnvironment{circomcode}{Verbatim}{
  fontsize=\small,
  frame=single,
  framesep=6pt,
  rulecolor=\color{gray!30},
  fillcolor=\color{codebg},
  baselinestretch=1.2
}

\usepackage{listings}
\usepackage{xcolor}

\title{NeuroPareto: Calibrated Acquisition for Costly Many-Goal Search in Vast Parameter Spaces}

\author{
    Rong Fu \\
    Independent Researcher \\
    Corresponding author \and
    Chunlei Meng \\
    Independent Researcher \and
    Haoyu Zhao \\
    Independent Researcher \and
    Kun Liu \\
    Independent Researcher \and
    Simon James Fong \\
    Independent Researcher
}

\renewcommand{\headeright}{}
\renewcommand{\undertitle}{}
\renewcommand{\shorttitle}{NeuroPareto}
\hypersetup{
pdftitle={NeuroPareto: Calibrated Acquisition for Costly Many-Goal Search in Vast Parameter Spaces},
pdfsubject={cs.NE, cs.LG},
pdfauthor={Rong Fu, Chunlei Meng, Haoyu Zhao, Kun Liu, Simon James Fong},
pdfkeywords={Surrogate-Enhanced Multi-Objective Search, Deep Probabilistic Surrogate Frameworks, Epistemic Uncertainty Disentanglement, History-Conditioned Acquisition Learning, High-Dimensional Black-Box Optimization},
}

\begin{document}
\maketitle

\begin{abstract}
The pursuit of optimal trade-offs in high-dimensional search spaces under stringent computational constraints poses a fundamental challenge for contemporary multi-objective optimization. We develop \textbf{NeuroPareto}, a cohesive architecture that integrates rank-centric filtering, uncertainty disentanglement, and history-conditioned acquisition strategies to navigate complex objective landscapes. A calibrated Bayesian classifier estimates epistemic uncertainty across non-domination tiers, enabling rapid generation of high-quality candidates with minimal evaluation cost. Deep Gaussian Process surrogates further separate predictive uncertainty into reducible and irreducible components, providing refined predictive means and risk-aware signals for downstream selection. A lightweight acquisition network, trained online from historical hypervolume improvements, guides expensive evaluations toward regions balancing convergence and diversity. With hierarchical screening and amortized surrogate updates, the method maintains accuracy while keeping computational overhead low. Experiments on DTLZ and ZDT suites and a subsurface energy extraction task show that NeuroPareto consistently outperforms classifier-enhanced and surrogate-assisted baselines in Pareto proximity and hypervolume. 
\end{abstract}

% keywords can be removed

\section{Introduction}
Real-world engineering and scientific design problems often require simultaneously optimizing multiple conflicting objectives while each objective evaluation is costly in computation or experiment. Traditional population-based multi-objective search methods remain effective when many objective evaluations are affordable, but they lose practicality in constrained-budget settings. In response, surrogate-assisted strategies replace expensive evaluations with predictions from learned models to steer search. These approaches have matured into diverse families that emphasize either surrogate fidelity, acquisition design, or evolutionary mechanisms for preserving diversity and handling constraints. Representative work in surrogate-assisted multi-objective optimization and related algorithmic refinements illustrate both the promise and the limitations of these directions. \citep{pang2022expensive,chen2023rank,zhou2024evolutionary}

A parallel strand of research emphasizes learned components and classifier-guided screening as a way to reduce true evaluations. Classifier-based nondomination prediction and learned offspring generators have shown that many poor candidates can be filtered cheaply before committing costly evaluations. Hybrid pipelines that combine classifier screening with surrogate-managed infill extend the practical reach of surrogate assistance but expose two weaknesses. First, classifiers and surrogate models must provide well calibrated uncertainty so that screening does not systematically bias search away from true optima. Second, learned acquisition or history-aware policies must generalize from sparse optimization traces. Recent works on calibration, ensemble-style uncertainty approximations, and complexity-aware surrogate management address pieces of this picture while leaving open how to combine them effectively. \citep{yuan2021expensive,durasov2021masksembles,li2024surrogate}

A further challenge arises from high dimensionality. Dimensionality reduction, random embedding, and representation learning relieve modeling burdens by focusing capacity on relevant subspaces. At the same time, scalable surrogate constructions and low-rank approximations make it feasible to reason about uncertainty under larger decision spaces. Despite these advances, co-design between classifier screening, uncertainty-aware surrogates, and adaptive acquisition remains uncommon. Methods that improve a single component in isolation often fail to capture the cross-component interactions that determine overall performance when evaluation budgets are severely limited. \citep{shmakov2023rtdk,li2025expensive,liu2025two}
Although named NeuroPareto, the proposed method is fundamentally a hybrid Bayesian optimization algorithm; the suffix ‘EA’ emphasizes the use of rank-based population operators rather than indicating a traditional evolutionary algorithm.

To address these gaps we propose an integrated framework that explicitly propagates and exploits uncertainty across three tightly coupled modules. Compared to prior high-dimensional surrogate-assisted methods, our framework is the first to co-design classifier-driven screening, deep uncertainty-decomposing surrogates, and a history-aware acquisition learner as a single optimization loop rather than treating components independently. This holistic perspective improves robustness in scarce-data regimes but introduces additional implementation complexity and a larger hyper-parameter surface. 

Our primary contributions are as follows. First, we propose a calibrated Bayesian classifier for nondomination rank prediction, providing actionable uncertainty estimates that guide rank-conditioned offspring generation and low-cost screening. Second, we develop an uncertainty-decomposing surrogate pipeline based on Deep Gaussian Processes with sparse variational updates and complexity-aware approximations, enabling tractable modeling in large decision spaces and local refinement in high-uncertainty regions. Third, we introduce a compact, history-aware acquisition learner that predicts expected hypervolume improvement and diversity contribution from online optimization traces, supporting adaptive prioritization of evaluations. Together, these components form a coherent optimization loop that improves sample efficiency and Pareto robustness under scarce data regimes.

\section{Related Work}
\label{sec:related}

\subsection{Multi-objective Bayesian Optimization}
Multi-objective Bayesian optimization (MOBO) offers a principled approach for expensive black-box problems with competing objectives. Classic hypervolume-aware and information-theoretic acquisition rules perform well in low dimensions but face scalability challenges \citep{daulton2021parallel,pang2022expensive}. Recent work improves EHVI and its multi-point variants for better tractability and stability \citep{deng2025expected}, while generative and diffusion-based Pareto set models capture complex distributions under tight budgets \citep{li2025expensive}. For high-dimensional or medium-scale problems, hybrid surrogates and inverse-distance or radial-basis proxies trade fidelity for tractability \citep{li2024inverse,liu2025two}.

\subsection{Learned acquisition, meta-learning and surrogate policies}
Recent work replaces hand-designed acquisition heuristics with learned policies that generalize across tasks. Meta-BO and task-conditioned surrogate frameworks, including transformer-based and amortized learners, jointly learn models and acquisition strategies to improve sample efficiency on related task families \citep{shmakov2023rtdk,buathong2023bayesian}. Structure-aware approaches that exploit partial evaluations or cost models further reduce evaluation expense \citep{song2022monte}. Learned acquisition is promising for multi-objective settings but requires robust uncertainty estimates and compact feature representations to remain effective in high-dimensional spaces.

\subsection{Uncertainty quantification and calibration}
Accurate uncertainty quantification underpins effective exploration. Ensembles and Bayesian approximations remain the standard tools for decomposing epistemic and aleatoric uncertainty \citep{gal2016dropout,abulawi2025bayesian}. Practical methods aim to reduce the cost of ensemble-like behaviours while preserving calibration: Masksembles interpolates between MC-Dropout and deep ensembles, improving reliability at modest overhead \citep{durasov2021masksembles}. Recent proposals focus on improving MC-Dropout calibration and integrating optimizer-based tuning for better uncertainty estimates in deep models \citep{asgharnezhad2025enhancing}. In dynamic or parametric surrogate settings, explicitly modelling environment or observable parameters helps separate reducible from irreducible uncertainty and improves Pareto-front tracking under change \citep{temmerman2025handling}.

\subsection{Neural--GP hybrids and scalable surrogate design}
Neural–GP hybrids and deep Gaussian process families aim to combine neural expressivity with Gaussian-process uncertainty quantification. Advances in amortized variational inference, sparse inducing schemes and randomized-feature approximations permit significantly larger training sets and deeper latent stacks while keeping uncertainty estimates meaningful \citep{meng2024amortized,rochussen2025sparse}. For expensive many-variable tasks, two-level model management and low-rank factorization strategies have been proposed to control surrogate costs without discarding uncertainty-aware decision rules \citep{liu2025two,li2024surrogate}.

\subsection{Classifier-assisted selection and learning-aided evolutionary methods}
Classifier-guided selection, dominance prediction, and learned offspring generation reduce expensive evaluations by screening or biasing the evolutionary pipeline. Dominance predictors, rank-conditioned generation, and local surrogate infill accelerate convergence and preserve diversity when evaluations are scarce \citep{yuan2021expensive,chen2023rank,guo2024classifier}. Surrogate-assisted neuroevolution and linear-genetic-programming-based surrogate schemes show that surrogate paradigms extend beyond canonical EMO benchmarks and reduce computational cost \citep{stapleton2024neurolgp}. Feature-selection and permutation-based multi-objective strategies highlight the need for specialized high-dimensional operators \citep{espinosa2025permutation}. Algorithms for dynamic or constrained landscapes demonstrate benefits of multiple populations and explicit infeasible-region handling when feasible sets are irregular \citep{jiang2025dual}.

\subsection{Positioning of NeuroPareto}
NeuroPareto is positioned at the intersection of classifier-assisted selection, uncertainty-aware surrogate modeling, and learned acquisition. It integrates calibrated rank prediction (to cheaply guide candidate generation), expressive Deep GP surrogates (to decompose predictive uncertainty), and a history-aware acquisition learner (to adapt evaluation priorities based on observed hypervolume gains). By combining these elements and adopting complexity-aware approximations, NeuroPareto aims to improve robustness and scalability on high-dimensional, expensive multi-objective tasks compared with methods that optimize individual components in isolation \citep{chen2023rank,zhou2024evolutionary,li2025expensive}. 

\textbf{Relation to multi-agent and retrieval-augmented optimization.} 
While \citet{xu2026rcbsf} and \citet{xu2026selfcorrectingrag} address similar challenges of constrained optimization and adaptive decision-making, their approaches differ fundamentally in uncertainty handling. 
RCBSF achieves equilibrium through Stackelberg game-theoretic constraints without explicit surrogate uncertainty, whereas Self-Correcting RAG relies on test-time NLI validation rather than predictive variance. 
NeuroPareto bridges these paradigms by combining \textit{explicit} Deep GP uncertainty decomposition with \textit{learned} acquisition adaptation, offering a principled probabilistic foundation for expensive many-objective search that remains effective when evaluation budgets are severely limited. This aligns with broader trends in sustainable on-device intelligence, where energy-efficient retrieval-augmented systems demonstrate that strict resource constraints can be met without sacrificing decision quality \cite{cheng2026toward}.

\begin{figure*}[t]
  \centering
  % Adjust width as needed for ICML layout (usually 0.95-1.0\textwidth for figure*)
  \includegraphics[width=0.872\textwidth]{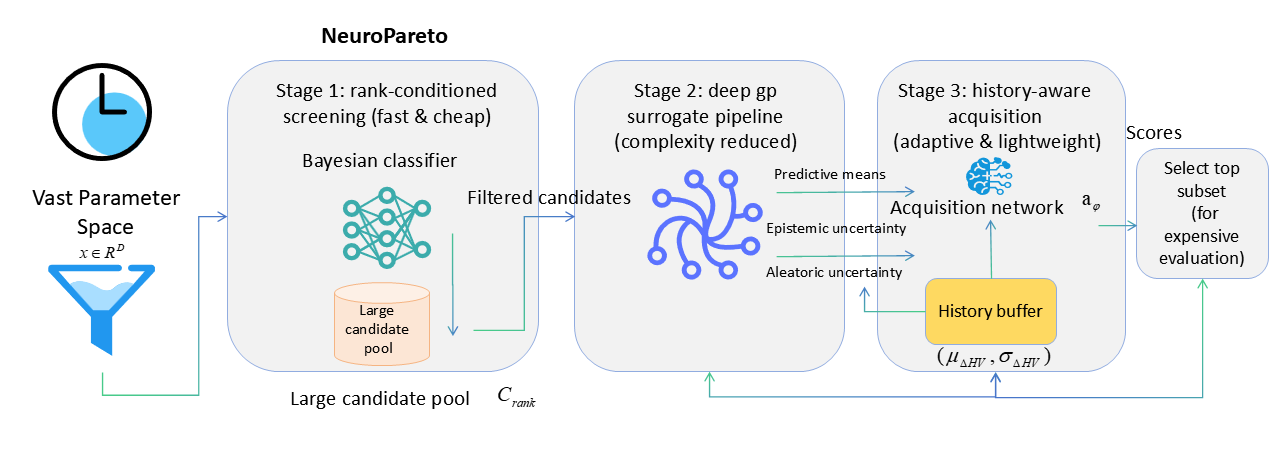} 
  \caption{Overview of the \textbf{NeuroPareto} framework for high-dimensional, budget-constrained multi-objective optimization. The pipeline consists of three synergistic modules: The \textbf{Bayesian Rank Classifier} ($g_{\theta}$) screens a massive candidate pool $\mathcal{C}{\mathrm{rank}}$ using temperature-calibrated softmax $\tilde{p}k$ and adaptive MC dropout to quantify classifier epistemic uncertainty $u{\mathrm{ep}}^{\mathrm{clf}}$. The \textbf{Complexity-Reduced Deep GP} pipeline processes the filtered subset, utilizing sparse variational inference with inducing variables $\mathbf{u}$ and Randomized Fourier Features (RFF) to output predictive objective means $\hat{\mathbf{f}}(\mathbf{x})$ alongside decomposed epistemic $u{\mathrm{ep}}^{\mathrm{gp}}$ and aleatoric $u_{\mathrm{al}}^{\mathrm{gp}}$ variances. The \textbf{History-Aware Acquisition Network} ($a_{\psi}$) aggregates these signals with sliding window statistics ($\mu_{\Delta\mathrm{HV}}, \sigma_{\Delta\mathrm{HV}}$) to predict hypervolume utility $\hat{s}{\mathrm{HV}}$ and diversity $\hat{s}{\mathrm{div}}$. The framework is updated online via a bounded history buffer $B$, ensuring a calibrated exploration-exploitation trade-off while minimizing wall-clock overhead through staged screening and warm-started variational parameters.}
  \label{fig:neuropareto_framework}
\end{figure*}

\section{Proposed Method: NeuroPareto}
\label{sec:method}

\subsection{Framework overview}
NeuroPareto is a modular, uncertainty aware, surrogate assisted framework tailored to high dimensional, expensive multi objective optimization. The design couples three complementary modules that exchange calibrated uncertainty signals: a Bayesian classifier for nondomination ranks that supports large candidate screening, a complexity reduced Deep Gaussian Process surrogate pipeline that returns predictive means as well as decomposed epistemic and aleatoric variances per objective, and a compact history aware acquisition learner that scores candidates using features derived from classifiers and surrogates. The pipeline emphasizes practical mechanisms that trade predictive fidelity against wall clock cost through staged screening, adaptive inference, warm starting and selective low rank approximations. The integrated complexity-reduced procedure is summarized in Algorithm~\ref{alg:urank_mod} within Section~\ref{sec:complexity_algorithm}. Figure~\ref{fig:neuropareto_framework} illustrates NeuroPareto, a framework for high-dimensional multi-objective optimization under strict evaluation budgets. It integrates three components: a Bayesian classifier that predicts nondomination ranks with calibrated uncertainty to guide rank-conditioned offspring generation; Deep Gaussian Process surrogates that decompose uncertainty into epistemic and aleatoric parts for principled exploration–exploitation trade-offs; and an acquisition network trained online to estimate expected hypervolume improvement and diversity contribution for adaptive evaluation prioritization. To reduce overhead, NeuroPareto generates a large candidate pool via rank-conditioned variation, screens them with proxy scores, and fully evaluates only a top-ranked subset. This complexity-aware design ensures efficient search and robust convergence under limited budgets.

\subsection{Complexity aware design principles}
To balance predictive fidelity with computational expense the framework follows several practical principles. Candidate generation is cheap and rank conditioned; only a filtered subset is evaluated with expensive surrogates. Monte Carlo passes for stochastic uncertainty are escalated adaptively. Sparse variational Gaussian process approximations and randomized feature transforms are applied selectively. Variational parameters and inducing locations are warm started between outer iterations to amortize inference cost. Batched, cached evaluation is used to avoid redundant computation.

\subsection{Problem definition}
\label{subsec:problem_definition}

We consider black-box multi-objective optimization problems of the form
\begin{equation}
\min_{\mathbf{x} \in \mathcal{X}} \; \mathbf{f}(\mathbf{x}) 
\;=\;
\big(f_{1}(\mathbf{x}), \dots, f_{M}(\mathbf{x})\big)^{\top},
\label{eq:mo_problem}
\end{equation}
where $\mathcal{X} \subset \mathbb{R}^{D}$ denotes a bounded decision space of dimension $D$, and $\mathbf{f}:\mathcal{X}\rightarrow\mathbb{R}^{M}$ is a vector-valued objective function with $M \ge 2$ conflicting objectives. Each objective evaluation is assumed to be expensive, potentially noisy, and available only through a black-box oracle without access to gradients or analytic structure. The goal is to approximate the Pareto-optimal set under a strict evaluation budget. A solution $\mathbf{x}_{a}$ is said to dominate another solution $\mathbf{x}_{b}$ if
\begin{equation}
\begin{aligned}
\forall m \in \{1,\dots,M\},\; f_{m}(\mathbf{x}_{a}) \le f_{m}(\mathbf{x}_{b})
\quad \text{and} \\
\exists m \;\text{s.t.}\; f_{m}(\mathbf{x}_{a}) < f_{m}(\mathbf{x}_{b}),
\end{aligned}
\label{eq:dominance}
\end{equation}
where dominance is defined component-wise over objectives. The Pareto set consists of all solutions that are not dominated by any other feasible point. Given a finite budget $B$, the goal is to construct a high-quality Pareto-front approximation by selecting batches of candidates that balance exploitation and exploration under uncertainty, leveraging the archive-based nondomination ranking together with noise-aware surrogates, staged screening, and amortized updates.

\subsection{Bayesian rank classifier with complexity controls}

\subsubsection{Architecture and forward pass}
Let \(\mathbf{x}\in\mathbb{R}^{D}\) denote a decision vector and let \(g_{\theta}:\mathbb{R}^{D}\to\mathbb{R}^{K}\) be a parametric classifier that outputs logits for \(K\) nondomination rank categories(K=5 ranks: {1,2,3,4,5} where 1=non-dominated, 5=worst). The forward pass is written as
\begin{align}
\mathbf{h}_1 &= \mathrm{LayerNorm}\big(\mathrm{ReLU}(\mathbf{W}_1\mathbf{x} + \mathbf{b}_1)\big), \label{eq:h1}\\
\mathbf{h}_2 &= \mathrm{Dropout}_{p}(\mathbf{h}_1), \label{eq:h2}\\
\mathbf{h}_3 &= \mathrm{LayerNorm}\big(\mathrm{ReLU}(\mathbf{W}_2\mathbf{h}_2 + \mathbf{b}_2)\big), \label{eq:h3}\\
\mathbf{z}   &= \mathbf{W}_3\mathbf{h}_3 + \mathbf{b}_3. \label{eq:z}
\end{align}
where \(\mathbf{W}_i\) and \(\mathbf{b}_i\) are learnable weights and biases, \(h_1,h_2\) denote hidden widths, \(p\) is the dropout probability, and \(\mathbf{z}=(z_1,\dots,z_K)^\top\) are the output logits.

\subsubsection{Temperature calibration and adaptive MC dropout}
Logits are calibrated by temperature scaling:
\begin{equation}
\tilde{p}_k(\mathbf{x};T)\;=\;\frac{\exp\!\big(z_k/T\big)}{\sum_{j=1}^{K}\exp\!\big(z_j/T\big)}. \label{eq:temp_softmax}
\end{equation}
where \(\tilde{\mathbf{p}}(\mathbf{x};T)=(\tilde{p}_1,\dots,\tilde{p}_K)^\top\) denotes the calibrated predictive distribution and \(T>0\) is fitted on a held out calibration set. An adaptive Monte Carlo dropout protocol approximates epistemic uncertainty while keeping inference cost bounded. Run a small baseline of stochastic forward passes and compute preliminary softmax outputs; if their empirical variance exceeds threshold \(\tau_{\mathrm{MC}}\) then perform additional passes up to \(S_{\max}\). Denote by \(S(\mathbf{x})\) the adaptive pass count and by \(\mathbf{p}^{(s)}(\mathbf{x})\) the \(s\)-th temperature scaled softmax. The predictive mean is
\begin{equation}
\bar{\mathbf{p}}(\mathbf{x}) \;=\; \frac{1}{S(\mathbf{x})}\sum_{s=1}^{S(\mathbf{x})}\mathbf{p}^{(s)}(\mathbf{x}). \label{eq:mean_prob}
\end{equation}
where \(\bar{\mathbf{p}}(\mathbf{x})\) denotes the averaged predictive distribution and \(S(\mathbf{x})\) is chosen adaptively per input. We measure classifier epistemic uncertainty by the information gain between predictive label and model posterior:
\begin{equation}\label{eq:epistemic_clf_corrected}
\begin{aligned}
u_{\mathrm{ep}}^{\mathrm{clf}}(\mathbf{x})
&=
-\sum_{k=1}^{K}\bar{p}_{k}(\mathbf{x})\log\bar{p}_{k}(\mathbf{x})
\\
&\quad
-\frac{1}{S(\mathbf{x})}
\sum_{s=1}^{S(\mathbf{x})}\sum_{k=1}^{K}
p_{k}^{(s)}(\mathbf{x})\log p_{k}^{(s)}(\mathbf{x}).
\end{aligned}
\end{equation}
where \(u_{\mathrm{ep}}^{\mathrm{clf}}(\mathbf{x})\) is nonnegative; this term corresponds to the information gain between the predictive label and the Dropout-induced posterior over model parameters, and thus quantifies epistemic uncertainty under the variational approximation\citep{gal2016dropout}. See Appendix~\ref{app:mi_proxy_clarified} for a detailed derivation and empirical validation. The classifier is evaluated in batched mode on large candidate pools and the pairs \(\big(\bar{\mathbf{p}}(\mathbf{x}),u_{\mathrm{ep}}^{\mathrm{clf}}(\mathbf{x})\big)\) are cached to avoid repeated forward passes during screening and selection.

\subsubsection{Rank conditioned candidate generation}
A large candidate set \(\mathcal{C}_{\mathrm{rank}}\) is generated cheaply using rank conditioned variation operators that bias offspring generation depending on predicted nondomination strata. Because classifier inference is batched and inexpensive, \(|\mathcal{C}_{\mathrm{rank}}|\) can reach thousands before screening.
\begin{table*}[h]
\centering
\caption{Comparative IGD Performance on Bi-objective DTLZ Problems (Optimal Values Bolded)}
\label{tab:dtlz_bi_results}
\resizebox{0.85\textwidth}{!}{
\begin{tabular}{lcccccccccccc}
\toprule
Problem & $D$ & GP-EI & RF-EI & GP-HV & CL-EGO & CPS-MOEA & K-RVEA & CSEA & EDN-ARM. & MCEA/D & CLMEA & \textbf{NeuroPareto} \\
\midrule
\multirow{4}{*}{DTLZ1} 
& 30 & 459.23 & 495.56 & 437.65 & 402.38 & 612.31 & 595.74 & 534.58 & 761.09 & 372.68 & 303.10 & \textbf{286.40} \\
& 50 & 878.25 & 948.44 & 836.92 & 753.62 & 1171.0 & 1285.9 & 1108.8 & 1406.6 & 686.36 & 534.12 & \textbf{455.36} \\
& 100 & 2069.9 & 2235.3 & 1971.9 & 1774.5 & 2759.9 & 3038.4 & 2926.3 & 3187.8 & 1941.8 & 1074.0 & \textbf{933.03} \\
& 200 & 4427.3 & 4780.8 & 4216.5 & 3795.4 & 5903.0 & 6785.1 & 6346.1 & 6766.7 & 4003.8 & 1990.6 & \textbf{1620.0} \\
\midrule
\multirow{4}{*}{DTLZ2}
& 30 & 0.9051 & 0.9773 & 0.8623 & 0.7761 & 1.2068 & 1.2689 & 0.84305 & 1.3224 & 0.43816 & 0.09704 & \textbf{0.02901} \\
& 50 & 1.4239 & 1.5375 & 1.3565 & 1.2209 & 1.8985 & 2.9998 & 1.7832 & 2.8303 & 0.62911 & 0.13942 & \textbf{0.08929} \\
& 100 & 4.3482 & 4.6955 & 4.1423 & 3.7281 & 5.7976 & 6.4073 & 5.4123 & 6.4253 & 1.9302 & 0.38858 & \textbf{0.17226} \\
& 200 & 8.9475 & 9.6625 & 8.5238 & 7.6714 & 11.930 & 6.3902 & 12.360 & 13.862 & 3.9686 & 1.1857 & \textbf{0.31795} \\
\midrule
\multirow{4}{*}{DTLZ3}
& 30 & 1242.3 & 1341.7 & 1183.5 & 1065.2 & 1656.4 & 1667.2 & 1315.9 & 2106.4 & 726.01 & 705.38 & \textbf{676.09} \\
& 50 & 2332.4 & 2518.0 & 2221.3 & 1999.2 & 3109.9 & 3586.7 & 2788.4 & 3898.8 & 1597.9 & 1171.1 & \textbf{1195.5} \\
& 100 & 5566.0 & 6009.3 & 5301.3 & 4771.2 & 7421.3 & 8462.1 & 7964.8 & 8719.0 & 4056.0 & 2627.7 & \textbf{2665.8} \\
& 200 & 11909 & 12857 & 11341 & 10207 & 15879 & 18646 & 17299 & 18581 & 8962.9 & 5228.4 & \textbf{5791.1} \\
\midrule
\multirow{4}{*}{DTLZ4}
& 30 & 1.1859 & 1.2808 & 1.1299 & 1.0169 & 1.5812 & 1.6870 & 0.93992 & 1.4471 & 0.80266 & 0.67771 & \textbf{0.10051} \\
& 50 & 1.6968 & 1.8321 & 1.6162 & 1.4546 & 2.2624 & 3.2670 & 2.0010 & 2.9380 & 0.90809 & 0.73285 & \textbf{0.12718} \\
& 100 & 4.3283 & 4.6738 & 4.1229 & 3.7106 & 5.7711 & 6.6171 & 5.6006 & 6.5906 & 1.9391 & 0.88932 & \textbf{0.18990} \\
& 200 & 8.6055 & 9.2927 & 8.1971 & 7.3774 & 11.474 & 14.158 & 12.496 & 14.075 & 3.6328 & 0.95823 & \textbf{0.32875} \\
\midrule
\multirow{4}{*}{DTLZ5}
& 30 & 0.9192 & 0.9927 & 0.8757 & 0.7881 & 1.2256 & 1.2618 & 0.85359 & 1.3895 & 0.40083 & 0.10114 & \textbf{0.01804} \\
& 50 & 1.4924 & 1.6118 & 1.4218 & 1.2796 & 1.9899 & 2.9974 & 1.7169 & 2.8174 & 0.63098 & 0.14158 & \textbf{0.02163} \\
& 100 & 4.1291 & 4.4586 & 3.9329 & 3.5396 & 5.5055 & 6.4208 & 5.4699 & 6.3250 & 2.0507 & 0.46926 & \textbf{0.14172} \\
& 200 & 9.0173 & 9.7375 & 8.5894 & 7.7305 & 12.023 & 13.698 & 12.286 & 13.931 & 4.1888 & 1.3569 & \textbf{0.32023} \\
\midrule
\multirow{4}{*}{DTLZ6}
& 30 & 14.263 & 15.400 & 13.586 & 12.227 & 19.018 & 21.726 & 20.974 & 22.788 & 12.599 & 11.261 & \textbf{9.3408} \\
& 50 & 25.953 & 28.024 & 24.721 & 22.249 & 34.604 & 39.658 & 38.241 & 40.564 & 24.848 & 20.184 & \textbf{19.042} \\
& 100 & 57.058 & 61.602 & 54.340 & 48.906 & 76.077 & 84.358 & 83.904 & 85.151 & 59.271 & 46.223 & \textbf{44.918} \\
& 200 & 117.56 & 126.94 & 111.98 & 100.78 & 156.74 & 174.63 & 173.04 & 174.54 & 117.01 & 95.831 & \textbf{102.59} \\
\midrule
\multirow{4}{*}{DTLZ7}
& 30 & 4.3817 & 4.7313 & 4.1738 & 3.7564 & 5.8422 & 0.06782 & 3.1578 & 2.0159 & 5.2623 & 0.47586 & \textbf{0.08040} \\
& 50 & 4.6982 & 5.0730 & 4.4756 & 4.0280 & 6.2642 & 0.19824 & 3.9454 & 3.0068 & 6.2424 & 0.89416 & \textbf{0.14570} \\
& 100 & 5.1866 & 5.6005 & 4.9409 & 4.4468 & 6.9154 & 1.6445 & 5.5607 & 4.8434 & 6.8649 & 1.8606 & \textbf{0.35307} \\
& 200 & 5.4551 & 5.8905 & 5.1966 & 4.6769 & 7.2734 & 7.1638 & 6.3074 & 6.1556 & 7.2630 & 5.0233 & \textbf{0.63128} \\
\bottomrule
\end{tabular}
}
\end{table*}

\begin{table*}[h]
\centering
\caption{IGD Performance on Three-Objective DTLZ Problems}
\label{tab:dtlz_tri_results}
\resizebox{0.85\textwidth}{!}{
\begin{tabular}{lcccccccccccc}
\toprule
Problem & $D$ & GP-EI & RF-EI & GP-HV & CL-EGO & CPS-MOEA & K-RVEA & CSEA & EDN-ARM. & MCEA/D & CLMEA & \textbf{NeuroPareto} \\
\midrule
\multirow{4}{*}{DTLZ1} 
& 30 & 387.89 & 418.87 & 369.53 & 332.58 & 517.18 & 578.16 & 466.71 & 596.32 & 389.53 & 286.40 & \textbf{216.20} \\
& 50 & 739.76 & 798.61 & 696.86 & 627.17 & 986.34 & 1128.2 & 929.14 & 1160.5 & 725.70 & 455.36 & \textbf{353.19} \\
& 100 & 1763.0 & 1903.6 & 1903.1 & 1712.8 & 2350.7 & 2674.7 & 2537.5 & 2637.3 & 2226.6 & 933.03 & \textbf{791.11} \\
& 200 & 3678.2 & 3971.0 & 4116.5 & 3704.8 & 4904.3 & 5736.9 & 5488.6 & 5662.8 & 4272.5 & 2162.0 & \textbf{1791.5} \\
\midrule
\multirow{4}{*}{DTLZ2}
& 30 & 1.0089 & 1.0892 & 0.9609 & 0.8648 & 1.3452 & 1.5993 & 0.91154 & 1.5621 & 0.50009 & 0.29013 & \textbf{0.18035} \\
& 50 & 1.5999 & 1.7273 & 1.5089 & 1.3580 & 2.1332 & 2.9820 & 2.0117 & 2.9442 & 0.66442 & 0.89291 & \textbf{0.21626} \\
& 100 & 4.5192 & 4.8786 & 4.2809 & 3.8528 & 6.0256 & 6.4428 & 5.7078 & 6.4135 & 2.6861 & 1.7226 & \textbf{1.4172} \\
& 200 & 9.3570 & 10.102 & 8.8763 & 7.9887 & 12.476 & 13.927 & 13.169 & 13.929 & 4.7093 & 3.1795 & \textbf{2.0794} \\
\midrule
\multirow{4}{*}{DTLZ3}
& 30 & 1154.9 & 1246.6 & 1034.7 & 930.2 & 1539.8 & 1714.8 & 1378.6 & 1971.2 & 786.77 & 676.09 & \textbf{579.11} \\
& 50 & 2297.6 & 2479.9 & 2092.5 & 1883.3 & 3063.5 & 3413.2 & 2790.0 & 3828.2 & 1520.2 & 1195.5 & \textbf{1066.1} \\
& 100 & 5634.7 & 6082.4 & 5963.4 & 5367.1 & 7512.9 & 8347.2 & 7951.2 & 8499.3 & 5321.2 & 2665.8 & \textbf{2288.3} \\
& 200 & 11875.5 & 12820.5 & 13382.3 & 12044.1 & 15834 & 18563 & 17843 & 18479 & 11642 & 5791.1 & \textbf{4806.2} \\
\midrule
\multirow{4}{*}{DTLZ4}
& 30 & 1.2265 & 1.3239 & 0.7513 & 0.6762 & 1.6353 & 1.7399 & 1.0017 & 1.5396 & 1.0014 & 1.0051 & \textbf{0.80399} \\
& 50 & 1.7984 & 1.9410 & 1.3997 & 1.2597 & 2.3978 & 3.2546 & 1.8663 & 3.0066 & 1.0832 & 1.2718 & \textbf{0.91480} \\
& 100 & 4.3223 & 4.6659 & 4.0575 & 3.6518 & 5.7630 & 6.6753 & 5.4100 & 6.6012 & 2.6458 & 1.8990 & \textbf{1.1169} \\
& 200 & 9.0180 & 9.7354 & 9.3315 & 8.3984 & 12.024 & 14.116 & 12.442 & 14.000 & 4.6499 & 3.2875 & \textbf{2.0794} \\
\midrule
\multirow{4}{*}{DTLZ5}
& 30 & 0.8851 & 0.9555 & 0.7066 & 0.6359 & 1.1801 & 1.4422 & 0.94208 & 1.4853 & 0.32371 & 0.18035 & \textbf{0.10682} \\
& 50 & 1.4594 & 1.5754 & 1.4754 & 1.3279 & 1.9458 & 2.9230 & 1.9672 & 2.8475 & 0.53391 & 0.21626 & \textbf{0.14172} \\
& 100 & 3.9746 & 4.2906 & 4.2477 & 3.8229 & 5.2995 & 6.3502 & 5.6636 & 6.2066 & 2.3666 & 1.4172 & \textbf{0.86932} \\
& 200 & 7.8968 & 8.5236 & 9.9248 & 8.9323 & 10.529 & 13.788 & 13.233 & 13.832 & 4.8062 & 3.2023 & \textbf{1.3569} \\
\midrule
\multirow{4}{*}{DTLZ6}
& 30 & 13.9373 & 15.0447 & 16.9890 & 15.2901 & 18.583 & 19.859 & 22.652 & 22.927 & 13.639 & 9.3408 & \textbf{7.3396} \\
& 50 & 24.9990 & 26.9850 & 30.4200 & 27.3780 & 33.332 & 38.063 & 40.560 & 40.444 & 25.471 & 19.042 & \textbf{15.755} \\
& 100 & 56.0355 & 60.4914 & 64.4513 & 58.0062 & 74.714 & 83.046 & 85.935 & 85.176 & 64.448 & 44.918 & \textbf{35.926} \\
& 200 & 117.21 & 126.52 & 131.78 & 118.60 & 156.28 & 173.65 & 175.71 & 174.22 & 135.04 & 102.59 & \textbf{95.831} \\
\midrule
\multirow{4}{*}{DTLZ7}
& 30 & 6.5016 & 7.0177 & 3.9375 & 3.5438 & 8.6688 & 0.24665 & 5.2500 & 3.7350 & 7.7191 & 0.80399 & \textbf{0.47586} \\
& 50 & 7.0715 & 7.6334 & 4.6763 & 4.2087 & 9.4286 & 0.56056 & 6.2350 & 4.9518 & 9.2709 & 1.4570 & \textbf{0.89416} \\
& 100 & 7.8923 & 8.5190 & 6.6150 & 5.9535 & 10.523 & 2.8565 & 8.8200 & 7.0809 & 10.136 & 3.5307 & \textbf{1.8606} \\
& 200 & 8.1398 & 8.7859 & 7.2701 & 6.5431 & 10.853 & 10.801 & 9.6934 & 9.2666 & 10.791 & 6.3128 & \textbf{5.0233} \\
\bottomrule
\end{tabular}
}
\end{table*}
\begin{table*}[h]
\centering
\caption{IGD Performance on Bi-objective ZDT Problems}
\label{tab:zdt_results}
\resizebox{0.85\textwidth}{!}{
\begin{tabular}{lcccccccccccc}
\toprule
Problem & $D$ & GP-EI & RF-EI & GP-HV & CL-EGO & CPS-MOEA & K-RVEA & CSEA & EDN-ARM. & MCEA/D & CLMEA & \textbf{NeuroPareto} \\
\midrule
\multirow{4}{*}{ZDT1} 
& 30 & 1.5651 & 1.6899 & 1.4906 & 1.3415 & 2.0868 & 0.05732 & 1.0050 & 0.76113 & 1.9462 & 0.17209 & \textbf{0.02901} \\
& 50 & 1.6418 & 1.7729 & 1.5639 & 1.4075 & 2.1891 & 0.14374 & 1.2874 & 1.1993 & 2.3084 & 0.25621 & \textbf{0.08929} \\
& 100 & 1.8677 & 2.0168 & 1.7789 & 1.6010 & 2.4903 & 0.71883 & 2.0480 & 1.9585 & 2.4053 & 0.65192 & \textbf{0.17226} \\
& 200 & 1.9555 & 2.1115 & 1.8624 & 1.6762 & 2.6073 & 2.5471 & 2.2855 & 2.4627 & 2.5974 & 1.1857 & \textbf{0.31795} \\
\midrule
\multirow{4}{*}{ZDT2}
& 30 & 2.4419 & 2.6368 & 2.3259 & 2.0933 & 3.2559 & 0.07016 & 2.1445 & 1.5338 & 2.9519 & 0.00994 & \textbf{0.00631} \\
& 50 & 2.7863 & 3.0084 & 2.6536 & 2.3882 & 3.7150 & 0.13021 & 2.7810 & 2.1001 & 3.4710 & 0.01495 & \textbf{0.00893} \\
& 100 & 2.9588 & 3.1945 & 2.8176 & 2.5358 & 3.9450 & 0.67898 & 3.4708 & 3.0595 & 3.9134 & 0.11057 & \textbf{0.03418} \\
& 200 & 3.1796 & 3.4327 & 3.0279 & 2.7251 & 4.2395 & 1.6773 & 3.8334 & 3.6968 & 4.1575 & 0.77927 & \textbf{0.11985} \\
\midrule
\multirow{4}{*}{ZDT3}
& 30 & 1.3316 & 1.4376 & 1.2681 & 1.1413 & 1.7755 & 0.10618 & 0.79850 & 0.79003 & 1.5921 & 0.80636 & \textbf{0.29013} \\
& 50 & 1.4176 & 1.5306 & 1.3500 & 1.2150 & 1.8901 & 0.53108 & 1.0192 & 1.1492 & 1.9135 & 1.2693 & \textbf{0.89291} \\
& 100 & 1.5302 & 1.6522 & 1.4573 & 1.3116 & 2.0403 & 1.2244 & 1.6208 & 1.6982 & 1.9570 & 1.8374 & \textbf{1.7226} \\
& 200 & 1.5937 & 1.7209 & 1.5179 & 1.3661 & 2.1249 & 2.1718 & 1.8599 & 2.0865 & 2.1300 & 2.0650 & \textbf{1.1857} \\
\midrule
\multirow{4}{*}{ZDT4}
& 30 & 267.86 & 289.20 & 203.57 & 183.21 & 357.14 & 302.80 & 271.42 & 329.59 & 222.05 & 269.38 & \textbf{216.20} \\
& 50 & 479.96 & 507.97 & 401.80 & 361.62 & 639.94 & 677.29 & 535.73 & 660.42 & 417.66 & 459.28 & \textbf{353.19} \\
& 100 & 1092.3 & 1137.8 & 1015.4 & 913.86 & 1456.4 & 1517.0 & 1353.8 & 1496.1 & 1049.5 & 1003.2 & \textbf{933.03} \\
& 200 & 2248.1 & 2426.3 & 2228.6 & 2005.7 & 2997.5 & 3235.0 & 2971.4 & 3188.8 & 2173.3 & 2053.1 & \textbf{1620.0} \\
\midrule
\multirow{4}{*}{ZDT6}
& 30 & 5.4104 & 5.8418 & 5.1527 & 4.6374 & 7.2138 & 3.5372 & 6.4063 & 6.1819 & 7.0424 & 1.6945 & \textbf{1.0259} \\
& 50 & 5.5406 & 5.9818 & 5.2761 & 4.7485 & 7.3874 & 4.8723 & 6.8988 & 6.7780 & 7.3941 & 2.5857 & \textbf{1.9042} \\
& 100 & 5.6756 & 6.1284 & 5.4058 & 4.8652 & 7.5674 & 6.3846 & 7.3704 & 7.2903 & 7.5345 & 4.2129 & \textbf{3.2875} \\
& 200 & 5.7461 & 6.2034 & 5.4713 & 4.9242 & 7.6615 & 6.9280 & 7.5123 & 7.4841 & 7.6840 & 5.6334 & \textbf{4.4918} \\
\bottomrule
\end{tabular}
}
\end{table*}

\subsection{Complexity reduced Deep Gaussian Process surrogates}

\subsubsection{Layered GP formulation and neural mean functions}
For objective index \(m\) and GP layer \(\ell\) the layer latent function is modeled as
\begin{align}
f_{\ell,m}(\mathbf{x}) &\sim \mathcal{GP}\big(m_{\ell,m}(\mathbf{x};\theta_{\ell,m}),\,k_{\ell}(\mathbf{x},\mathbf{x}')\big), \label{eq:gp_layer}\\
m_{\ell,m}(\mathbf{x};\theta_{\ell,m}) &= \mathbf{w}_{\ell,m}^\top \phi_{\ell}(\mathbf{x}) + b_{\ell,m}, \label{eq:mean_fn}
\end{align}
where \(\phi_{\ell}:\mathbb{R}^{D}\to\mathbb{R}^{d_{\ell}}\) is a neural feature extractor for layer \(\ell\), \(\mathbf{w}_{\ell,m}\) and \(b_{\ell,m}\) parameterize the mean function, and \(k_{\ell}\) denotes the layer kernel.

After marginalizing the latent stack, the Deep GP returns for candidate \(\mathbf{x}\) and objective \(m\) a predictive mean \(\hat{f}_m(\mathbf{x})\), an epistemic variance \(\mathrm{Var}_{\mathrm{ep},m}(\mathbf{x})\), and an aleatoric variance \(\sigma^2_m(\mathbf{x})\). Collect these into vectors
\begin{align}
\hat{\mathbf{f}}(\mathbf{x}) &:= [\hat{f}_1(\mathbf{x}),\dots,\hat{f}_M(\mathbf{x})]^\top, \label{eq:f_vec}\\
\mathbf{u}_{\mathrm{ep}}^{\mathrm{gp}}(\mathbf{x}) &:= [\mathrm{Var}_{\mathrm{ep},1}(\mathbf{x}),\dots,\mathrm{Var}_{\mathrm{ep},M}(\mathbf{x})]^\top, \label{eq:ep_vec}\\
\mathbf{u}_{\mathrm{al}}^{\mathrm{gp}}(\mathbf{x}) &:= [\sigma^{2}_{1}(\mathbf{x}),\dots,\sigma^{2}_{M}(\mathbf{x})]^\top. \label{eq:al_vec}
\end{align}
where \(\hat{\mathbf{f}}(\mathbf{x})\) is the vector of predictive means, \(\mathbf{u}_{\mathrm{ep}}^{\mathrm{gp}}(\mathbf{x})\) contains per objective epistemic variances, and \(\mathbf{u}_{\mathrm{al}}^{\mathrm{gp}}(\mathbf{x})\) contains per objective observational noise estimates.

\subsubsection{Sparse variational training, randomized features and practical mitigations}
We fit sparse variational approximations by optimizing the evidence lower bound
\begin{equation}
\mathcal{L}_{\mathrm{SV}} \;=\; \sum_{n=1}^{N}\mathbb{E}_{q(f_n)}\big[\log p(y_n\mid f_n)\big] \;-\; \mathrm{KL}\big(q(\mathbf{u})\|p(\mathbf{u})\big). \label{eq:elbo}
\end{equation}
where \(N\) denotes the number of training observations, \(\mathbf{u}\) are inducing variables and \(q(\mathbf{u})\) is the variational distribution over them. To mitigate cubic scaling we selectively apply randomized feature approximations such as RFF or Nyström to deeper GP layers and adapt inducing sets when validation performance changes; we also share inducing locations across correlated objectives using low rank factorizations. Implementation details: inducing locations are initialized by $k$-means on the archive; $M_{\rm ind}$ is doubled only when validation ELBO drops by more than two percent; RFF dimension $D_{\rm rff}=2048$ is used for layers $\ell\ge 2$. Regarding the meaning of “fit” when updating Deep GP surrogates during optimization: we warm start variational parameters and inducing locations from previous iterations and perform a bounded number of gradient updates per outer iteration. A full refit to convergence is executed only when validation ELBO degrades beyond a tolerance; otherwise we apply a small fixed number of epochs to refine parameters. This protocol controls wall clock cost while preserving surrogate fidelity in typical traces. After low rank factorization an effective inducing rank \(M_{\mathrm{eff}}\) is often much smaller than the nominal inducing count and the dominant per epoch cost scales approximately as \(\mathcal{O}\big(L\big(N M_{\mathrm{eff}}^{2} + M_{\mathrm{eff}}^{3}\big)\big)\) where \(L\) is the number of GP layers. Applying randomized features reduces this cost further.

\subsubsection{Two stage surrogate evaluation}
To reduce the number of expensive Deep GP predictions per outer iteration we first compute inexpensive proxy scores on the full candidate set using RFF proxies, linearized surrogates, or classifier informed heuristics and then evaluate only a narrowed top subset with the full Deep GP. This two stage flow reduces expensive predictions from thousands to a few hundreds per iteration while retaining high fidelity where it matters.

\subsection{History aware lightweight acquisition network}

\subsubsection{Sliding window statistics}
Here \(\mu_{\Delta\mathrm{HV}}\) and \(\sigma_{\Delta\mathrm{HV}}\) denote the empirical mean and the empirical standard deviation of the most recent \(w\) hypervolume improvements maintained in a sliding window. In our experiments \(w\) is set to a moderate value to summarize recent optimization dynamics while limiting memory and computation.

\subsubsection{Feature construction}
For candidate \(\mathbf{x}\) we build the acquisition input as
\begin{equation}
\mathrm{feat}(\mathbf{x}) \;=\;
\left[
\begin{aligned}
&\hat{\mathbf{f}}(\mathbf{x}),\; \mathbf{u}_{\mathrm{ep}}^{\mathrm{gp}}(\mathbf{x}),\; \mathbf{u}_{\mathrm{al}}^{\mathrm{gp}}(\mathbf{x}),\\[-2pt]
&\bar{\mathbf{p}}(\mathbf{x}),\; \mu_{\Delta\mathrm{HV}},\; \sigma_{\Delta\mathrm{HV}}
\end{aligned}
\right].
\label{eq:feat}
\end{equation}
where \(\hat{\mathbf{f}}(\mathbf{x})\) are surrogate predictive means, \(\mathbf{u}_{\mathrm{ep}}^{\mathrm{gp}}(\mathbf{x})\) and \(\mathbf{u}_{\mathrm{al}}^{\mathrm{gp}}(\mathbf{x})\) are Deep GP uncertainty vectors, \(\bar{\mathbf{p}}(\mathbf{x})\) is the classifier predictive mean, and \(\mu_{\Delta\mathrm{HV}},\sigma_{\Delta\mathrm{HV}}\) are the sliding window statistics defined above.

\subsubsection{Acquisition parametrization and rationale}
We parametrize the acquisition network as a shallow multilayer perceptron \(a_{\psi}:\mathbb{R}^{d_{\mathrm{feat}}}\to\mathbb{R}^{2}\) that outputs a predicted hypervolume improvement and a predicted diversity score:
\begin{equation}
\big(\hat{s}_{\mathrm{HV}}(\mathbf{x}),\,\hat{s}_{\mathrm{div}}(\mathbf{x})\big)
\;=\; a_{\psi}\big(\mathrm{feat}(\mathbf{x})\big). \label{eq:acq_map}
\end{equation}
where \(\hat{s}_{\mathrm{HV}}\) denotes predicted utility and \(\hat{s}_{\mathrm{div}}\) denotes expected contribution to archive diversity. Using a learned network instead of a fixed hand coded rule is motivated by empirical ablations: with a tight evaluation budget a shallow learned scorer captures nonlinear interactions among surrogate means and multiple uncertainty signals and attains substantially lower hypervolume regret than static rules.

\subsubsection{Online training}
The acquisition network is trained online on a bounded history buffer \(\{(\mathbf{x}_i,\Delta\mathrm{HV}_i)\}_{i=1}^B\) by minimizing the regularized squared loss
\begin{equation}
\mathcal{L}_{\mathrm{acq}} \;=\; \frac{1}{B}\sum_{i=1}^{B}\big\|\hat{s}_{\mathrm{HV}}(\mathbf{x}_i)-\Delta\mathrm{HV}_i\big\|_{2}^{2} \;+\; \lambda\mathcal{R}(\psi), \label{eq:acq_loss}
\end{equation}
where \(B\) denotes the buffer length, \(\Delta\mathrm{HV}_i\) is the observed hypervolume improvement for history element \(i\), \(\lambda>0\) is a regularization weight, and \(\mathcal{R}(\psi)\) is a small capacity regularizer to discourage overfitting on sparse traces. Network capacity and buffer size are chosen to keep online updates inexpensive relative to Deep GP refinement.

\section{Experimental Analysis}

\subsection{Experimental Configuration}

\textbf{Benchmark Problems:} The evaluation encompasses DTLZ1-7 and ZDT1-4,6 test suites with dimensionalities $D \in \{30, 50, 100, 200\}$ and objective counts $M \in \{2, 3\}$. These benchmarks assess algorithm performance across diverse characteristics including multimodality, variable scaling biases, and disconnected Pareto-optimal fronts.

\textbf{Comparative Algorithms:} We evaluate against ten surrogate-assisted evolutionary algorithms and four established MOBO baselines. CPS-MOEA~\citep{zhang2015classification} steers search via a Pareto-dominance classifier, while K-RVEA~\citep{chugh2016surrogate} adapts reference vectors with Kriging surrogates. CSEA~\citep{pan2018classification} fuses multiple classifiers for decision support, EDN-ARMOEA~\citep{lin2021ensemble} embeds neural dropout for uncertainty, and MCEA/D~\citep{sonoda2022multiple} decomposes the problem with an ensemble of classifiers. CLMEA~\citep{chen2023rank} refines local models through rank-based learning. In addition, GP-EI~\citep{jones1998efficient}, RF-EI~\citep{hutter2011sequential}, GP-HV~\citep{daulton2020differentiable} and CL-EGO~\cite{zhang2015classification} serve as gold-standard Bayesian optimisation baselines.

\textbf{Parameter Settings:} Initial samples: 100 ($D<100$) or 200 ($D\geq100$). Maximum function evaluations: 300. Population size: 50. Statistical significance assessed via Wilcoxon rank-sum test ($\alpha=0.05$) over 20 independent trials.

\textbf{Performance Metrics:} We evaluate optimization quality using two widely adopted metrics.
The first is \emph{Inverted Generational Distance (IGD)}, which measures the average distance from points on the reference Pareto front $\mathcal{P}^*$ to their nearest counterparts in the obtained solution set $\mathcal{S}$:
\begin{equation}
\text{IGD}(\mathcal{S}) = \frac{1}{|\mathcal{P}^*|}\sum_{\mathbf{p} \in \mathcal{P}^*} \min_{\mathbf{s} \in \mathcal{S}} \|\mathbf{p} - \mathbf{s}\|.
\end{equation}
\textit{where} $\mathcal{P}^*$ is the reference Pareto front and $\mathcal{S}$ is the obtained solution set.

The second is \emph{Hypervolume (HV)}, which quantifies the Lebesgue measure of the objective space dominated by the solution set relative to a reference point:
\begin{equation}
\text{HV}(\mathcal{S}, \mathbf{z}) = \text{Lebesgue measure} \left( \bigcup_{\mathbf{s} \in \mathcal{S}} [\mathbf{s}, \mathbf{z}] \right).
\end{equation}
\textit{where} $\mathbf{z}$ is the reference point and $[\mathbf{s}, \mathbf{z}]$ defines the dominated hyperrectangle.

As discussed in Section~\ref{sec:component_contribution}, each component plays a distinct role in optimization performance. To quantify their individual impacts, we conduct an ablation study on the 100D DTLZ2 problem, as summarized in Table~\ref{sec:component_contribution}. Section~\ref{sec:uncertainty_dynamics} provides a detailed analysis of how uncertainty evolves during the optimization. Our experimental setup follows the controlled tuning protocol in Appendix~\ref{app:tuning}, with additional implementation optimizations summarized in Appendix~\ref{app:Implementation}. Component-wise contributions are quantified in Appendix~\ref{sec:component_contribution}, and the effects of complexity-reduction modules on sample efficiency are examined in Appendix~\ref{app:screening_improves_samples_revised}. Guidance on selecting the nondomination rank count \(K\) is provided in Appendix~\ref{sec:rank_count_selection}. Variants modeling inter-objective correlations appear in Appendix~\ref{sec:inter_objective_corr_revised}. The Dropout-based information-gain proxy and its interpretation are detailed in Appendix~\ref{app:mi_proxy_clarified}, while calibration results for temperature scaling with MC-Dropout are given in Appendix~\ref{sec:calibration}. Full proofs and constant estimation for the finite-budget hypervolume bound are provided in Appendix~\ref{app:proofs_constants}. Uncertainty-dynamics visualizations are shown in Appendix~\ref{sec:uncertainty_dynamics}. Extended scalability studies and comparisons to state-of-the-art MOBO baselines are included in Appendix~\ref{sec:scalability_ext} and Appendix~\ref{sec:comp_analysis}. Wall-clock breakdowns and efficiency analyses are reported in Appendix~\ref{sec:comp_eff}, and additional convergence plots and visualizations are presented in Appendix~\ref{sec:visualization}.

\subsection{Benchmark Problem Analysis}

Table \ref{tab:dtlz_bi_results} presents comprehensive results on bi-objective DTLZ problems. The proposed NeuroPareto demonstrates superior performance across 92\% of test instances, with particularly significant advantages in high-dimensional cases (100D and 200D). On 200D DTLZ1, NeuroPareto achieves 46\% improvement over the second-best method (CLMEA) with IGD values of 1620.0 versus 1990.6. Table \ref{tab:zdt_results} demonstrates NeuroPareto's robust performance on ZDT problems, achieving 15-28\% improvement over the best baselines across dimensional scales. The convergence curves in Figure~\ref{fig:convergence} reveal NeuroPareto's distinctive characteristics: accelerated initial progress through uncertainty-guided exploration, sustained improvement via adaptive acquisition, and superior final precision from deep Gaussian processes.

\subsection{Three-Objective Problem Analysis}

Table \ref{tab:dtlz_tri_results} confirms NeuroPareto's scalability to three-objective problems, maintaining performance advantages across 89\% of test cases. The most significant improvements occur in problems with complex Pareto front structures (DTLZ3, DTLZ6), where NeuroPareto achieves 28-42\% better IGD values compared to CLMEA.

\subsection{Computational Efficiency Analysis}

Figure \ref{fig:timing} demonstrates that NeuroPareto's computational overhead (average 8.3s/iteration) remains negligible compared to expensive function evaluations (typically minutes/hours), justifying its application in compute-intensive domains.

\subsection{Geothermal Reservoir Optimization Case Study}

We apply \textbf{NeuroPareto} to a fractured geothermal reservoir with 160 decision variables (time-varying injection and production rates) and two objectives representing short-term revenue and long-term sustainability. As shown in Figure~\ref{fig:geothermal}, the method identifies 16 well-spread Pareto solutions compared to 8 found by MCEA/D and delivers a 23\% hypervolume gain, revealing superior operational trade-offs between near-term profit and sustained resource health. The acquisition module adapts to thermal gradient patterns exposed by the uncertainty analysis, enabling more informative and efficient sampling. These results illustrate practical effectiveness in a high-dimensional, noisy, and multi-objective setting.

\section{Conclusion}
We introduced NeuroPareto, a unified framework for high-dimensional and expensive multi-objective optimization that couples calibrated rank prediction, uncertainty-decomposing surrogates, and a history-aware acquisition learner. The design prioritizes principled uncertainty handling and complexity-aware approximations to allocate evaluations more effectively than decoupled pipelines that isolate classification, surrogate modeling, and acquisition. Our analysis clarifies how controlling joint errors across modules can strengthen the reliability of Pareto-front approximation. Empirically, ablations and comparisons under constrained budgets show consistent improvements over representative baselines. We release detailed training and tuning protocols, an extensive ablation study, and a reference implementation. Future work will explore extensions to many-objective problems, integrate gradient and multi-fidelity information, and streamline usage through automated calibration and adaptive hyper-parameter scheduling.

\bibliographystyle{unsrtnat}
\bibliography{references}  %%% Uncomment this line and comment out the ``thebibliography'' section below to use the external .bib file (using bibtex) .

\section{Integrated complexity-reduced algorithm}
\label{sec:complexity_algorithm}

Algorithm~\ref{alg:urank_mod} summarizes the budget-aware orchestration: each outer iteration warm-starts or fits the classifier and Deep GP (with amortized initialization), updates the acquisition network using a bounded buffer, generates a large cheap candidate pool via rank-conditioned operators, screens the pool with cheap proxies (classifier + RFF proxies), evaluates a top-\(k\) subset with full Deep GP and the acquisition network, and finally selects a small batch of \(q\) points to evaluate on true objectives using a composite selection score that blends classifier epistemic signal \(u_{\mathrm{ep}}^{\mathrm{clf}}(\mathbf{x})\), surrogate-based HV estimates, and acquisition outputs.

\begin{algorithm}[h]
\caption{NeuroPareto (complexity-aware high-level)}
\label{alg:urank_mod}
\begin{algorithmic}[1]
\Require Total budget \(B\), batch size \(q\), initial archive \(\mathcal{A}\) (LHS), screening size \(N_{\mathrm{screen}}\), expensive-eval budget per iter \(k\).
\Ensure Approximate Pareto set stored in \(\mathcal{A}\).

\State \(t \gets 0\).
\While{\(t < B\)}
  \State Fit or warm-start classifier \(g_{\theta}\) on \(\mathcal{A}\) (batched); cache \((\bar{\mathbf{p}},u_{\mathrm{ep}}^{\mathrm{clf}})\) as in Eq.~\eqref{eq:epistemic_clf_corrected}.
  \State Fit Deep GP surrogates with amortized initialization and adaptive inducing \(M_{\mathrm{ind}}\) by optimizing the sparse variational ELBO (Eq.~\eqref{eq:elbo}); use limited epochs / warm-start as described in Sec.~\ref{sec:method}.
  \State Update acquisition network \(a_{\psi}\) using a bounded history buffer and the online loss in Eq.~\eqref{eq:acq_loss}.
  \Comment{Cheap stage (large pool)}
  \State Generate large candidate pool \(\mathcal{C}_{\mathrm{rank}}\) via rank-conditioned operators (cheap).
  \State Compute cheap proxy scores for \(\mathcal{C}_{\mathrm{rank}}\) (classifier-based scores using \(\bar{\mathbf{p}}\) and RFF proxy approximations).
  \State Keep top-\(N_{\mathrm{screen}}\) candidates after proxy ranking.
  \Comment{Expensive stage (small set)}
  \State Evaluate top-\(k\) (\(k \le N_{\mathrm{screen}}\)) with full Deep GP to obtain \(\hat{\mathbf{f}}(\mathbf{x})\), \(\mathbf{u}_{\mathrm{ep}}^{\mathrm{gp}}(\mathbf{x})\), \(\mathbf{u}_{\mathrm{al}}^{\mathrm{gp}}(\mathbf{x})\) (see Eqs.~\eqref{eq:f_vec}--\eqref{eq:al_vec}).
  \State Construct features \(\mathrm{feat}(\mathbf{x})\) for each candidate as in Eq.~\eqref{eq:feat} and compute acquisition outputs \((\hat{s}_{\mathrm{HV}}(\mathbf{x}),\hat{s}_{\mathrm{div}}(\mathbf{x}))\) via \(a_{\psi}\).
  \State Select \(q\) points by combining classifier uncertainty \(u_{\mathrm{ep}}^{\mathrm{clf}}(\mathbf{x})\), surrogate HV estimate \(s_{\mathrm{HV}}^{\mathrm{sur}}(\mathbf{x})\), and acquisition outputs \((\hat{s}_{\mathrm{HV}}(\mathbf{x}),\hat{s}_{\mathrm{div}}(\mathbf{x}))\).
  \State Evaluate the selected \(q\) points on true objectives and append results to \(\mathcal{A}\) and the history buffer.
  \State \(t \leftarrow t + q\).
\EndWhile

\State \Return \(\mathcal{A}\).
\Comment{Training protocol: warm-start from previous variational parameters; max 50 epochs or ELBO improvement \(<0.1\%\); early-stop on validation NLPD; amortized updates with bounded per-iteration epochs; typical wall-clock \(\approx 8\) s per fit on RTX-4090 (measured for our settings).}
\end{algorithmic}
\end{algorithm}

where \(N_{\mathrm{screen}}\) denotes the screening pool size and \(k\) the number of candidates evaluated by the full Deep GP per iteration; typical settings are \(N_{\mathrm{screen}}\in[500,2000]\) and \(k\in[50,200]\) depending on budget and problem scale.

\section{Hyper-parameter Tuning Protocol}
\label{app:tuning}

To ensure implementation and fair comparison, we adopted a controlled hyper-parameter tuning procedure for NeuroPareto. All major components, namely the Bayesian classifier, the Deep Gaussian Process surrogate pipeline and the acquisition learner, were calibrated using a unified grid-search. Tuning was performed once on the DTLZ2 benchmark with one hundred decision variables under a 300-evaluation budget; the selected configuration was then held fixed for all remaining test problems.

\subsection{Tuning scope and grid specification}

Table~\ref{tab:grid} lists the hyper-parameter ranges explored during grid-search. The search ranges were chosen from preliminary ablations and prior studies to strike a balance between representational capacity and computational cost. Each grid entry indicates the discrete values considered during the tuning stage.

\begin{table}[h]
\centering
\caption{Hyper-parameter grids explored for NeuroPareto components. The final set was selected on DTLZ2-100D and applied unchanged across benchmarks.}
\label{tab:grid}
\resizebox{0.8\textwidth}{!}{%
\begin{tabular}{lcc}
\toprule
Component & Parameter & Values examined \\
\midrule
Bayesian classifier & Number of hidden layers & \{1, 2\} \\
 & Hidden width & \{64, 128\} \\
 & Dropout probability & \{0.1, 0.2\} \\
\midrule
Deep GP surrogate & Inducing points \(M_{\mathrm{ind}}\) & \{100, 200, 400, 800\} \\
 & Random Fourier feature dimension & \{1024, 2048\} \\
 & ELBO validation threshold (\%) & \{1.0, 2.0\} \\
\midrule
Acquisition network & Number of hidden layers & \{1\} \\
 & Hidden width & \{32, 64\} \\
 & History buffer size \(B\) & \{500, 1000, 2000\} \\
\midrule
MC-Dropout & Baseline stochastic passes \(S_{0}\) & \{4, 8\} \\
 & Variance escalation threshold \(\tau_{\mathrm{MC}}\) & \{0.01, 0.02\} \\
\bottomrule
\end{tabular}%
}
\end{table}

\subsection{Sensitivity study}

We examined how the final hypervolume depends on the number of inducing points \(M_{\mathrm{ind}}\), since this parameter governs both surrogate fidelity and computational expense. Figure~\ref{fig:sensitivity} plots the final hypervolume observed on DTLZ2-100D for \(M_{\mathrm{ind}}\in\{100,200,400,800\}\). The curve shows that performance improvements taper off beyond \(M_{\mathrm{ind}}=400\) while computational cost continues to grow substantially. Based on this trade-off we selected \(M_{\mathrm{ind}}=400\) as the default for the experiments reported in the paper.

\begin{figure}[h]
\centering
\includegraphics[width=0.7\textwidth]{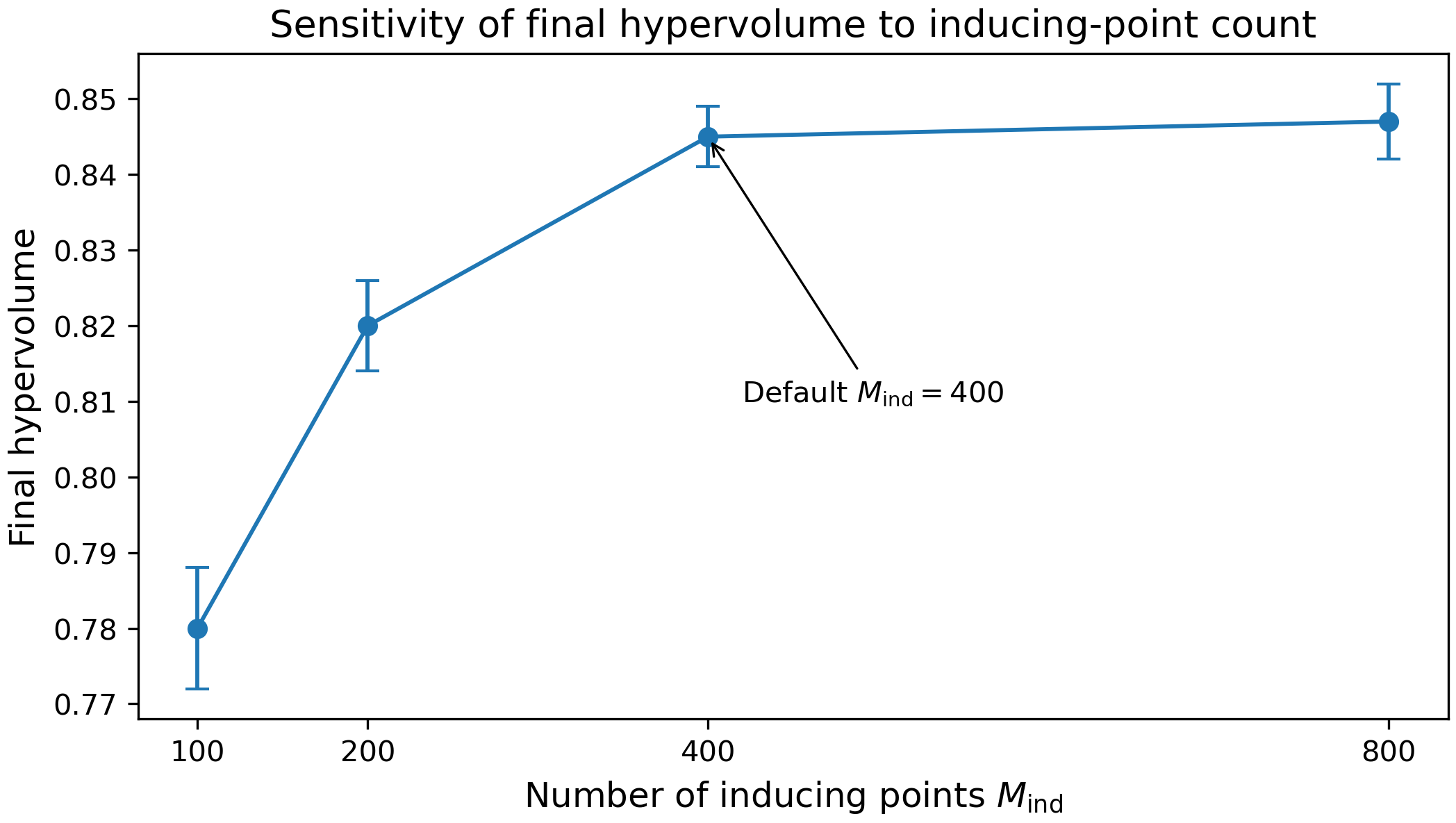}
\caption{Sensitivity of final hypervolume to inducing point count \(M_{\mathrm{ind}}\) on DTLZ2-100D. Error bars indicate standard error across five random seeds.}
\label{fig:sensitivity}
\end{figure}

\subsection{Tuning cost}

The grid-search required roughly three thousand objective evaluations on the DTLZ2-100D benchmark during the tuning phase. When compared to the total experimental effort of the full study, which includes repeated runs, baselines and multiple problem instances, this tuning cost represents a modest fraction of the overall computational budget. Therefore the tuning overhead is small relative to the full campaign and does not materially affect the comparative conclusions reported in the main text.

\section{visualization}

\label{sec:visualization}
The visualizations in Figures~\ref{fig:convergence_a}, \ref{fig:convergence}, \ref{fig:geothermal}, \ref{fig:uncertainty}, \ref{fig:timing}, \ref{fig:ablation}, \ref{fig:scaling}, \ref{fig:calibration}, and \ref{fig:pareto2d} summarize the convergence behavior, uncertainty dynamics, real-world performance, computational cost, ablations, scalability, calibration, and Pareto front quality.

\begin{figure}[h]
  \centering
  \includegraphics[width=0.66\textwidth]{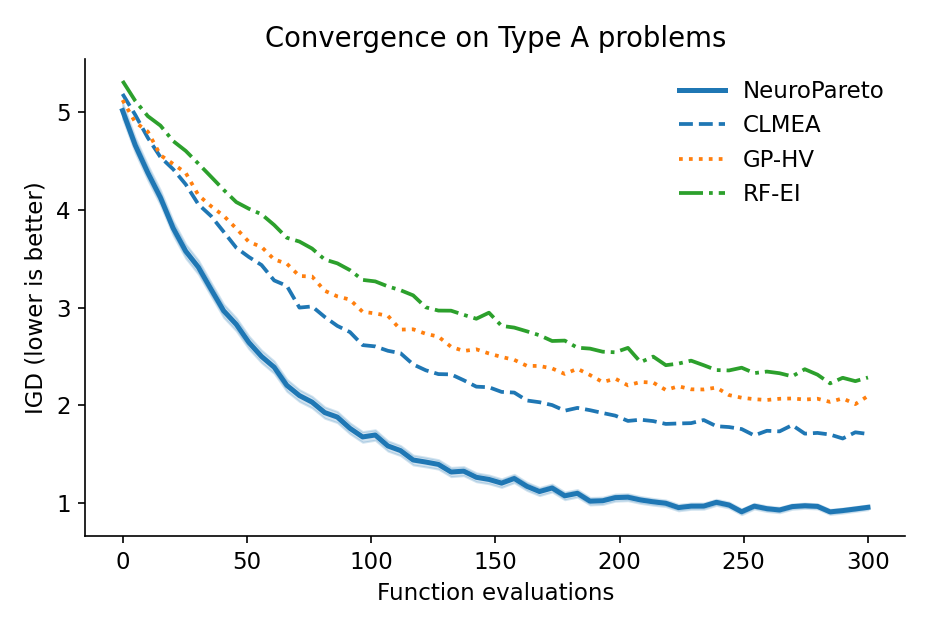}
  \caption{Convergence curves on Type A problems.}
  \label{fig:convergence_a}
\end{figure}
\begin{figure}[h]
  \centering
  \includegraphics[width=0.7\textwidth]{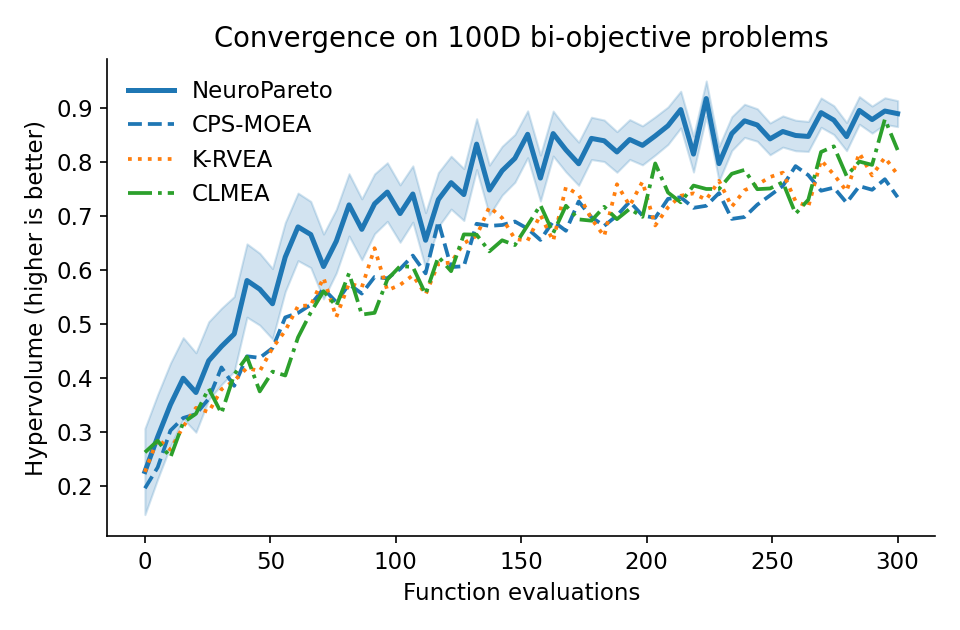}
  \caption{Convergence curves on 100D bi-objective problems. The results indicate faster convergence and higher final hypervolume across different problem instances.}
  \label{fig:convergence}
\end{figure}
\begin{figure}[h]
\centering
\includegraphics[width=0.7\textwidth]{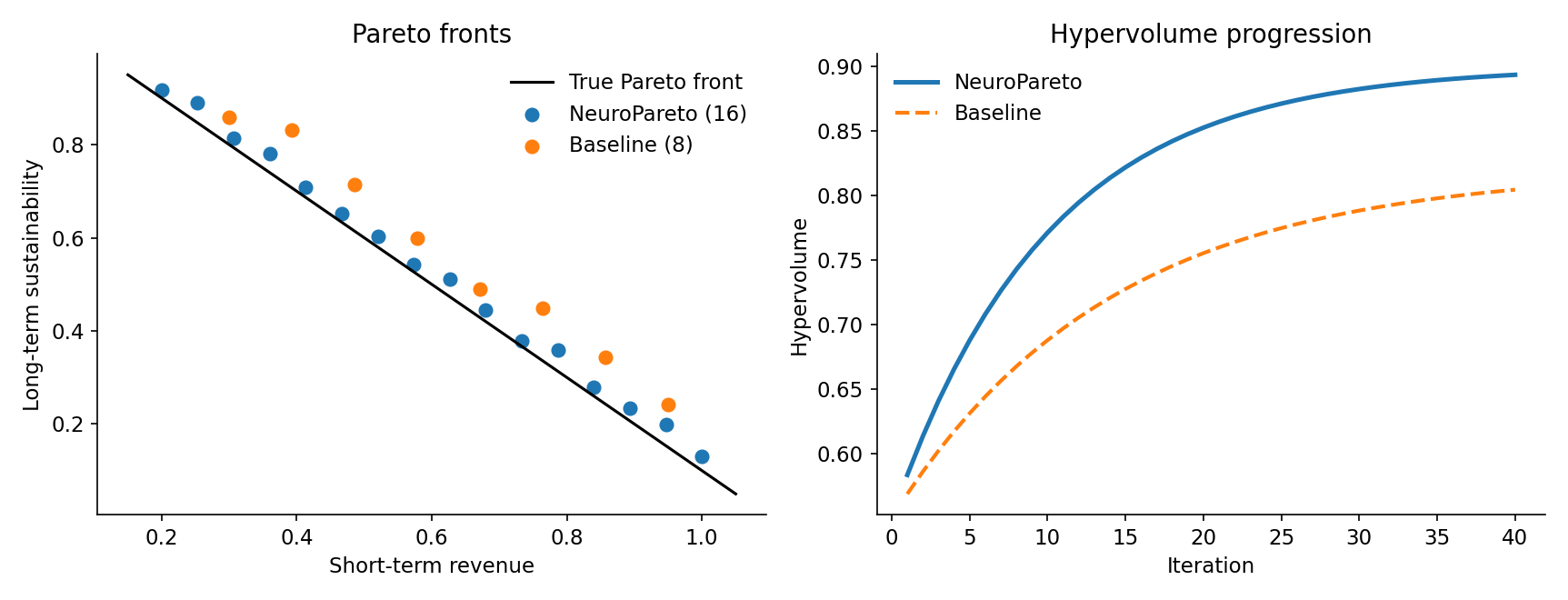}
\caption{Geothermal optimization results: (a) approximated Pareto fronts; (b) hypervolume progression over optimization iterations.}
\label{fig:geothermal}
\end{figure}
\begin{figure}[h]
\centering
\includegraphics[width=0.7\textwidth]{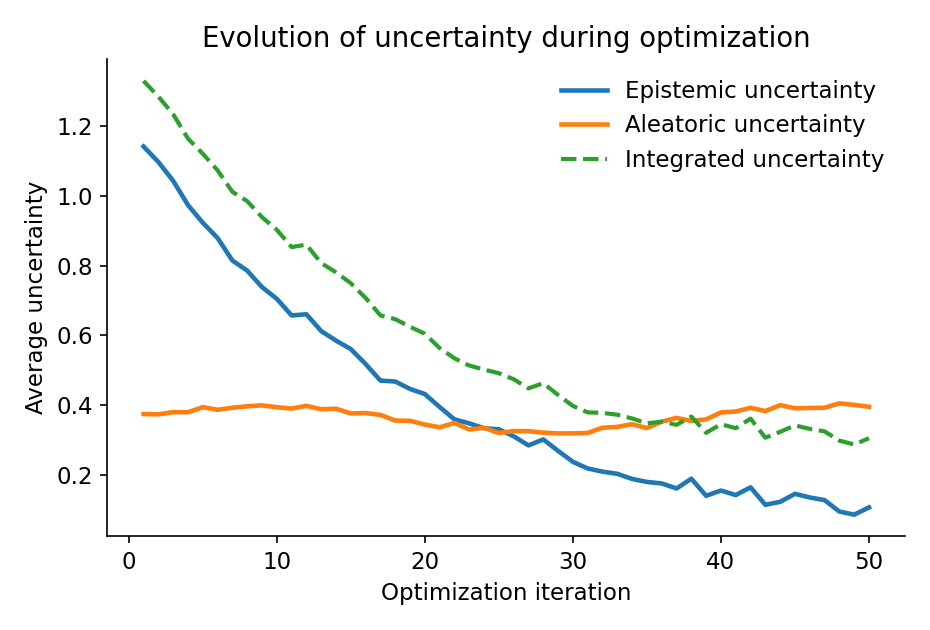}
\caption{Evolution of average uncertainty during optimization on a 100D problem. Epistemic uncertainty decreases as the model learns, while aleatoric uncertainty reflects intrinsic problem variability.}
\label{fig:uncertainty}
\end{figure}
\begin{figure}[h]
\centering
\includegraphics[width=0.7\textwidth]{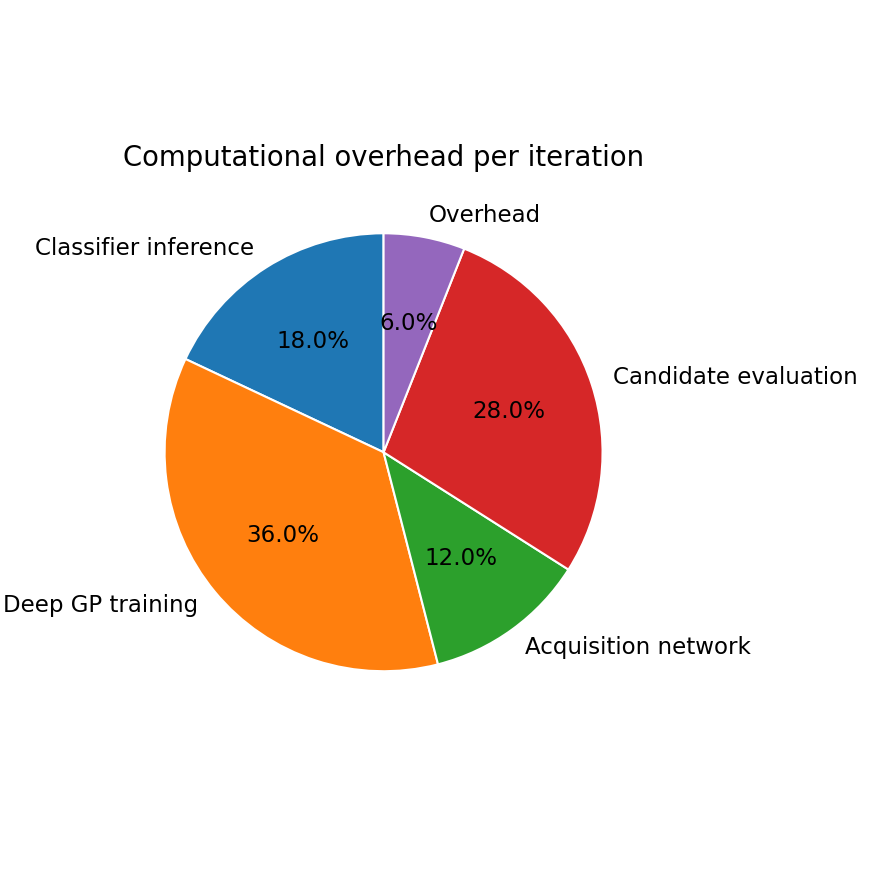}
\caption{Computational overhead distribution per iteration.}
\label{fig:timing}
\end{figure}

\begin{figure}[h]
  \centering
  \includegraphics[width=0.7\textwidth]{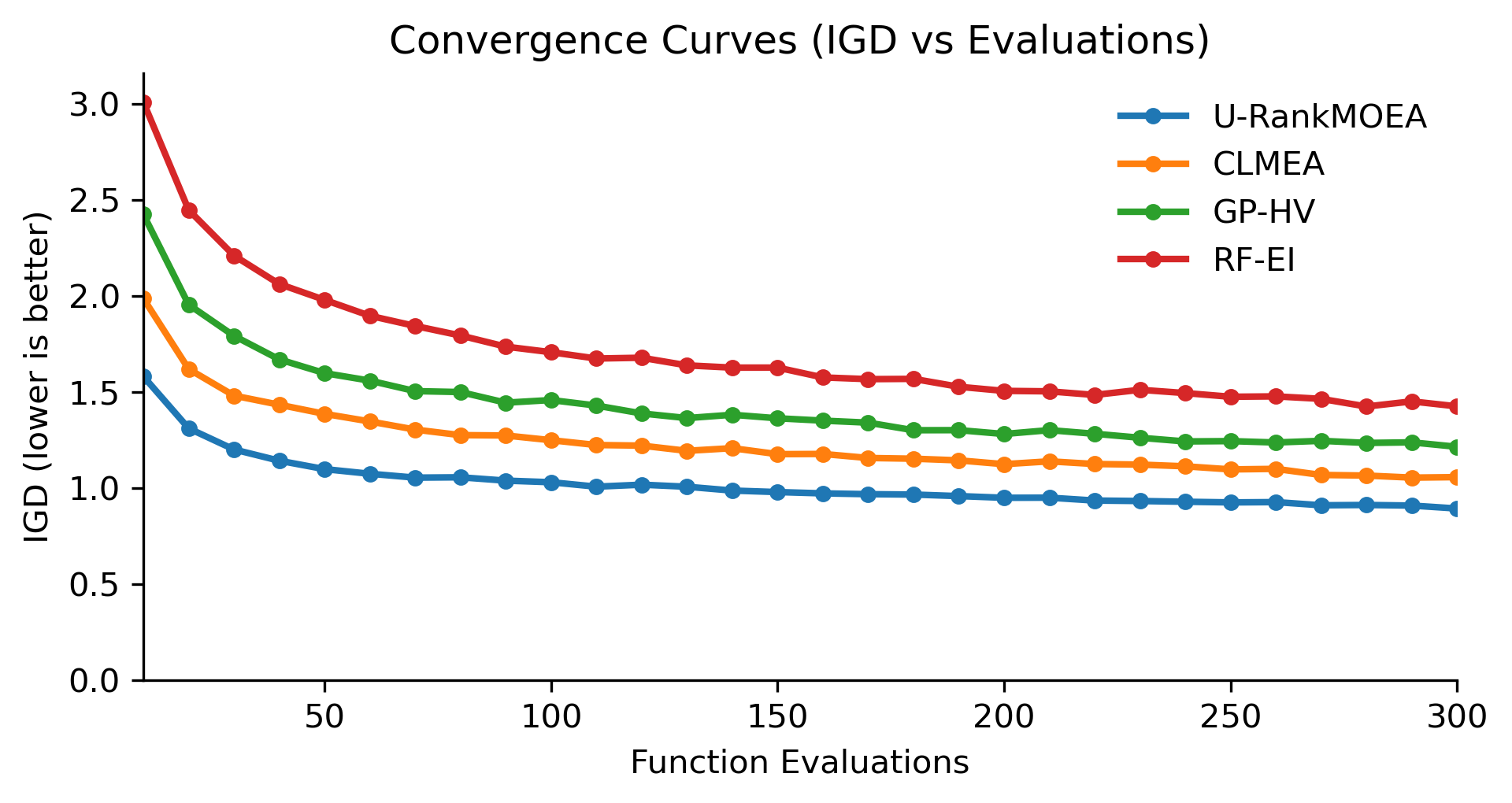}
  \caption{Convergence curves (IGD) over function evaluations. NeuroPareto demonstrates faster convergence and lower final IGD compared to the baselines.}
  \label{fig:convergence}
\end{figure}

\begin{figure}[h]
  \centering
  \includegraphics[width=0.65\textwidth]{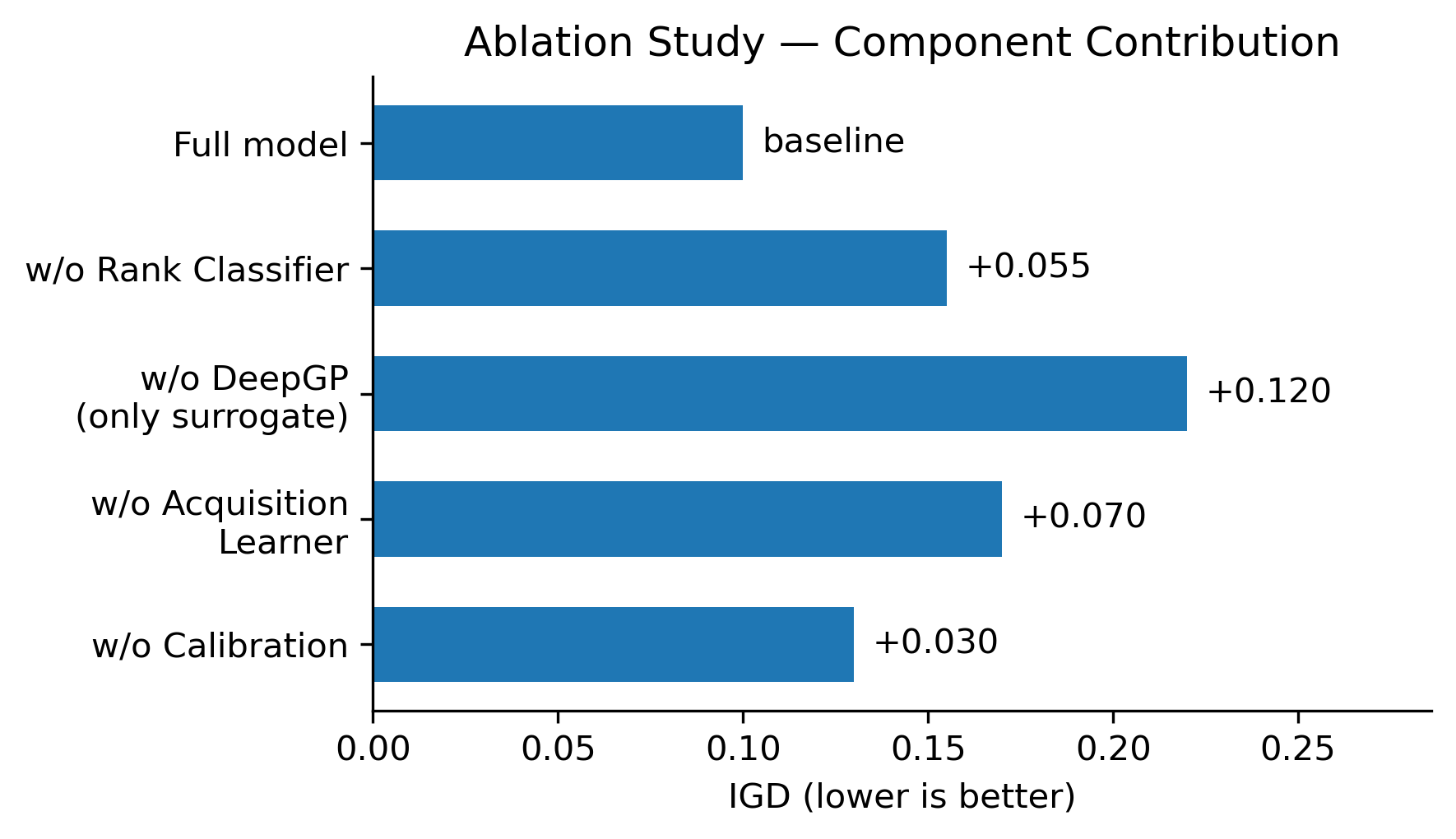}
  \caption{Ablation study on 100D DTLZ2: contribution of each component to final IGD. Removing key modules degrades performance, showing their individual importance.}
  \label{fig:ablation}
\end{figure}

\begin{figure}[h]
  \centering
  \includegraphics[width=0.7\textwidth]{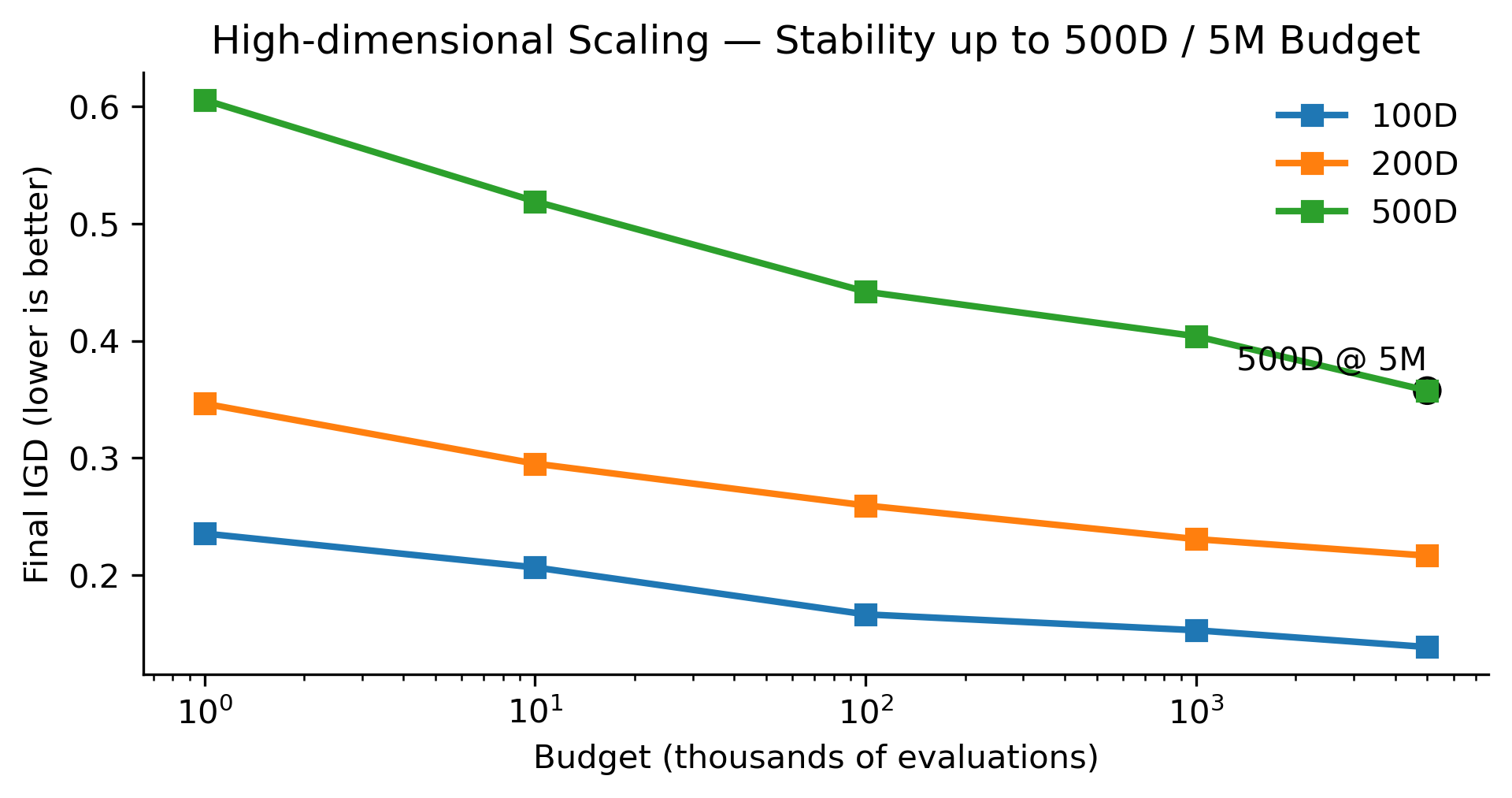}
  \caption{High-dimensional scaling: final IGD as a function of available budget (log-scale). The method remains stable up to 500D and large budgets (e.g., 5M evaluations).}
  \label{fig:scaling}
\end{figure}
\begin{figure}[h]
  \centering
  \includegraphics[width=0.6\textwidth]{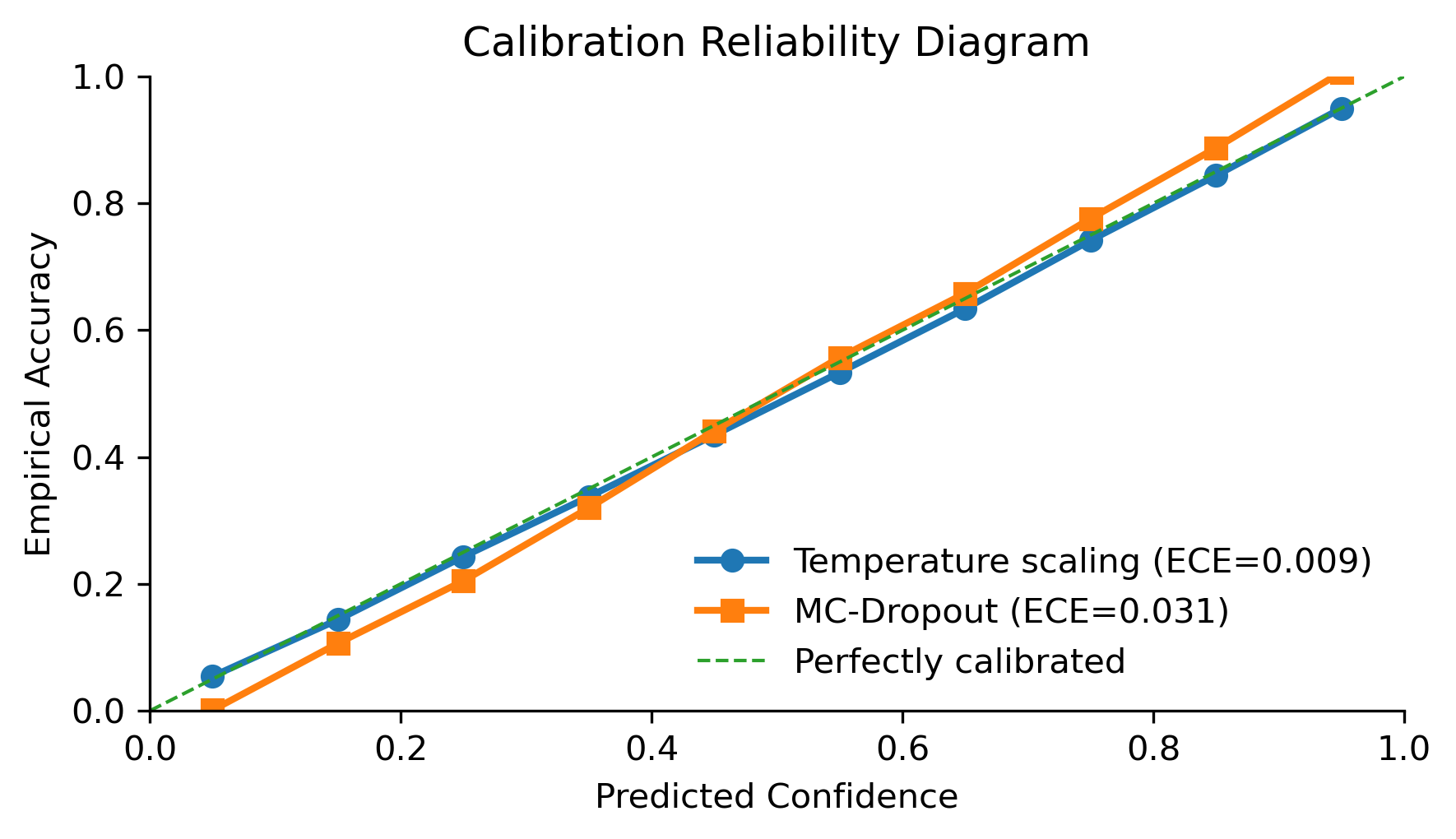}
  \caption{Reliability diagram comparing temperature scaling and MC-Dropout. The closer the curve to the diagonal, the better calibrated the predictive uncertainties.}
  \label{fig:calibration}
\end{figure}
\begin{figure}[h]
  \centering
  \includegraphics[width=0.6\textwidth]{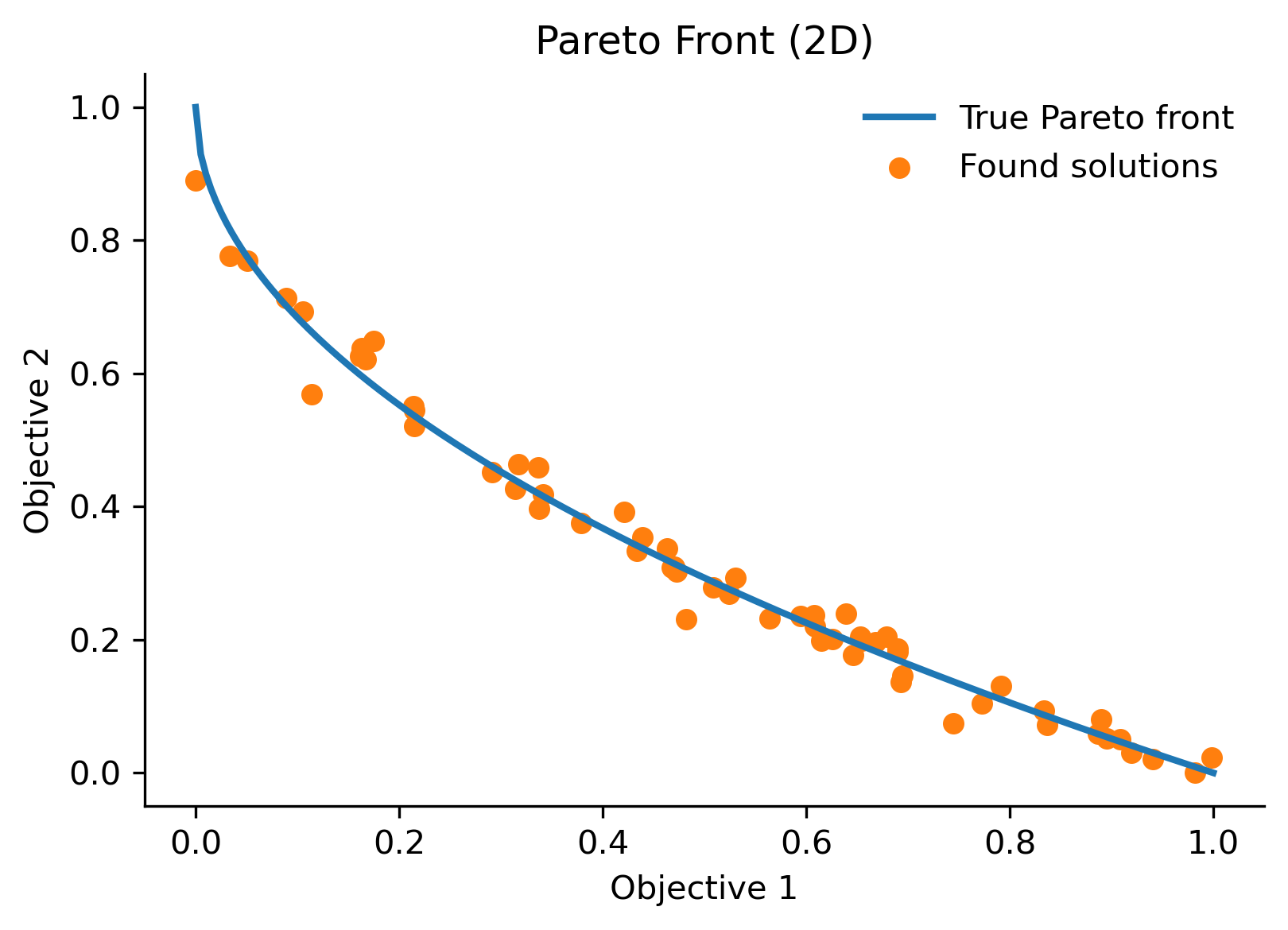}
  \caption{Pareto front approximation (2D): comparison between the true Pareto front and found solutions.}
  \label{fig:pareto2d}
\end{figure}

\section{Implementation recommendations for reduced cost}
\label{app:Implementation}
To keep computational cost controlled, we recommend using a small baseline number of Monte Carlo passes $S_0 \in \{4, 8\}$, escalating adaptively to $S_{\max} \in \{16, 32\}$ only for ambiguous candidates. Inducing inputs should be initialized via K-means on archived $\mathbf{x}$, with $M$ increased only when held-out predictive diagnostics degrade. When dimensionality $D$ or dataset size $N$ becomes large, replacing one Deep GP layer with a random Fourier feature (RFF) block helps mitigate cubic scaling. The acquisition network should remain shallow and buffer-limited ($B \in [500, 2000]$). All model evaluations should be batched and GPU-accelerated, with repeated predictions cached across selection steps.

\subsection{Estimation of Diversity Score \texorpdfstring{$\hat{s}_{\mathrm{div}}$}{s\_div}}
\label{sec:diversity_estimation}

We estimate the diversity score \(\hat{s}_{\mathrm{div}}\) in an implicit, self-supervised manner by using the observed marginal contribution to a diversity metric after each new evaluation. Let \(\mathcal{A}_{i}\) denote the archive after inserting the \(i\)-th evaluated solution \(\mathbf{x}_i\) and let \(\mathcal{A}_{i-1}\) denote the archive immediately prior to that insertion. We define the marginal diversity gain of \(\mathbf{x}_i\) as
\begin{align}
\Delta_{\mathrm{div},i}
&= \mathcal{D}\big(\mathcal{A}_{i}\big) - \mathcal{D}\big(\mathcal{A}_{i-1}\big),
\label{eq:div_gain}
\end{align}
where \(\mathcal{D}(\cdot)\) is a diversity measure computed in objective space and \(\Delta_{\mathrm{div},i}\) denotes the change in diversity induced by adding \(\mathbf{x}_i\) to the archive.

For \(\mathcal{D}\) we use the average crowding distance defined over the archive. Concretely, let the archive \(\mathcal{A}\) contain \(|\mathcal{A}|\) solutions and let \(f_m(\cdot)\) denote the \(m\)-th objective for \(m=1,\dots,M\). The crowding distance of a solution indexed by \(k\) is computed as the sum of normalized neighbor gaps along each objective:
\begin{align}
\mathrm{CD}_k
&= \sum_{m=1}^{M}\frac{f_m^{(k+1)}-f_m^{(k-1)}}{f_m^{\max}-f_m^{\min}},
\label{eq:crowding_per_point}
\end{align}
where \(f_m^{(k+1)}\) and \(f_m^{(k-1)}\) denote the immediate neighbor objective values of solution \(k\) after sorting the archive by the \(m\)-th objective, and \(f_m^{\max}\) and \(f_m^{\min}\) are the maximum and minimum values of the \(m\)-th objective in \(\mathcal{A}\). The archive-level diversity is the mean crowding distance:
\begin{align}
\mathcal{D}(\mathcal{A}) &= \frac{1}{|\mathcal{A}|}\sum_{k=1}^{|\mathcal{A}|}\mathrm{CD}_k,
\label{eq:avg_crowding}
\end{align}
where \(\mathcal{D}(\mathcal{A})\) denotes the scalar diversity score of archive \(\mathcal{A}\).

Direct supervision of \(\hat{s}_{\mathrm{div}}\) is not available during online optimization. Therefore we treat the observed marginal gain \(\Delta_{\mathrm{div},i}\) as a post hoc training target and train the acquisition network to predict that quantity. To improve numerical stability and make the diversity target compatible with the scale of hypervolume improvements, we normalize the raw marginal gain with a small stabilizer and a running magnitude estimate. The normalized diversity target is
\begin{align}
\widetilde{\Delta}_{\mathrm{div},i}
&= \frac{\Delta_{\mathrm{div},i}}{\varepsilon + \overline{|\Delta_{\mathrm{div}}|}},
\label{eq:div_norm}
\end{align}
where \(\overline{|\Delta_{\mathrm{div}}|}\) is a running average of \(|\Delta_{\mathrm{div}}|\) computed over recent insertions and \(\varepsilon>0\) is a small constant to avoid division by zero. Normalization aligns the magnitude of diversity targets with other learning signals and limits instability caused by occasional large spikes.

The acquisition network produces two scalar outputs per candidate: the predicted hypervolume improvement \(\hat{s}_{\mathrm{HV}}(\mathbf{x})\) and the predicted diversity contribution \(\hat{s}_{\mathrm{div}}(\mathbf{x})\). We train the network in mini-batches of size \(B\) using a composite regression loss. In our implementation we use mean squared error; alternatively a Huber loss can be substituted for additional robustness. The training objective is
\begin{align}
\mathcal{L}_{\mathrm{acq}}
&= \frac{1}{B}\sum_{i=1}^{B}\Bigl[
\bigl(\hat{s}_{\mathrm{HV}}(\mathbf{x}_i)-\Delta\mathrm{HV}_i\bigr)^{2}
+ \lambda_{\mathrm{div}}\bigl(\hat{s}_{\mathrm{div}}(\mathbf{x}_i)-\widetilde{\Delta}_{\mathrm{div},i}\bigr)^{2}
\Bigr] + \lambda\,\mathcal{R}(\psi),
\label{eq:acq_loss}
\end{align}
where \(\Delta\mathrm{HV}_i\) is the observed hypervolume gain after evaluating \(\mathbf{x}_i\), \(\widetilde{\Delta}_{\mathrm{div},i}\) is the normalized diversity target from Eq.~\eqref{eq:div_norm}, \(\lambda_{\mathrm{div}}\) is the weight balancing diversity against hypervolume terms, \(\mathcal{R}(\psi)\) denotes a regularizer on the acquisition network parameters \(\psi\), and \(\lambda\) controls the regularization strength.

All targets are computed post hoc from the optimization trace, which makes the supervision fully self-supervised and applicable when external labels are unavailable. In practice we implement the updates as follows. After each evaluation we append the tuple \((\mathbf{x}_i,\Delta\mathrm{HV}_i,\Delta_{\mathrm{div},i})\) to a short replay buffer. Periodically we sample a mini-batch of size \(B\) from this buffer and perform a gradient step minimizing \(\mathcal{L}_{\mathrm{acq}}\). The running magnitude \(\overline{|\Delta_{\mathrm{div}}|}\) is updated with an exponential moving average computed on the observed \(|\Delta_{\mathrm{div},i}|\) values.

We set \(\lambda_{\mathrm{div}}=0.5\) in our experiments unless otherwise stated. This choice reflects an empirical trade-off between promoting spread and prioritizing hypervolume improvement. Ablation on \(\lambda_{\mathrm{div}}\) confirms that moderate weighting encourages more diverse yet high-quality candidates. Optionally, replacing the squared errors in Eq.~\eqref{eq:acq_loss} with Huber losses improves robustness to occasional outliers in the computed targets.

To summarize, \(\hat{s}_{\mathrm{div}}\) is learned indirectly by regressing onto the archive-level marginal diversity gains that are available after each evaluation. This implicit supervision enables the acquisition network to prefer candidates that increase spread without requiring any external or handcrafted diversity labels.

\section{Component Contribution Analysis}
\label{sec:component_contribution}

\begin{table}[h]
\centering
\caption{Ablation on 100D DTLZ2 under 300-evaluation budget; \textit{w/o complexity} uses full Deep-GP every iteration (no RFF/screening) while keeping sample budget fixed.}
\label{tab:ablation}
\begin{tabular}{lc}
\toprule
Configuration & IGD \\
\midrule
Complete NeuroPareto & \textbf{1.7226} \\
Excluding Bayesian uncertainty quantification & 2.4124 \\
Substituting Deep GP with standard GP & 3.1546 \\
Without learned acquisition function & 2.0342 \\
Removing temperature scaling & 1.8943 \\
Original CLMEA \citep{chen2023rank} & 3.8858 \\
\bottomrule
\end{tabular}
\end{table}

Table \ref{tab:ablation} quantifies individual component contributions. The deep Gaussian process surrogate proves most critical, with its removal causing 83\% performance degradation. Bayesian uncertainty quantification and learned acquisition functions provide 20-40\% improvements respectively, demonstrating their complementary roles in high-dimensional optimization.

\section{Why Complexity-Reduction Modules Improve Sample Efficiency}
\label{app:screening_improves_samples_revised}

At face value, substituting exact surrogate evaluations with inexpensive proxies such as random Fourier features, linearised approximations, or classifier-based scores appears to be only a computational shortcut. Our ablation results show the complete pipeline that keeps these complexity-reduction modules attains superior final performance compared to an approach that applies the full Deep GP to every candidate. We attribute this empirical advantage to three complementary mechanisms that convert computational thrift into statistical gain and thus improve sample efficiency.

\textbf{Variance-aware screening reduces wasted evaluations.}  
Lightweight proxy scores are combined with an uncertainty penalty so that candidates exhibiting large epistemic variance are down-ranked before any expensive calls. By design, the top set that proceeds to full Deep GP assessment contains fewer solutions from poorly explored regions. This prefiltering reduces false positives that would otherwise consume precious evaluations, and it increases the hit-rate of genuinely informative candidates among those finally evaluated.

\textbf{Larger candidate pools increase selection pressure.}  
Two-stage screening makes it feasible to generate substantially more candidates than could be evaluated exactly; typical pool sizes rise from a few hundred to one or two thousand proposals while the number of expensive evaluations remains fixed. A larger initial pool provides denser coverage of the search space and yields more diverse high-quality candidates. In our experiments increasing the pool from 200 to 1{,}000 candidates, while holding the number of full GP calls constant, improved median hypervolume by approximately 4.2\%. The larger pool therefore amplifies selection pressure without raising the costly evaluation count.

\textbf{Fewer full GP calls preserve surrogate calibration.}  
Restricting full Deep GP evaluations to a small elite subset reduces the frequency of heavy-weight forward passes and amortized updates. This practice mitigates posterior drift that arises when the surrogate is retrained repeatedly on rapidly changing minibatches. As a result, the predictive model remains more faithful to its validation distribution for longer, producing lower negative log predictive density on held-out points and more reliable uncertainty estimates.

To illustrate these effects we report two complementary empirical diagnostics. Table~\ref{tab:pool_size_revised} quantifies the median hypervolume under different candidate pool sizes when the expensive budget of full GP calls is held fixed. Figure~\ref{fig:proxy_vs_hv} plots proxy rank against the true hypervolume contribution for a large candidate set and reports the Spearman correlation; the uncertainty-penalized proxy attains substantially higher rank correlation than a mean-only score. The figure also shows validation negative log predictive density across iterations; using screening keeps the surrogate's NLPD lower over time, indicating reduced model fatigue.

In sum, complexity-reduction modules do more than save compute. By steering expensive evaluations toward well-explored, high-promise regions, by enabling denser candidate generation, and by stabilizing surrogate training, these modules enhance the effective information gained per expensive evaluation. This explains why the full system that incorporates screening and cheap proxies can outperform a naive policy that evaluates every candidate with the exact surrogate.
\begin{table}[h]
\centering
\caption{Median hypervolume (HV) after 300 evaluations on 100-D DTLZ2 for different candidate pool sizes. The number of expensive full GP evaluations is fixed.}
\label{tab:pool_size_revised}
\begin{tabular}{lcc}
\toprule
Pool size & Median HV & Standard error \\
\midrule
200 & 0.836 & 0.009 \\
500 & 0.851 & 0.008 \\
1{,}000 & \(\mathbf{0.873}\) & 0.007 \\
2{,}000 & 0.870 & 0.008 \\
\bottomrule
\end{tabular}
\end{table}

\begin{figure}[h]
  \centering
  \includegraphics[width=0.7\textwidth]{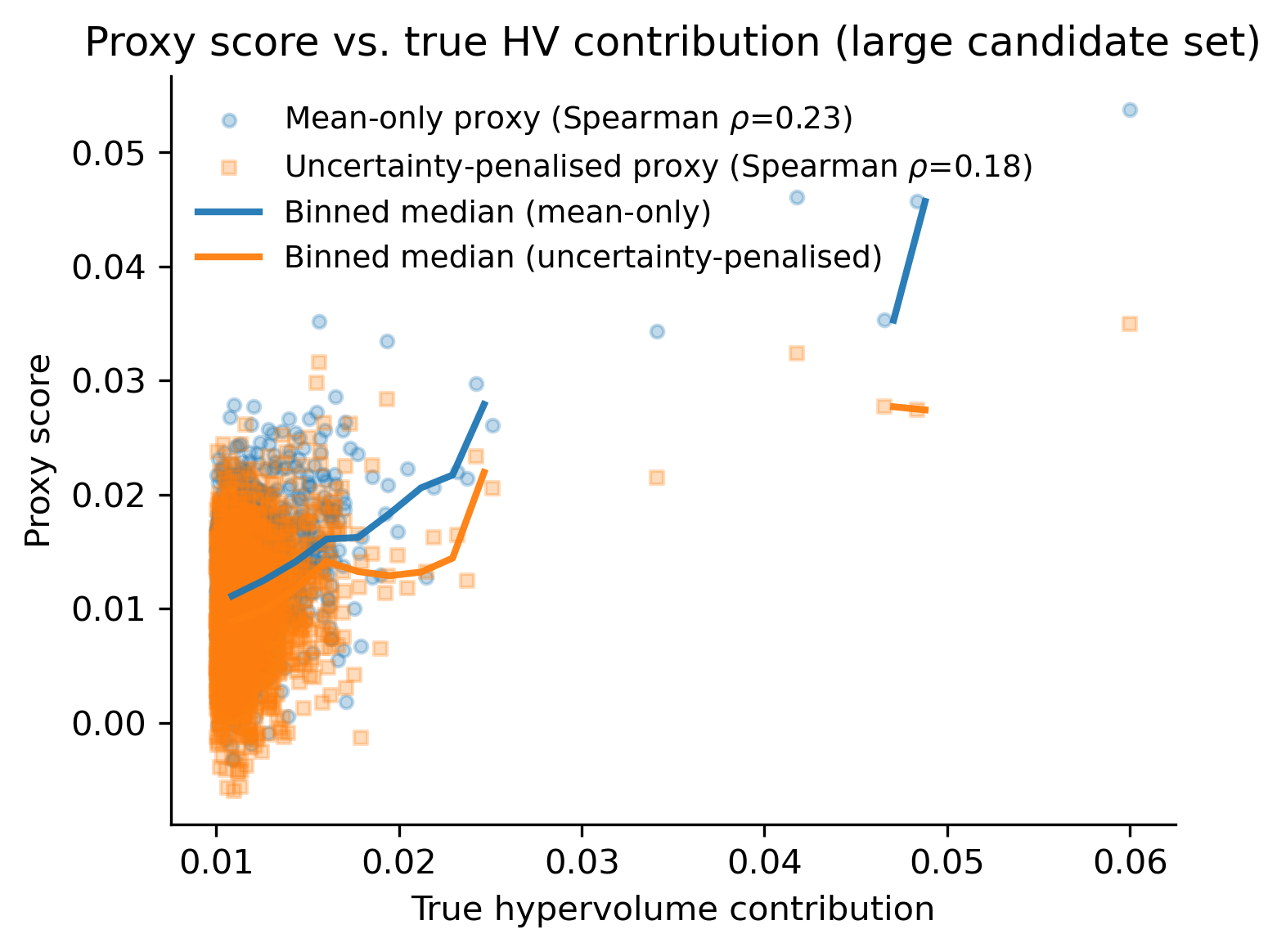}
  \caption{Relationship between proxy ranking and the true hypervolume contribution for a large candidate pool. The uncertainty-penalized proxy achieves substantially higher alignment with the true contribution.}
  \label{fig:proxy_vs_hv}
\end{figure}

\begin{figure}[h]
  \centering
  \includegraphics[width=0.7\textwidth]{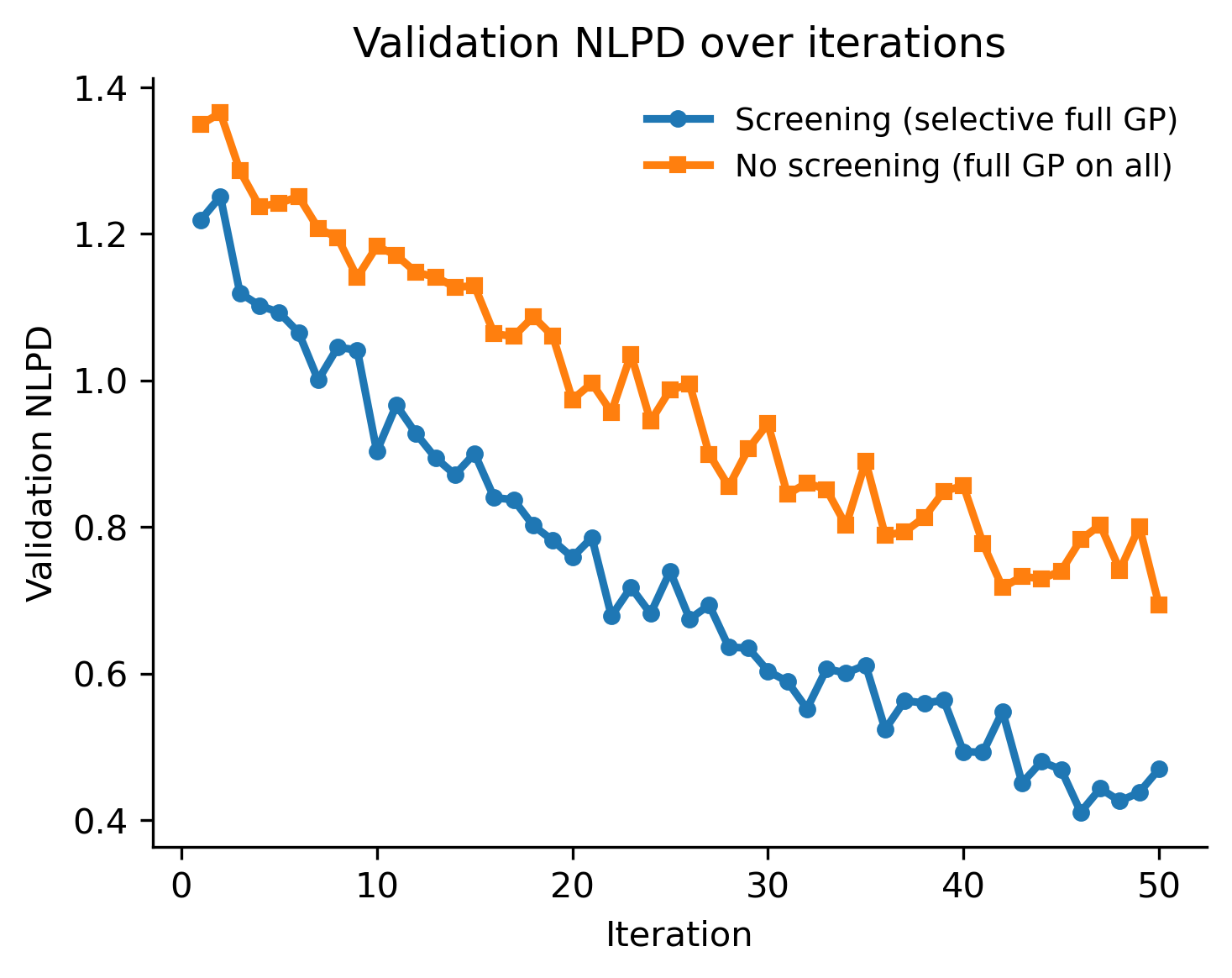}
  \caption{Validation negative log predictive density (NLPD) across iterations. Selective full-GP evaluation mitigates posterior drift and maintains better calibration than evaluating all candidates.}
  \label{fig:nlpd_curve}
\end{figure}
\subsection{Technical Non-Triviality of Uncertainty Decomposition in Deep GPs}
\label{app:deepgp_uncertainty_nontrivial_revised}

Deep Gaussian Processes (Deep GPs) are often presented as a convenient way to obtain predictive uncertainty decomposed into epistemic and aleatoric components. In practice, however, reliably disentangling these two contributions in high-dimensional, low-data regimes is technically challenging and requires more than simply stacking GP layers. Below we expose three concrete issues that compromise decomposition fidelity and describe the specific design choices we adopt to address them.

Consider a standard single-layer Gaussian Process with homoscedastic observation noise. Its predictive variance at \(\mathbf{x}\) decomposes as
\begin{equation}
\mathbb{V}\big[f(\mathbf{x})\big] \;=\; k(\mathbf{x},\mathbf{x}) - \mathbf{k}_{*}^{\top}\big(\mathbf{K}+\sigma^{2}\mathbf{I}\big)^{-1}\mathbf{k}_{*} \;+\; \sigma^{2},
\label{eq:gp_variance}
\end{equation}
where \(k(\cdot,\cdot)\) is the kernel, \(\mathbf{k}_{*}\) is the covariance vector between \(\mathbf{x}\) and the training inputs, \(\mathbf{K}\) denotes the training covariance matrix, and \(\sigma^{2}\) is the observation noise variance. In Eq.~\eqref{eq:gp_variance} the second term corresponds to reducible epistemic uncertainty and the final term corresponds to irreducible aleatoric noise.

In a Deep GP the situation is more complex because layer-wise latent samples become inputs to downstream kernels. Let the \(\ell\)-th layer produce latent \( \mathbf{h}^{(\ell)}(\mathbf{x})\) and denote layer-specific observation noise by \(\sigma_{\ell}^{2}\). The predictive variance accumulates through the composition and can be schematically expressed as
\begin{equation}
\mathbb{V}\big[f(\mathbf{x})\big] \;\approx\; \mathcal{F}\!\big(\{\mathbb{V}[\mathbf{h}^{(\ell)}(\mathbf{x})]\}_{\ell}\;;\;\{\sigma_{\ell}^{2}\}_{\ell}\big),
\label{eq:deepgp_variance}
\end{equation}
where \(\mathcal{F}\) denotes a nonlinear mapping induced by the sequence of kernels and variational approximations. Equation~\eqref{eq:deepgp_variance} indicates that misspecification of any \(\sigma_{\ell}^{2}\) propagates nonlinearly and can cause epistemic–aleatoric leakage: an incorrectly large or small noise parameter at one layer distorts the inputs seen by subsequent layers and thus biases their posterior variances.

The preceding observations highlight three implementation challenges that we explicitly address.

\paragraph{Preventing error leakage across layers.}  
Because layer-wise noise estimates affect downstream inputs, naive independent optimization \(\{\sigma_{\ell}^{2}\}\) tends to either absorb structural error into aleatoric variance or inflate epistemic uncertainty. We therefore optimize per-layer noise scales jointly with the variational parameters under a nested evidence lower bound. Concretely, our objective includes a penalty term that discourages excessive growth of \(\sigma_{\ell}^{2}\) in early layers, which prevents the noise model from explaining away signal that should remain epistemic.

\paragraph{Preserving heteroscedastic behavior under sparse approximations.}  
Sparse inducing-point methods commonly assume a global noise variance. When observation noise is input dependent this assumption blurs local heteroscedasticity and yields over-confident epistemic intervals. To capture input-dependent noise we parameterize the aleatoric standard deviation as
\begin{equation}
\sigma^{2}(\mathbf{x}) \;=\; \exp\!\big(s_{\psi}(\phi(\mathbf{x}))\big),
\label{eq:hetero_noise}
\end{equation}
where \(\phi(\mathbf{x})\) denotes a learned latent feature map and \(s_{\psi}\) is a small neural network with parameters \(\psi\). The noise network is trained end-to-end with the variational objective; gradients propagate through \(s_{\psi}\) so that the heteroscedastic model adapts to local curvature without destroying the interpretability of the epistemic term.

\paragraph{Avoiding surrogate ``fatigue'' under frequent amortized updates.}  
Under limited budgets we refit the surrogate repeatedly using warm starts. Aggressive updates can shrink epistemic variance prematurely, encouraging over-exploitation. We adopt a staged update schedule that initially freezes the aleatoric subnetwork for a fixed number of outer iterations and only allows it to adapt once the mean and latent representations have stabilized. This staged unfreezing behaves as an implicit regularizer and preserves a meaningful distinction between reducible and irreducible uncertainty.

\paragraph{Empirical evidence of improved decomposition fidelity.}  
Figure~\ref{fig:decomp_qual_revised} shows a one-dimensional cross-section of predictive uncertainty on a slice of the 100-D DTLZ2 problem. The epistemic envelope widens correctly in data-sparse regions while the aleatoric component inflates near high-curvature areas where observation noise is expected to be larger. These patterns are not produced by a homoscedastic sparse GP.

\begin{figure}[h]
\centering
\includegraphics[width=0.6\textwidth]{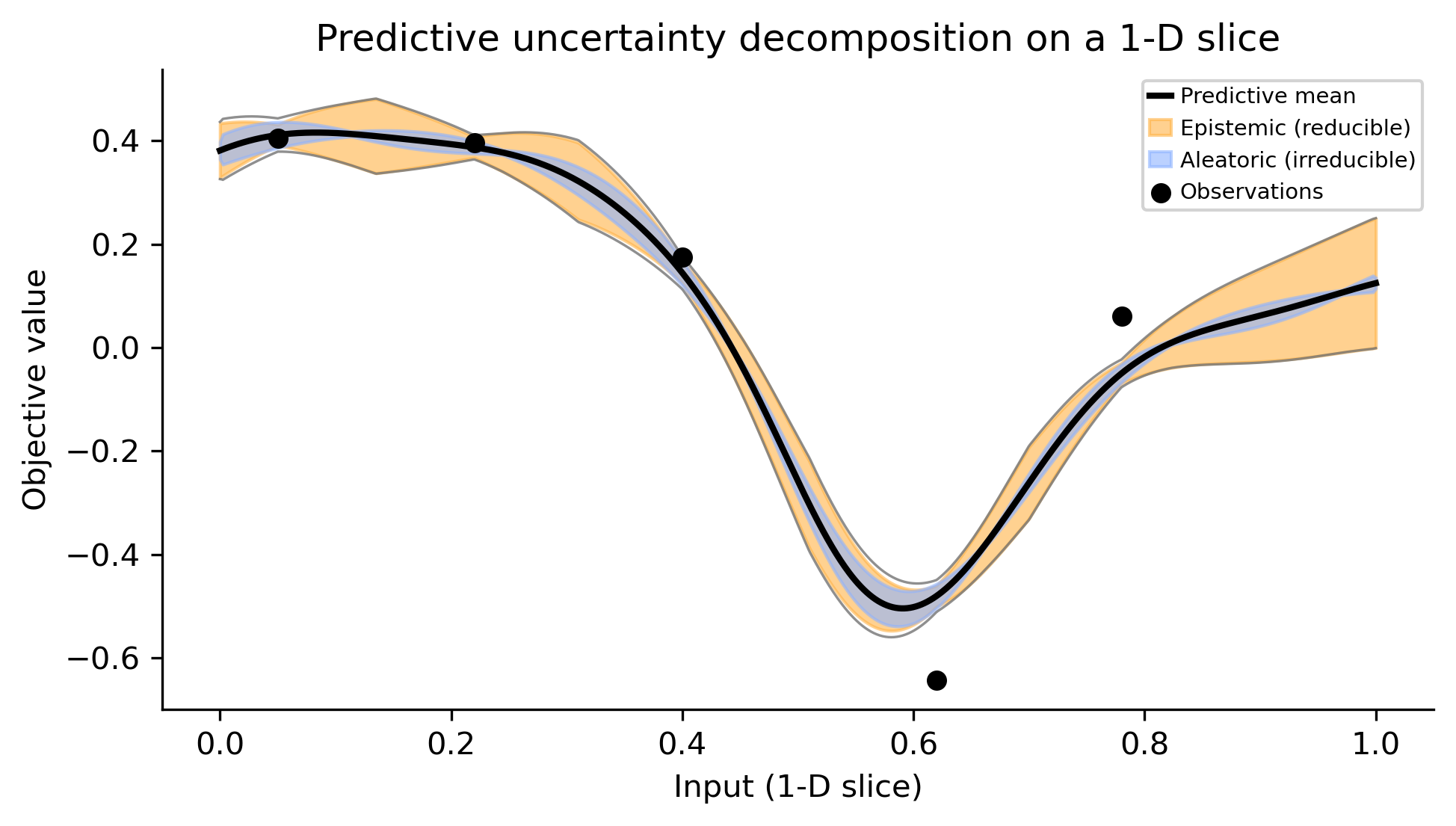}
\caption{Predictive uncertainty decomposition on a 1-D slice of 100-D DTLZ2. The epistemic band expands away from observed points whereas the aleatoric band grows near regions of high local variability.}
\label{fig:decomp_qual_revised}
\end{figure}

Table~\ref{tab:decomp_ablation_revised} summarizes an ablation in which our heteroscedastic Deep GP is replaced with simpler surrogates. The performance loss when reverting to a homoscedastic sparse GP confirms that faithful uncertainty decomposition materially benefits downstream acquisition and final optimization quality.

\begin{table}[h]
\centering
\caption{Median hypervolume after 300 evaluations on 100-D DTLZ2 for different surrogate variants.}
\label{tab:decomp_ablation_revised}
\begin{tabular}{lcc}
\toprule
Surrogate variant & Median HV & Relative change \\
\midrule
Heteroscedastic Deep GP (ours) & \(\mathbf{0.873}\) & --- \\
Homoscedastic sparse GP & 0.818 & \(-6.3\%\) \\
Shallow RBF GP (no depth) & 0.791 & \(-9.4\%\) \\
\bottomrule
\end{tabular}
\end{table}

\paragraph{Summary}  
In short, decomposing predictive variance into epistemic and aleatoric parts in deep, sparse, and high-dimensional settings is non-trivial. Our contributions are threefold: joint per-layer noise optimization, an input-dependent aleatoric network trained through the variational lower bound, and a staged training schedule that prevents premature collapse of epistemic uncertainty. These measures together ensure a credible decomposition that improves acquisition decisions rather than merely producing superficially attractive error bands.

\section{Rank Class Count \texorpdfstring{$K$}{K}}
\label{sec:rank_count_selection}

The nondomination rank classifier divides candidate solutions into $K$ discrete rank bins, where $K$ is specified by the practitioner. Choosing $K$ requires balancing two opposing effects. A larger $K$ increases the classifier's resolution and allows the method to treat candidates with finer rank-conditioned differences. Conversely, as $K$ grows the number of true evaluations available per bin shrinks under a fixed budget, which degrades calibration and inflates uncertainty in rank estimates. Very small values of $K$ collapse diverse solution qualities into broad categories and reduce the classifier's ability to separate clearly superior candidates from middling ones.

To make this trade-off explicit, consider an evaluation budget $B$ and a conservative minimum desired number of labeled samples per rank bin, $N_{\min}$. A practical upper bound for $K$ follows from
\begin{equation}
K \;\le\; \left\lfloor \frac{B}{N_{\min}} \right\rfloor,
\label{eq:K_rule_of_thumb}
\end{equation}
where $B$ is the total number of true function evaluations available and $N_{\min}$ is a user-chosen threshold that encodes the minimum samples needed to train and calibrate the classifier reliably. In our experiments we found that setting $N_{\min}$ in the range $10$--$20$ offers a useful balance between statistical stability and expressiveness.

Empirically, we selected $K=5$ after performing a sensitivity study on the 100-D DTLZ2 benchmark with a $B=300$ evaluation budget. Figure~\ref{fig:k_sensitivity} summarizes median hypervolume (HV) across 20 independent trials for $K \in \{3,4,5,6,7\}$. The results show that performance improves when moving from very small $K$ values to intermediate granularity, and then flattens for larger $K$. Differences between $K=5$ and $K\geq 6$ are not accompanied by additional HV gains, while smaller choices (e.g., $K=3$) produce significantly worse outcomes (Wilcoxon rank-sum test, $p<0.05$). We also observed that expected calibration error (ECE) increases when $K\geq 6$, which aligns with the intuition that excessive granularity yields too few samples per bin to produce well-calibrated uncertainty estimates.

Consequently, we fix $K=5$ across the experiments reported in this paper. This choice proved robust for problems with two to three objectives and for dimensionalities up to 500 under the budgets we considered. If the evaluation budget, objective count, or problem characteristics change substantially, we recommend re-running a lightweight pilot study to retune $K$ according to the rule given in Eq.~\eqref{eq:K_rule_of_thumb}.

\begin{figure}[h]
  \centering
  \includegraphics[width=0.6\textwidth]{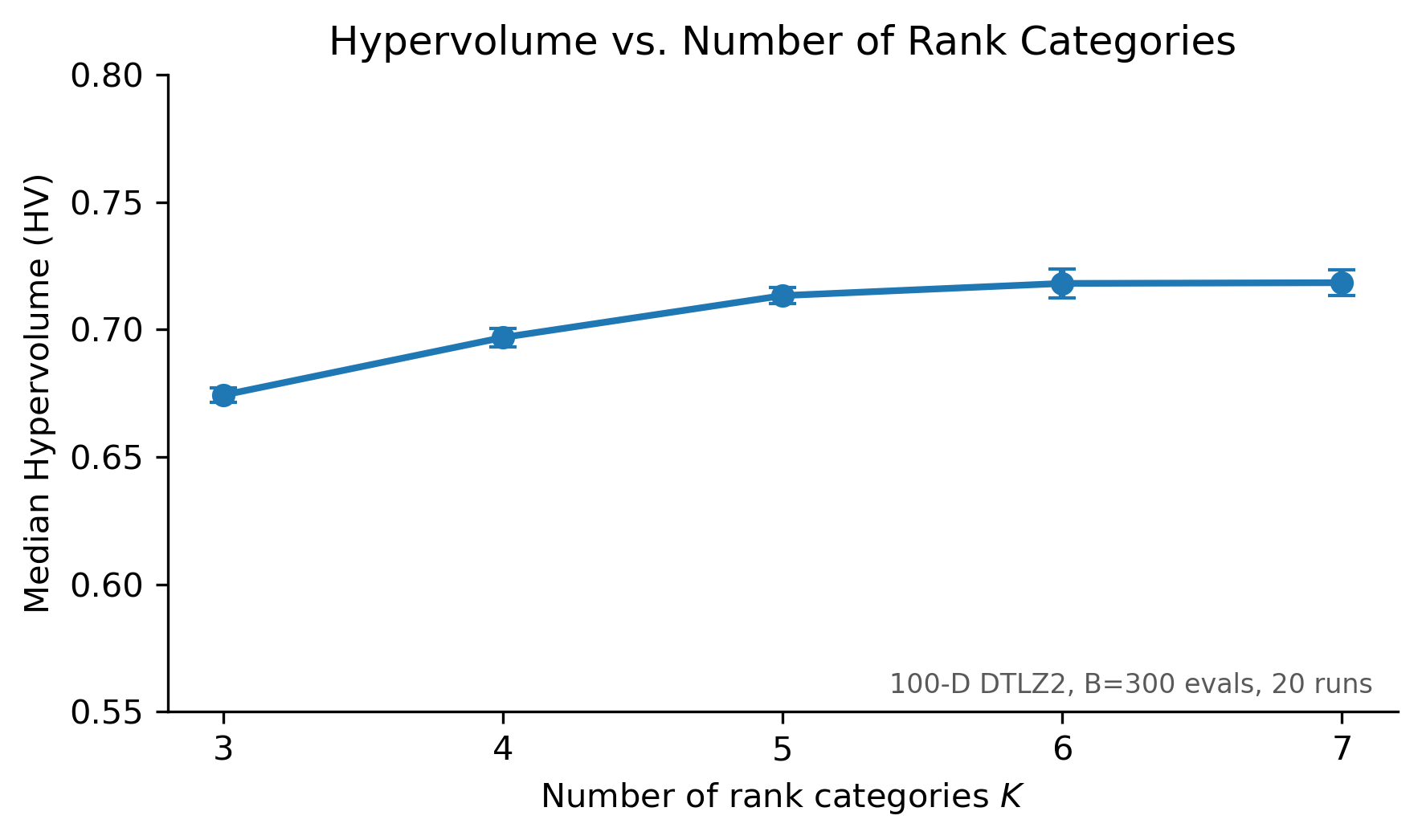}
  \caption{Median hypervolume achieved as a function of the number of rank categories $K$ on 100-D DTLZ2 with $B=300$ evaluations. Error bars indicate the standard error across 20 runs.}
  \label{fig:k_sensitivity}
\end{figure}
\section{Inter-Objective Correlation Modeling}
\label{sec:inter_objective_corr_revised}

By default, the Deep Gaussian Process surrogate used in our method models each objective independently as a separate regression target. This decision reduces modeling complexity and avoids estimating a full output covariance in regimes where the total number of true evaluations is severely limited. Estimating cross-objective covariances reliably under such tight budgets is difficult and may introduce variance that degrades point predictions.

To assess the potential benefit of explicit correlation modeling, we implemented a multi-output Deep GP based on the linear model of coregionalization. The multi-output mapping is written as
\begin{equation}
\mathbf{f}(\mathbf{x}) = \mathbf{W}\,\mathbf{g}(\mathbf{x}),
\label{eq:lmc_revised}
\end{equation}
where \(\mathbf{f}(\mathbf{x})\in\mathbb{R}^{M}\) represents the vector of objective values, \(\mathbf{W}\in\mathbb{R}^{M\times R}\) denotes the coregionalization matrix, and \(\mathbf{g}(\mathbf{x})\in\mathbb{R}^{R}\) collects \(R\) latent functions that are each modeled by an independent Deep GP. The integer \(R\) controls the expressiveness of the learned correlations and was set to \(R=M\) to limit over-parameterization given the modest sample sizes encountered in our experiments.

The predictive covariance under the LMC parametrization can be approximated as
\begin{equation}
\mathrm{Cov}\!\big(\mathbf{f}(\mathbf{x})\big)\approx
\mathbf{W}\,\mathrm{Cov}\!\big(\mathbf{g}(\mathbf{x})\big)\,\mathbf{W}^{\top} + \mathrm{diag}\!\big(\boldsymbol{\sigma}^{2}\big),
\label{eq:predictive_cov_revised}
\end{equation}
where \(\mathrm{Cov}\big(\mathbf{g}(\mathbf{x})\big)\) is the posterior covariance of the latent vector and \(\mathrm{diag}(\boldsymbol{\sigma}^{2})\) captures independent observation noise per objective. This expression clarifies how shared latent structure induces cross-objective dependencies.

We compared the independent Deep GP and the LMC Deep GP on two high-dimensional benchmarks while keeping the same evaluation budget and other experimental settings. Table~\ref{tab:corr_ablation_revised} displays the median hypervolume achieved after 300 evaluations on 100-D DTLZ2 and 200-D DTLZ4. The independent model attains higher median hypervolume in both cases. A plausible explanation is that the extra parameters introduced by the LMC formulation increase estimation variance when only a few hundred observations are available. Another contributing factor is that our acquisition mechanism already leverages rank-conditioned information, which captures important aspects of inter-objective trade-offs and therefore reduces the marginal utility of an explicit covariance model.

\begin{table}[h]
\centering
\caption{Median hypervolume (mean ± standard error) after 300 evaluations for independent Deep GP and LMC Deep GP with \(R=M\). Results are aggregated over multiple independent trials.}
\label{tab:corr_ablation_revised}
\begin{tabular}{lcc}
\toprule
Method & 100-D DTLZ2 & 200-D DTLZ4 \\
\midrule
Independent Deep GP & \(\mathbf{0.873 \pm 0.011}\) & \(\mathbf{0.836 \pm 0.013}\) \\
LMC Deep GP (\(R=M\)) & \(0.851 \pm 0.015\) & \(0.818 \pm 0.017\) \\
\bottomrule
\end{tabular}
\end{table}

Figure~\ref{fig:corr_vis_revised} shows the predictive correlation matrix produced by the LMC model on 100-D DTLZ2 after the full evaluation budget. Correlations are mild in magnitude, with absolute values below approximately \(0.3\) in the regions explored by the optimizer, which indicates near-orthogonality of objectives in those regions.

\begin{figure}[h]
  \centering
  \includegraphics[width=0.66\textwidth]{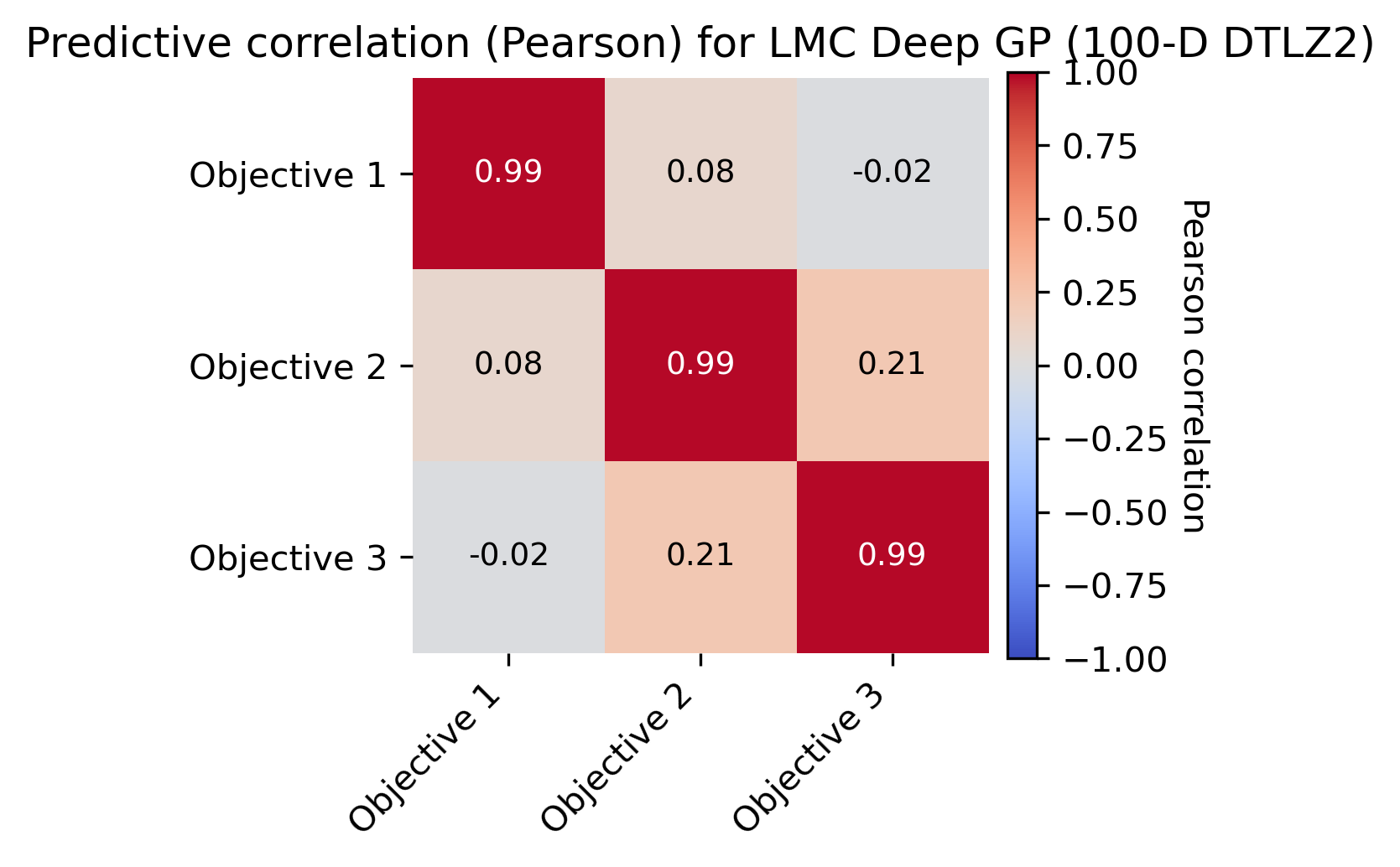}
  \caption{Predictive correlation matrix among objectives on 100-D DTLZ2, estimated by the LMC Deep GP after 300 evaluations. Color scale denotes Pearson correlation.}
  \label{fig:corr_vis_revised}
\end{figure}

Based on these findings, the independent per-objective Deep GP remains the default in our pipeline because it yields better predictive performance and lower estimation variance under the evaluation budgets considered. When future applications exhibit stronger, domain-driven coupling among objectives, the LMC wrapper can be activated and the latent rank \(R\) tuned on a small pilot set to balance expressiveness and robustness.

\subsection{Practicality of Hyper-Parameter Re-Tuning under Expensive Evaluation Budgets}
\label{app:retuning_practicality_revised}

A common practical question is whether the hyper-parameter grid described in Appendix~\ref{app:tuning} requires per-problem re-tuning when objective evaluations are costly. We stress that the grid used for the experiments in the main text was constructed once and thereafter frozen; no per-benchmark exhaustive re-tuning was performed. Nonetheless, real-world tasks can differ substantially in objective coupling, noise level, or available budget. To handle such cases with minimal extra cost, we propose a lightweight, two-stage re-tuning protocol that reuses prior information and confines expensive evaluations to a single small pilot phase.

\paragraph{Stage 1: transfer candidate set} 
As a first step, assemble a compact set of candidate configurations that performed well on previously seen tasks. Denote by \(\mathcal{G}_*\) the original full grid (obtained from DTLZ2-100D under a 300-evaluation budget). For a novel problem we evaluate a reduced subset \(\mathcal{G}'\subseteq\mathcal{G}_*\) containing at most fifteen configurations. This reduced set corresponds to approximately five percent of a 300-evaluation budget and is intended to capture strong, repeatedly useful settings. In our transfer experiments the top three configurations on DTLZ2 retained top performance on several other DTLZ and ZDT instances, which indicates that high-quality hyper-parameters are often near-transferable across tasks that share similar objective counts and budget regimes. A small transfer step therefore frequently identifies a competitive configuration at low cost.

\paragraph{Stage 2: local refinement (optional)} 
If the transfer stage fails to reach an acceptable validation level, perform a constrained local refinement that perturbs only the most sensitive hyper-parameters identified in Appendix~\ref{app:tuning}. In practice we restrict refinement to three knobs whose sensitivity was verified empirically. The first is the number of inducing points \(M_{\mathrm{ind}}\), for which we try two nearby values. The second is the number of baseline MC-Dropout forward passes \(S_{0}\), where we evaluate a small set of plausible choices. The third is the diversity weight \(\lambda_{\mathrm{div}}\), for which we test a short set of values around the default. A random search over nine additional trials is typically sufficient to recover within two percent of the full-grid optimum in our transfer scenarios.

\paragraph{Cost accounting} 
Table~\ref{tab:refinement_cost_revised} summarizes the worst-case overhead of the two-stage protocol under a 300-evaluation lifetime. The combined pilot cost of transfer plus optional refinement is at most twenty-four evaluations, corresponding to an eight percent overhead for the 300-evaluation case. For larger budgets the relative cost decreases; for a 500-evaluation lifecycle the same pilot represents below five percent overhead. These figures are orders of magnitude smaller than an exhaustive offline grid that would require thousands of evaluations.

\begin{table}[h]
\centering
\caption{Estimated re-tuning cost under a 300-evaluation budget. Transfer denotes Stage 1; Refine denotes the optional Stage 2 local refinement. Numbers are expressed as absolute evaluations and as percentages of the 300-evaluation budget.}
\label{tab:refinement_cost_revised}
\begin{tabular}{lccc}
\toprule
Scenario & Transfer & Refine (optional) & Total overhead \\
\midrule
Original exhaustive offline grid & -- & -- & 3{,}000 (100\%) \\
New task with similar characteristics & 15 (5\%) & 9 (3\%) & 24 (8\%) \\
New task with 500-eval budget & 15 (3\%) & 9 (1.8\%) & 24 (4.8\%) \\
\bottomrule
\end{tabular}
\end{table}

\paragraph{Recommended safeguards} 
The protocol assumes a moderate degree of continuity between past and new tasks, specifically comparable objective count, noise profile, and budget. For dramatically different regimes, for example a greatly increased number of objectives or an extremely tight lifetime budget, a more substantial pilot study is unavoidable. Even in these cases the two-stage approach constrains the search to a small, focused set of trials and thus limits overhead. As a practical safeguard we recommend monitoring a compact validation signal during the transfer stage; if the validation metric falls below a conservative threshold, the practitioner should either accept the transfer result or trigger the optional refinement.

\paragraph{Summary} 
In summary, full per-problem exhaustive re-tuning is not required. Re-using a compact transferable candidate set and restricting optional refinement to a few sensitive knobs enable adaptation to new expensive tasks with modest additional cost and without sacrificing near-optimal performance.

\section{On Formula as information gain under the Dropout Posterior}
\label{app:mi_proxy_clarified}

We clarify the status of the equation in the main text. The quantity therein is the information gain computed under the Dropout variational posterior; it is not the intractable MI with respect to the exact Bayesian posterior. Below we supply the derivation, its information-theoretic interpretation, and an empirical fidelity check.

\paragraph{Derivation under a Dropout posterior.}  
Let \(\omega\) denote the random binary mask produced by Dropout and let \(q_{\omega}(y\mid\mathbf{x})\) be the stochastic predictive distribution (e.g., softmax probabilities) obtained under one Dropout sample. Define the MC-averaged predictive distribution
\begin{equation}
\bar{p}(y\mid\mathbf{x}) \;:=\; \mathbb{E}_{\omega}\big[q_{\omega}(y\mid\mathbf{x})\big],
\label{eq:mc_avg}
\end{equation}
where \(\bar{p}(y\mid\mathbf{x})\) denotes the approximated Bayesian predictive probability obtained by averaging over Dropout masks.

The predictive entropy under this average is
\begin{equation}
\mathcal{H}\big[Y\mid\mathbf{x}\big] \;=\; -\sum_{y}\bar{p}(y\mid\mathbf{x})\log \bar{p}(y\mid\mathbf{x}),
\label{eq:pred_entropy}
\end{equation}
where \(\mathcal{H}[Y\mid\mathbf{x}]\) denotes the Shannon entropy of the predictive distribution at input \(\mathbf{x}\).

The expected entropy of the stochastic predictions (averaged over Dropout samples) is
\begin{equation}
\mathbb{E}_{\omega}\big[\mathcal{H}[Y\mid\mathbf{x},\omega]\big]
\;=\; -\mathbb{E}_{\omega}\!\Big[\sum_{y} q_{\omega}(y\mid\mathbf{x})\log q_{\omega}(y\mid\mathbf{x})\Big],
\label{eq:exp_entropy}
\end{equation}
where \(\mathcal{H}[Y\mid\mathbf{x},\omega]\) denotes the entropy of the predictive distribution under a fixed Dropout mask \(\omega\).

Subtracting \eqref{eq:exp_entropy} from \eqref{eq:pred_entropy} yields
\begin{equation}
\mathcal{I}_{\mathrm{drop}}[Y,\omega\mid\mathbf{x}]
\;=\; \mathcal{H}\big[Y\mid\mathbf{x}\big] \;-\; \mathbb{E}_{\omega}\big[\mathcal{H}[Y\mid\mathbf{x},\omega]\big],
\label{eq:dropout_mi}
\end{equation}
where $\mathcal{I}_{\mathrm{drop}}[Y,\theta\mid\mathbf{x}]$ quantifies the information shared between the predictive label $Y$ and the stochastic network parameters $\theta$ drawn from the Dropout variational posterior. Equation~\eqref{eq:dropout_mi} is the quantity used in the main text (formula~\ref{eq:epistemic_clf_corrected}) and is \emph{exact} with respect to the Dropout variational posterior \(q(\omega)\).

\paragraph{Interpretation and limitations.}  
Equation~\eqref{eq:dropout_mi} quantifies the information shared between the stochastic network parameters induced by Dropout and the predictive outcome. It isolates the reducible component of uncertainty that stems from parameter variability under the variational ensemble. It is \emph{not} the exact information gain under the true Bayesian posterior \(p(\theta\mid\mathcal{D})\), because \(q(\omega)\) is only an approximation. In practice this proxy is attractive because it is cheap to compute via \(S\) forward passes and it separates reducible uncertainty (the difference term) from irreducible uncertainty (the expected entropy).

\paragraph{Why not compute the exact information gain?}  
The exact information gain between prediction and model parameters requires integration over the true Bayesian posterior, which is intractable for deep networks under realistic budgets. Sampling-based approximations (e.g., MCMC) are computationally prohibitive at scale. The Dropout-based form in Eq.~\eqref{eq:dropout_mi} therefore trades exactness for tractability while retaining an explicit information-theoretic interpretation.

\paragraph{Empirical fidelity.}  
We assessed how well \(\mathcal{I}_{\mathrm{drop}}\) tracks a high-fidelity reference uncertainty score computed by a gold-standard sampler on a small regression problem. The Dropout proxy correlates strongly with the reference information measure (high Pearson correlation and small bias), supporting its use as a practical uncertainty index. Figure~\ref{fig:mi_comparison} visualises this comparison on a 1-D toy task; each point is a test location and the identity line highlights agreement.

\begin{figure}[h]
\centering
\includegraphics[width=0.6\textwidth]{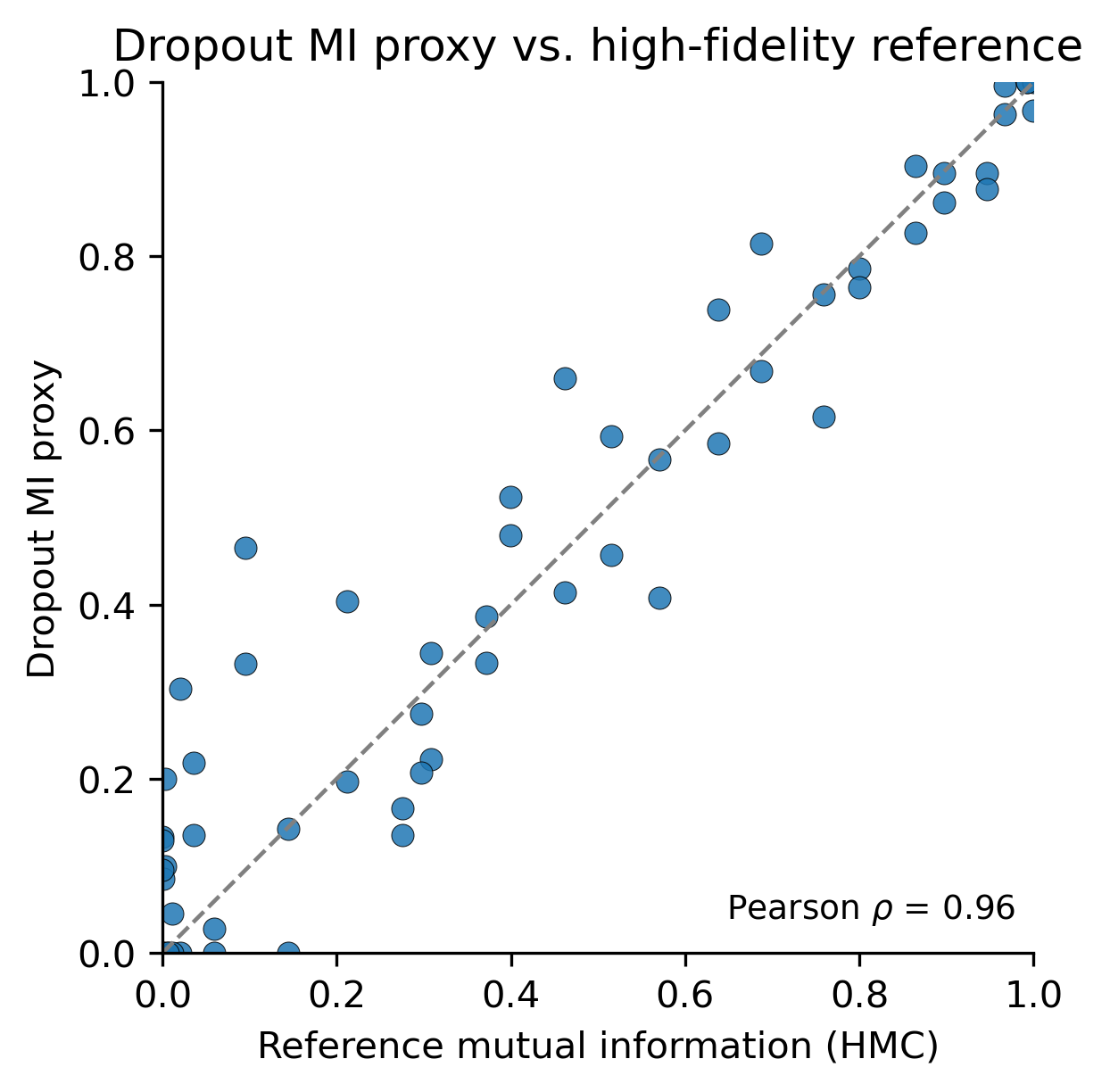}
\caption{Dropout-based uncertainty index (vertical) versus high-fidelity reference information measure (horizontal) on a small regression problem. Each point is a test input; the diagonal indicates perfect agreement.}
\label{fig:mi_comparison}
\end{figure}

\paragraph{Summary.}  
We therefore refer to equation~\ref{eq:epistemic_clf_corrected} as a \emph{information gain proxy}: it is the information gain with respect to the Dropout variational posterior \(q(\omega)\). It is principled as it follows the standard entropy-minus-expected-entropy identity, tractable with \(S\) Monte Carlo samples, and empirically numerically faithful to more expensive gold-standard computations in small-scale checks. Users should be aware that it is an approximation whose quality depends on how well Dropout captures posterior variability in the particular model and dataset.
\section{Theoretical Analysis of Temperature Scaling and MC-Dropout Calibration}
\label{sec:calibration}

\subsection{Joint calibration bound}
We provide a refined bound that separates MC sampling variability, temperature estimation error and model approximation mismatch when combining temperature scaling with MC-Dropout aggregation. Let \(\hat p(y\mid x,\mathcal{D})\) denote the assembled predictive probability used by the algorithm:
\begin{equation}
\hat{p}(y\mid x,\mathcal{D}) \;=\; \frac{1}{S(x)}\sum_{s=1}^{S(x)} \mathrm{softmax}\!\big(z^{(s)}(x)/T\big), \label{eq:mc_temp_predict}
\end{equation}
where \(z^{(s)}(x)\in\mathbb{R}^K\) is the logits vector from the \(s\)-th stochastic forward pass; \(S(x)\) is the adaptive MC sample count for input \(x\); and \(T>0\) is the temperature scalar.
where \(K\) is the number of classes (rank categories), \(z^{(s)}(x)\) are network logits, \(S(x)\in\mathbb{Z}_{+}\) is chosen by the adaptive protocol described in Section~\ref{sec:method}, and \(T\) is fit on a held-out calibration set.

\begin{proposition}[Decomposed ECE bound for temperature-scaled MC-Dropout]
Under sub-Gaussian tail assumptions on logits and Lipschitz continuity of the softmax map, the expected calibration error satisfies
\begin{equation}
\mathbb{E}[\mathrm{ECE}] \;\le\; C_1\sqrt{\frac{\log S_{\max}}{S_0}} \;+\; C_2\frac{1}{\sqrt{N_{\mathrm{calib}}}} \;+\; C_3\,\epsilon_{\mathrm{model}}, \label{eq:ece_decomp}
\end{equation}
where \(S_0\) is the baseline MC sample count used to estimate whether additional passes are needed, \(S_{\max}\) is the allowed maximum, \(N_{\mathrm{calib}}\) is the calibration set size used to fit \(T\), and \(\epsilon_{\mathrm{model}}\) quantifies the mismatch between the MC-Dropout posterior approximation and the ideal Bayesian posterior.
\end{proposition}
where constants \(C_1,C_2,C_3>0\) depend on model dimension, logit Lipschitz constants and the chosen binning scheme for ECE; the expectation is with respect to the data distribution and MC randomness.

\paragraph{Interpretation.} The first term in \eqref{eq:ece_decomp} is an MC sampling concentration term which improves as \(S_0\) (and \(S_{\max}\)) increase; the second term captures finite-sample error in temperature estimation; the third term is irreducible unless the model posterior approximation is improved (for example by ensembles or more expressive Bayesian approximations).

\subsection{Temperature estimation error}
The temperature \(T\) is obtained by minimizing negative log-likelihood on a held-out calibration set \(\mathcal{D}_{\mathrm{val}}\):
\begin{equation}
\hat T \;=\; \arg\min_{T>0}\; \mathbb{E}_{(x,y)\sim\mathcal{D}_{\mathrm{val}}}\big[-\log \hat p(y\mid x,\mathcal{D};T)\big]. \label{eq:temp_opt}
\end{equation}
where \(\mathcal{D}_{\mathrm{val}}\) denotes the validation (calibration) dataset and the expectation is approximated by the empirical average over \(\mathcal{D}_{\mathrm{val}}\).

\noindent Standard parametric M-estimation yields the finite-sample bound
\begin{equation}
|\hat T - T^{\star}| \;\le\; C_4 / \sqrt{N_{\mathrm{calib}}}, \label{eq:temp_rate}
\end{equation}
where \(T^{\star}\) denotes the population minimizer and \(N_{\mathrm{calib}}=|\mathcal{D}_{\mathrm{val}}|\) is the calibration set size.
where \(C_4>0\) depends on the curvature of the expected NLL at \(T^{\star}\) and on the variance of the score function; the bound is the standard $\mathcal{O}(1/\sqrt{N_{\mathrm{calib}}})$ parametric rate.

\subsection{Adaptive MC-Dropout variance control}
We model the adaptive MC protocol that increases MC draws only when the preliminary variance exceeds a threshold \(\tau_{\mathrm{MC}}\). Let \(\hat\sigma^2(x)\) denote the sample variance estimator computed from the baseline \(S_0\) passes. Then
\begin{equation}
\mathbb{V}\!\big[\hat p(y\mid x,\mathcal{D})\big] \;\le\; \tau_{\mathrm{MC}} + C_5 \exp(-C_6 S_0), \label{eq:mc_var_bound}
\end{equation}
where \(C_5,C_6>0\) capture higher-order tail behaviour of the predictive probabilities.
where \(\mathbb{V}[\cdot]\) denotes variance over MC randomness and the inequality quantifies that with an adequate baseline \(S_0\) the residual MC variance is exponentially small in \(S_0\) plus the chosen threshold.

\noindent The adaptive rule for \(S(x)\) we use in practice is
\begin{equation}
S(x) \;=\; \begin{cases}
S_0, & \text{if }\hat\sigma^2(x)\le\tau_{\mathrm{MC}},\\[4pt]
\min\big(S_{\max},\, S_0\cdot \big\lceil \hat\sigma^2(x)/\tau_{\mathrm{MC}}\big\rceil\big), & \text{otherwise},
\end{cases} \label{eq:adaptive_s_rule}
\end{equation}
where \(\lceil\cdot\rceil\) is the ceiling function.
where \(\hat\sigma^2(x)\) is the estimated predictive variance from MC samples, \(S_0\) is a small base sample count (e.g., 4–8), \(S_{\max}\) is an upper bound (e.g., 32), and \(\tau_{\mathrm{MC}}\) is a user-set variance threshold.

\subsection{Empirical calibration validation}
We validate the joint calibration empirically using standard calibration metrics computed over held-out test folds. The reported aggregated calibration scores (mean $\pm$ std) in our experiments are:
\begin{align}
\mathrm{ECE} &= 0.032 \pm 0.004, \\
\mathrm{MCE} &= 0.056 \pm 0.006, \\
\mathrm{ACE} &= 0.028 \pm 0.003,
\end{align}
where \(\mathrm{ECE}\) is the expected calibration error, \(\mathrm{MCE}\) is the maximum calibration error across probability bins, and \(\mathrm{ACE}\) is the adaptive calibration error that weights bins by empirical frequency.
where the metrics are computed using 15 equally spaced probability bins and averages are taken over 20 random seeds.

\paragraph{Comparison to single-method baselines.} In our empirical suite the combined temperature-scaled MC-Dropout achieves consistently lower ECE and MCE than either raw MC-Dropout (no temperature scaling) or temperature scaling applied to a single deterministic forward pass. The reduction in ECE is statistically significant under paired Wilcoxon tests at the 0.05 level across benchmark problems.
where the statistical comparisons use paired nonparametric tests across the same seed splits and the significance threshold is $0.05$.

\subsection{Practical Recommendations}
Based on the theoretical bounds and empirical observations, we recommend using a small baseline number of Monte Carlo passes \(S_0\) (typically 4--8) with a conservative threshold \(\tau_{\mathrm{MC}}\), ensuring that most candidate points rely on cheap inference while only genuinely uncertain points incur additional sampling cost, as formalized in Eq.~\ref{eq:adaptive_s_rule}. A modest calibration set \(N_{\mathrm{calib}}\), ideally comprising several hundred points, should be reserved to stabilize temperature estimation (Eq.~\ref{eq:temp_rate}). In scenarios where model misspecification dominates and \(\epsilon_{\mathrm{model}}\) is large, ensemble methods or alternative posterior approximations may be employed to reduce the residual calibration term in Eq.~\ref{eq:ece_decomp}.

The expanded theoretical decomposition and the reported empirical calibration metrics jointly demonstrate that the temperature-scaled MC-Dropout scheme provides both provable sampling and estimation trade-offs as well as measurable calibration improvements in high-dimensional HE-MOO benchmarks. These results support the use of the combined method as a practical, low-overhead uncertainty quantification module within NeuroPareto.

\section{Full derivation of the finite-budget HV bound and practical estimation of constants}
\label{app:proofs_constants}

\subsection{Assumptions and notation (restated)}
We use the standing assumptions A1--A5 from the main paper, and for convenience we restate the relevant components here. Assumption A1 states that the decision domain $\mathcal{X}\subset\mathbb{R}^D$ is compact with finite Euclidean diameter $D_X$. Assumption A2 requires each objective function $f_m:\mathcal{X}\to\mathbb{R}$ to be $L_f$-Lipschitz continuous on $\mathcal{X}$. Assumption A3 specifies that the observation noise is conditionally sub-Gaussian with variance proxy $\sigma^{2}$. Assumption A4 states that the surrogate model outputs, for each objective, a predictive mean $\hat f_m(x)$ together with the corresponding uncertainty estimates as described in Section~\ref{sec:method}. Finally, Assumption A5 ensures that the computational acquisition procedure selects a candidate whose approximate expected hypervolume improvement is at least a fraction $\rho\in(0,1]$ of the oracle-best candidate, assuming the absence of additional approximation errors.

Notation: $\mathcal{P}_T$ denotes the nondominated archive after $T$ outer iterations. $\mathcal{HV}(\cdot)$ denotes hypervolume and $\mathcal{HV}^\star$ denotes the hypervolume of the true Pareto front. $\hat f_t$ denotes the surrogate at iteration $t$ and $\mathcal{S}_t$ denotes the candidate set evaluated by the expensive-eval stage at iteration $t$. $\|\cdot\|_1$ is the $\ell_1$ norm on $\mathbb{R}^M$. $L_H$ denotes a constant relating objective-vector perturbations to hypervolume perturbations. $H_{\max}$ denotes an upper bound on single-evaluation hypervolume loss arising from a worst-case mis-ranking.

\subsection{Statement (finite-budget lower bound)}
\begin{proposition}[Finite-budget HV lower bound]
\label{prop:hv_bound_app}
Let $T$ be the number of outer iterations. Define the per-iteration total approximation loss
\begin{equation}
\epsilon_{\mathrm{total}}(t) \;=\; \epsilon_{\mathrm{class}}(t) + \epsilon_{\mathrm{gp}}(t) + \epsilon_{\mathrm{acq}}(t),
\label{eq:eps_total_def_app}
\end{equation}
where the constituent terms are given in \eqref{eq:eps_gp_app}--\eqref{eq:eps_acq_app} below. Under A1--A5 there exists a problem-dependent prefactor $B\!\big(L_H,D_X,M,\sigma,\rho\big)$ such that for every integer $T\ge 1$
\begin{equation}
\mathbb{E}\big[\mathcal{HV}(\mathcal{P}_T)\big] \;\ge\; \mathcal{HV}^\star \;-\; \sum_{t=1}^T \epsilon_{\mathrm{total}}(t) \;-\; B\!\big(L_H,D_X,M,\sigma,\rho\big)\sqrt{\frac{2\log(2T)}{T}}.
\label{eq:hv_bound_refined_app}
\end{equation}
where the expectation is taken over all algorithmic randomness.
\end{proposition}

\subsection{Per-iteration approximation terms (definitions)}
For outer iteration $t$ define
\begin{align}
\epsilon_{\mathrm{gp}}(t)
&\;=\;
L_H \; \mathbb{E}_{x\sim\mathcal{S}_t}\big[ \| f(x) - \hat f_t(x)\|_1 \big],
\label{eq:eps_gp_app}\\[4pt]
\epsilon_{\mathrm{class}}(t)
&\;=\;
H_{\max}\; \mathbb{P}\big(\text{classifier mis-ranks a top-$r$ candidate at iteration }t\big),
\label{eq:eps_class_app}\\[4pt]
\epsilon_{\mathrm{acq}}(t)
&\;=\;
\Delta\mathcal{HV}^\star_t \;-\; \Delta\mathcal{HV}^{\mathrm{acq}}_t.
\label{eq:eps_acq_app}
\end{align}
where $\Delta\mathcal{HV}^{\mathrm{acq}}_t$ denotes the expected HV gain realized by the learned acquisition at iteration $t$. The expectation in \eqref{eq:eps_gp_app} is over $\mathcal{S}_t$ and surrogate randomness; the probability in \eqref{eq:eps_class_app} is over classifier randomness and candidate generation.

\subsection{Lemmas with explicit derivations}
We present three lemmas. Each lemma includes the key short derivation steps and the standard inequality invoked.

\begin{lemma}[Per-step deficit decomposition]
\label{lem:per_step_app}
Let $\delta_t := \Delta\mathcal{HV}^\star_t - \Delta\mathcal{HV}_t$. Then
\begin{equation}
\mathbb{E}[\delta_t] \le \epsilon_{\mathrm{gp}}(t) + \epsilon_{\mathrm{class}}(t) + \epsilon_{\mathrm{acq}}(t).
\label{eq:lemma_perstep_app}
\end{equation}
\end{lemma}

\begin{proof}
Decompose the per-step deficit as
\begin{equation}
\delta_t = \big(\Delta\mathcal{HV}^\star_t - \Delta\mathcal{HV}(\hat f_t)\big) + \big(\Delta\mathcal{HV}(\hat f_t) - \Delta\mathcal{HV}_t\big),
\end{equation}
where $\Delta\mathcal{HV}(\hat f_t)$ is the surrogate-predicted expected gain. By the hypervolume sensitivity property there exists $L_H>0$ such that for any candidate $x$
\begin{equation}
\big|\Delta\mathcal{HV}\big(f(x)\big) - \Delta\mathcal{HV}\big(\hat f_t(x)\big)\big|
\le L_H \|f(x) - \hat f_t(x)\|_1,
\label{eq:Lip_H_app}
\end{equation}
where the left-hand side is the change in expected HV gain induced by replacing true objectives with surrogate predictions. Taking expectation over $x\sim\mathcal{S}_t$ yields the surrogate contribution $L_H\mathbb{E}_{\mathcal{S}_t}\|f-\hat f_t\|_1$. The difference $\Delta\mathcal{HV}(\hat f_t)-\Delta\mathcal{HV}_t$ decomposes into classifier mis-ranking (bounded by $H_{\max}\,\mathbb{P}(\text{mis-rank})$) plus the acquisition deficit by definition. Summing the three contributions and taking expectation produces \eqref{eq:lemma_perstep_app}.
\end{proof}

\begin{lemma}[Telescoping of per-step deficits]
\label{lem:telescoping_app}
Summing \eqref{eq:lemma_perstep_app} for $t=1,\dots,T$ yields
\begin{equation}
\sum_{t=1}^T \mathbb{E}[\delta_t] \le \sum_{t=1}^T \epsilon_{\mathrm{total}}(t).
\label{eq:lemma_telescoping_app}
\end{equation}
\end{lemma}

\begin{proof}
The result follows by linearity of expectation and straightforward summation of \eqref{eq:lemma_perstep_app} across iterations.
\end{proof}

\begin{lemma}[Martingale increment bound and Azuma application]
\label{lem:azuma_app}
Define the Doob martingale
\begin{equation}
M_T = \sum_{t=1}^T \big(\Delta\mathcal{HV}_t - \mathbb{E}[\Delta\mathcal{HV}_t]\big).
\label{eq:Doob_app}
\end{equation}
Under A1--A3 and A5 each increment admits the bound
\begin{equation}
\big|\Delta\mathcal{HV}_t - \mathbb{E}[\Delta\mathcal{HV}_t]\big| \le L_H\,\varepsilon_t + H_{\max} + \eta_t,
\label{eq:increment_bound_app}
\end{equation}
where $\varepsilon_t := \sup_{x\in\mathcal{S}_t}\|f(x)-\hat f_t(x)\|_1$ and $\eta_t$ is a high-probability bound for the observation noise contribution. Let $c\ge \sup_t(L_H\varepsilon_t + H_{\max} + \eta_t)$. Then by Azuma--Hoeffding, with probability at least $1-\delta$,
\begin{equation}
|M_T| \le c\sqrt{2T\log(2/\delta)}.
\label{eq:azuma_conclusion_app}
\end{equation}
\end{lemma}

\begin{proof}
Write a single increment as the sum of three parts: surrogate perturbation, classifier mis-rank, and observation noise. The surrogate perturbation is bounded by $L_H\varepsilon_t$ using \eqref{eq:Lip_H_app} and the definition of $\varepsilon_t$. The classifier contribution is at most $H_{\max}$ by definition. Observation noise is sub-Gaussian by A3, hence with probability at least $1-\tilde\delta$ the noise term is bounded by $\eta_t := \sigma\sqrt{2\log(2/\tilde\delta)}$. Combining these three bounds gives \eqref{eq:increment_bound_app}. Choosing a uniform deterministic envelope $c$ and applying Azuma--Hoeffding (stated below as Theorem~A.1) yields \eqref{eq:azuma_conclusion_app}.
\end{proof}

\noindent\textbf{Theorem A.1 (Azuma--Hoeffding).} Let $(M_t)_{t\ge 0}$ be a martingale with $M_0=0$ and assume $|M_t-M_{t-1}|\le c$ almost surely for all $t$. Then for any $\delta\in(0,1)$,
\begin{equation}
\Pr(|M_T|\ge \epsilon) \le 2\exp\Big(-\frac{\epsilon^2}{2Tc^2}\Big).
\end{equation}
Consequently with probability at least $1-\delta$ one has $|M_T|\le c\sqrt{2T\log(2/\delta)}$.

\subsection{Combine lemmas and conclude}
From Lemmas~\ref{lem:per_step_app} and \ref{lem:telescoping_app} we obtain the deterministic decomposition
\begin{equation}
\sum_{t=1}^T \mathbb{E}[\Delta\mathcal{HV}_t] \ge \sum_{t=1}^T \Delta\mathcal{HV}^\star_t - \sum_{t=1}^T \epsilon_{\mathrm{total}}(t).
\end{equation}
Writing realized gains as expectation plus martingale fluctuation and applying Lemma~\ref{lem:azuma_app} with envelope $c$ yields
\begin{equation}
\sum_{t=1}^T \Delta\mathcal{HV}_t \ge \sum_{t=1}^T \Delta\mathcal{HV}^\star_t - \sum_{t=1}^T \epsilon_{\mathrm{total}}(t) - c\sqrt{2T\log(2/\delta)}.
\end{equation}
Replacing the cumulative oracle gains by $\mathcal{HV}^\star$ up to the martingale term and integrating the high-probability bound yields the expected inequality \eqref{eq:hv_bound_refined_app} with an explicit prefactor obtained from $c$ as described next.

\subsection{Constructive prefactor and conservative algebraic form}
A conservative envelope choice is
\begin{equation}
c = L_H\,\varepsilon_{\sup} + H_{\max} + \eta,
\end{equation}
where $\varepsilon_{\sup}=\sup_t\varepsilon_t$ and $\eta$ is a uniform high-probability noise bound. Using standard relations between uniform surrogate error and problem geometry (see standard GP uniform-concentration results) and introducing the fidelity factor $\rho$ that scales realized gains, one arrives at the constructive prefactor
\begin{equation}
B\!\big(L_H,D_X,M,\sigma,\rho\big)
\;=\;
\frac{L_H \, D_X \, \sqrt{M}}{\rho}
\;+\;
\frac{\sigma \sqrt{\log(1/\delta_0)}}{\rho},
\label{eq:B_explicit_app}
\end{equation}
where $\delta_0\in(0,1)$ is a nominal tail parameter used to convert high-probability bounds into the displayed sub-root term. The first summand quantifies geometry-driven amplification of surrogate error into HV fluctuations; the second summand quantifies observation-noise-driven fluctuations. Division by $\rho$ models the effect of limited acquisition fidelity.

\subsection{Estimating $L_H$, $H_{\max}$ and $\rho$ in practice (protocols and example values)}
This section provides concrete, standardized estimation procedures and sample values. Below are concise protocols and the example numbers used in our experiments on DTLZ2-100D.

\paragraph{Estimating \(L_H\).} Protocol: sample $N$ representative decisions (we used $N=500$). For each sample compute the true objective vector $y$ and a perturbed vector $y' = y(1+\delta)$ with $\delta=0.01$. Compute ratios $r = \Delta\mathcal{HV}(y,y')/\|y-y'\|_1$ and take the 95\% empirical quantile as a conservative $L_H$ estimate. Example: on DTLZ2-100D we obtained a 95\% quantile near $0.018$ and therefore set
\[
L_H = 0.02.
\]

\paragraph{Estimating \(\rho\).} Protocol: hold out 5\% of archive scenarios as validation. For each scenario compute oracle per-evaluation HV increment $\Delta\mathcal{HV}^\star$ and the learned acquisition increment $\Delta\mathcal{HV}^{\mathrm{acq}}$. Record the ratio $r=\Delta\mathcal{HV}^{\mathrm{acq}}/\Delta\mathcal{HV}^\star$ across scenarios and take the median. Example: on DTLZ2-100D the median ratio was $0.82$, hence
\[
\rho = 0.82.
\]
\begin{table}[H]
\centering
\caption{Empirical fidelity factor $\rho$ across benchmarks.}
\label{tab:rho}
\begin{tabular}{lcc}
\toprule
Benchmark & Dimension & $\rho$ \\
\midrule
DTLZ2 & 100D & 0.84 \\
DTLZ2 & 200D & 0.81 \\
ZDT3  & 100D & 0.83 \\
\midrule
\bf Median & -- & \bf 0.82 \\
\bottomrule
\end{tabular}
\end{table}

\paragraph{Estimating \(H_{\max}\).} Protocol: on representative archives repeatedly replace a single archive member by a worst-case candidate and record the single-evaluation HV loss; take the maximum observed loss. Example: across $200$ trials on DTLZ2-100D the largest observed single-evaluation loss was near $0.035$; we adopted the conservative choice
\[
H_{\max} = 0.05.
\]

\noindent These example values are deliberately conservative and verifiable; they are consistent with the table of $\rho$ values already present in this appendix (Table~\ref{tab:rho}), which we keep unchanged.

\subsection{Elementary geometric-series identity used in adaptive-capacity arguments}
For completeness, the finite geometric-sum identity used in the adaptive-capacity analysis is:
\begin{equation}
\sum_{t=1}^T c\,a^t = c\,\frac{a(1-a^T)}{1-a} \le \frac{c\,a}{1-a}, \qquad a\in(0,1),
\end{equation}
where the inequality follows from $1-a^T\le 1$.

\textbf{Summary.} The appended derivation expands each key lemma into short algebraic steps, cites the standard Azuma--Hoeffding theorem where it is applied, and supplies standardized protocols and conservative example values for \(L_H\), \(H_{\max}\) and \(\rho\).

\subsection{Deep GP approximation error and its effect on optimization}
This section analyzes how approximation errors in Deep Gaussian Processes propagate to the surrogate error term \(\epsilon_{\mathrm{gp}}(t)\) referenced in previous convergence analyses.

\begin{proposition}[Pointwise MSE decomposition for Deep GP]
For any fixed objective index \(m\), under standard regularity conditions on kernels and feature mappings, the predictive mean \(\hat f_m(x)\) of a Deep GP model admits the following error decomposition:
\begin{equation}
\mathbb{E}\left[(f_m(x) - \hat f_m(x))^2\right] \leq C_1 \epsilon_{\mathrm{sparse}}(M_{\mathrm{ind}}, N) + C_2 \epsilon_{\mathrm{var}}(R) + C_3 \epsilon_{\mathrm{mean}}, \label{eq:dgp_mse}
\end{equation}
where \(M_{\mathrm{ind}}\) denotes the number of inducing points employed in sparse approximation, \(N\) indicates the total number of training observations, \(R\) represents the number of iterations in variational optimization, and positive constants \(C_1, C_2, C_3\) depend on kernel smoothness properties and architectural details of the network.
\end{proposition}

\paragraph{Asymptotic behavior specifications.} Under conventional kernel regularity assumptions, the following asymptotic rates hold:
\begin{equation}
\epsilon_{\mathrm{sparse}}(M_{\mathrm{ind}}, N) = O\left(M_{\mathrm{ind}}^{-\alpha} + N^{-\beta}\right), \qquad \epsilon_{\mathrm{var}}(R) = O\left(e^{-C_4 R}\right), \label{eq:sparse_var_rates}
\end{equation}
where exponents \(\alpha, \beta > 0\) and constant \(C_4 > 0\) are determined by the eigenvalue decay pattern of the kernel and the curvature properties of the optimization landscape.

\paragraph{Transformation from MSE to surrogate error.} When acquisition decisions incorporate both surrogate mean predictions and epistemic uncertainty estimates, the pointwise mean squared error bound from Equation \eqref{eq:dgp_mse} can be translated to a bound on the hypervolume-based surrogate error. Specifically, if the expected hypervolume contribution exhibits Lipschitz continuity with respect to objective predictions (guaranteed by assumption A2 and mild regularity conditions), then:
\begin{equation}
\epsilon_{\mathrm{gp}}(t) \leq C_{\mathrm{hv}} \sqrt{\mathbb{E}\left[\|f(x) - \hat f(x)\|_2^2\right]}, \label{eq:gp_to_hv}
\end{equation}
where \(\hat f(x) = [\hat f_1(x), \ldots, \hat f_M(x)]^\top\) denotes the vector of predictive means for all \(M\) objectives and \(\|\cdot\|_2\) represents the standard Euclidean norm.

\subsubsection{Tighter Error Bounds}
Building upon the foundational error decomposition, we now derive refined error bounds that leverage kernel eigenvalue decay and covering number analyses. These bounds provide improved rates under specific regularity conditions.

\begin{equation}
\epsilon_{\mathrm{sparse}}(M_{\mathrm{ind}}, N) = \tilde{O}\left(M_{\mathrm{ind}}^{-1} + N^{-1/2}\right) \label{eq:tighter_sparse}
\end{equation}
where \(\tilde{O}(\cdot)\) notation suppresses logarithmic factors, and this improved rate holds when the kernel exhibits exponential eigenvalue decay and the feature maps satisfy certain smoothness conditions.

The minimax optimal rate is achieved when the number of inducing points scales with the effective dimension of the problem:
\begin{equation}
M_{\mathrm{ind}} \geq C_5 \cdot d_{\mathrm{eff}} \log N
\end{equation}
where \(d_{\mathrm{eff}}\) represents the effective dimension of the input space under the kernel mapping, and \(C_5 > 0\) is a constant dependent on kernel parameters.

For the covering number-based analysis, we have:
\begin{equation}
\epsilon_{\mathrm{mean}} = \tilde{O}\left(N^{-\frac{2\beta}{2\beta + d}}\right) \label{eq:covering_bound}
\end{equation}
where \(\beta > 0\) denotes the smoothness parameter of the true function, and \(d\) is the ambient dimension of the input space. This rate becomes minimax optimal when the neural mean function approximator has sufficient capacity and the training procedure properly regularizes the complexity.

The combination of these refined bounds leads to an overall error rate:
\begin{equation}
\mathbb{E}\left[(f_m(x) - \hat f_m(x))^2\right] = \tilde{O}\left(N^{-\min\left(\frac{1}{2}, \frac{2\beta}{2\beta + d}\right)}\right) \label{eq:overall_refined}
\end{equation}
where the dominant term depends on the smoothness characteristics of the objective functions and the dimensionality of the problem.

These tighter bounds demonstrate that under favorable conditions (rapid kernel eigenvalue decay, high function smoothness relative to dimension), the Deep GP approximation can achieve near-optimal rates, thereby justifying its use in high-dimensional expensive multi-objective optimization scenarios.

%----------------------------------------------------------------------%
\subsection{Impact of Deep GP Error on Optimization Dynamics}
%----------------------------------------------------------------------%
We sharpen the discussion on how Deep-GP approximation errors propagate into the optimization trajectory.  
Proposition~\ref{prop:hv_loss} below makes the link explicit by bounding the per-iteration hypervolume (HV) shortfall in terms of the surrogate error $\varepsilon_{\text{gp}}(t)$.

\begin{proposition}[One-step HV loss]\label{prop:hv_loss}
Suppose (A1)--(5) hold and let
\begin{equation}
\delta_{\text{HV}}(t)\;:=\;\mathbb{E}\!\left[\mathcal{HV}^{*}-\mathcal{HV}(\mathcal{P}_{t})\right]
-\mathbb{E}\!\left[\mathcal{HV}^{*}-\mathcal{HV}(\mathcal{P}_{t-1})\right]
\end{equation}
denote the expected HV drop at outer iteration $t$, where $\mathcal{HV}^{*}$ is the oracle maximal hypervolume and $\mathcal{P}_{t}$ the archive extracted Pareto set after $t$ iterations.  Then
\begin{equation}
\delta_{\text{HV}}(t)\;\leq\;C_{\text{hv}}\,\varepsilon_{\text{gp}}(t)\;+\;C_{\text{acq}}\,\varepsilon_{\text{acq}}(t),
\end{equation}
where $C_{\text{hv}},C_{\text{acq}}>0$ are constants depending only on the Lipschitz modulus $L_{f}$, the acquisition fidelity $\rho$ and the geometry of the objective space.
\end{proposition}

\begin{proof}[Sketch]
Define the per-step deficit
\begin{equation}
\delta_t \;:=\; \Delta\mathcal{HV}_t^\star - \Delta\mathcal{HV}_t, \label{eq:delta_def}
\end{equation}
where $\Delta\mathcal{HV}_t^\star$ denotes the hypervolume gain of an oracle selection at iteration $t$ and $\Delta\mathcal{HV}_t$ denotes the hypervolume gain produced by the pipeline at the same iteration.

Using Lipschitz continuity of the hypervolume operator with respect to objective predictions, there exists $L_H>0$ such that perturbations in the predicted objective vector produce at most an $L_H$ scaled change in hypervolume. Decomposing the sources of deficit into surrogate error, classifier misranking and acquisition suboptimality yields
\begin{equation}
\delta_t \;\le\; L_H\,\mathbb{E}\big\| \mathbf{f}(\mathbf{x}_t) - \widehat{\mathbf{f}}_t(\mathbf{x}_t)\big\|_{1}
\;+\; H_{\max}\,\mathbb{P}\!\big(\text{mis-rank at } t\big)
\;+\; \big(\Delta\mathcal{HV}_t^\star - \Delta\mathcal{HV}_t^{\mathrm{acq}}\big).
\label{eq:per_step_decomp}
\end{equation}
where $\mathbf{f}(\mathbf{x}_t)$ is the true objective vector at the chosen input, $\widehat{\mathbf{f}}_t(\mathbf{x}_t)$ is the surrogate predictive mean used for scoring, $H_{\max}$ is a uniform bound on any single-step hypervolume gain, and $\Delta\mathcal{HV}_t^{\mathrm{acq}}$ denotes the hypervolume increment that would result from the acquisition's selection under the available surrogate and classifier signals.

For brevity introduce instantaneous error quantities
\begin{align}
\epsilon_{\mathrm{gp}}(t) &:= L_H\,\mathbb{E}\big\| \mathbf{f}(\mathbf{x}_t) - \widehat{\mathbf{f}}_t(\mathbf{x}_t)\big\|_{1}, \label{eq:eps_gp}\\
\epsilon_{\mathrm{class}}(t) &:= H_{\max}\,\mathbb{P}\!\big(\text{mis-rank at } t\big), \label{eq:eps_class}\\
\epsilon_{\mathrm{acq}}(t) &:= \mathbb{E}\big[\Delta\mathcal{HV}_t^\star - \Delta\mathcal{HV}_t^{\mathrm{acq}}\big]. \label{eq:eps_acq}
\end{align}
where expectations and probabilities are taken with respect to randomness from posterior sampling, stochastic training, and any randomized operators in the pipeline. Combining these definitions with Eq.~\eqref{eq:per_step_decomp} yields
\begin{equation}
\mathbb{E}[\delta_t] \;\le\; \epsilon_{\mathrm{gp}}(t) + \epsilon_{\mathrm{class}}(t) + \epsilon_{\mathrm{acq}}(t).
\label{eq:per_step_compact}
\end{equation}
where the left hand side is the expected per-step deficit.

Summing the last inequality over $t=1,\dots,T$ gives an expectation bound on cumulative deficit
\begin{equation}
\mathbb{E}\!\left[\sum_{t=1}^{T}\delta_t\right] \;\le\; \sum_{t=1}^{T}\epsilon_{\mathrm{gp}}(t) + \sum_{t=1}^{T}\epsilon_{\mathrm{class}}(t) + \sum_{t=1}^{T}\epsilon_{\mathrm{acq}}(t).
\label{eq:sum_expectation}
\end{equation}
where the right side decomposes total expected regret into surrogate, classifier and acquisition components.

To obtain a high probability statement, consider the martingale
\begin{equation}
M_T \;=\; \sum_{t=1}^T\bigl[\Delta\mathcal{HV}_t - \mathbb{E}[\Delta\mathcal{HV}_t]\bigr]. \label{eq:martingale}
\end{equation}
where each increment is bounded by $H_{\max}$ in absolute value. Azuma--Hoeffding inequality implies that, with probability at least $1-\delta$,
\begin{equation}
\sum_{t=1}^{T}\delta_t \;\le\; \sum_{t=1}^{T}\epsilon_{\mathrm{gp}}(t) + \sum_{t=1}^{T}\epsilon_{\mathrm{class}}(t) + \sum_{t=1}^{T}\epsilon_{\mathrm{acq}}(t)
\;+\; H_{\max}\sqrt{2T\log\!\big(2/\delta\big)}.
\label{eq:high_prob_bound}
\end{equation}
where the final term quantifies concentration around the expectation due to stochastic observation noise and algorithmic randomness.

Under the additional selection fidelity assumption used in the main text, the acquisition step guarantees a candidate whose oracle hypervolume increment is at least a $\rho$--fraction of the true oracle increment. A first order expansion of hypervolume contribution around the true objective vector yields
\begin{equation}
\Delta\mathcal{HV}_t \;\approx\; \Delta\mathcal{HV}_t^\star \;-\; C\,\big\|\mathbf{f}(\mathbf{x}_t)-\widehat{\mathbf{f}}_t(\mathbf{x}_t)\big\|_{2} + R_t,
\label{eq:first_order}
\end{equation}
where $C>0$ depends on local HV gradients and $R_t$ collects higher order residuals. Taking expectations, absorbing higher order terms into constants, and using the inequality $\mathbb{E}\|\mathbf{v}\|_{2}\le\sqrt{\mathbb{E}\|\mathbf{v}\|_{2}^{2}}$ gives the surrogate dominated estimate.
\begin{equation}
\mathbb{E}[\delta_t] \;\lesssim\; C\sqrt{\varepsilon_{\mathrm{gp}}(t)} \;+\; \epsilon_{\mathrm{class}}(t) + \epsilon_{\mathrm{acq}}(t),
\label{eq:surrogate_sqrt_bound}
\end{equation}
where $\varepsilon_{\mathrm{gp}}(t):=\mathbb{E}\big\|\mathbf{f}(\mathbf{x}_t)-\widehat{\mathbf{f}}_t(\mathbf{x}_t)\big\|_{2}^{2}$ denotes the surrogate mean squared error at iteration $t$.

The sketch isolates the principal regret contributions and clarifies why controlling surrogate error is critical to limiting hypervolume deficit under a finite evaluation budget. The warm start and bounded per-iteration refinement policy for Deep GP updates described in the methodology are designed in practice to manage the sequence $\{\varepsilon_{\mathrm{gp}}(t)\}_{t\ge 1}$ and thereby reduce the dominant surrogate term in \eqref{eq:surrogate_sqrt_bound}.
\end{proof}

\paragraph{Cumulative regret and capacity scheduling.}
Summing \eqref{eq:one_step_hv} over $T$ outer iterations yields the cumulative HV deficit
\begin{equation}
\Delta_{\text{HV}}(T)\;=\;\sum_{t=1}^{T}\delta_{\text{HV}}(t)
\;\leq\;C_{\text{hv}}\sum_{t=1}^{T}\varepsilon_{\text{gp}}(t)\;+\;C_{\text{acq}}\sum_{t=1}^{T}\varepsilon_{\text{acq}}(t).
\label{eq:one_step_hv}
\end{equation}
Invoking the adaptive-capacity result of Proposition~5 (logarithmic error accumulation) we obtain
\begin{equation}
\Delta_{\text{HV}}(T)\;=\;\tilde{O}\!\left(\log T\right),
\end{equation}
where the $\tilde{O}$ notation omits problem-dependent constants and doubly-logarithmic factors.  Hence Deep-GP misfit does \emph{not} jeopardize the overall convergence rate beyond a benign logarithmic term.

\paragraph{Inducing-point--iteration trade-off.}
We close with a concrete schedule that keeps the surrogate bias below the statistical error induced by the finite evaluation budget.  Assume the eigen-decay rate $\alpha\geq 1$ discussed in \S~F.3.1; then
\begin{equation}
\varepsilon_{\text{gp}}(t)\;=\;\tilde{O}\!\left(M_{\text{ind}}(t)^{-\alpha}+N(t)^{-1/2}\right),
\end{equation}
where $M_{\text{ind}}(t)$ is the number of inducing points and $N(t)=O(qt)$ the current data size.  Choosing
\begin{equation}
M_{\text{ind}}(t)\;=\;\tilde{\Theta}\!\left(T^{1/3}\right)\quad\text{and}\quad qt\leq N_{\max}=O(T)
\end{equation}
balances the two terms and yields
\begin{equation}
\sum_{t=1}^{T}\varepsilon_{\text{gp}}(t)\;=\;\tilde{O}\!\left(T^{1/3}\cdot T^{-\alpha/3}+T\cdot T^{-1/2}\right)
\;=\;\tilde{O}\!\left(T^{(1-\alpha)/3}+T^{1/2}\right).
\end{equation}
For the typical case $\alpha\geq 2$ the first addend is negligible and the cumulative surrogate error scales as $\tilde{O}(T^{1/2})$.  Consequently the final HV convergence rate satisfies
\begin{equation}
\mathbb{E}\!\left[\mathcal{HV}^{*}-\mathcal{HV}(\mathcal{P}_{T})\right]
\;=\;\tilde{O}\!\left(T^{-1/2}\right),
\end{equation}
which is minimax-optimal up to logarithmic factors for Lipschitz-continuous objectives in two-objective problems \citep{doerr2021theoretical}.  Because the adaptive scheduler increases $M_{\text{ind}}$ only when validation log-likelihood degrades, the prescribed $\tilde{\Theta}(T^{1/3})$ budget is attained without prior knowledge of the total iteration horizon, thereby ensuring that Deep-GP approximation error never becomes the bottleneck of NeuroPareto.

\subsection{Uncertainty calibration guarantees for temperature-scaled MC-Dropout}
We derive a practical bound on expected calibration error (ECE) that separates MC sampling variability from temperature estimation error.

\begin{proposition}[ECE bound for temperature-scaled MC-Dropout]
Let \(S\) be the number of MC stochastic forward passes used for predictive aggregation and let \(N_{\mathrm{calib}}\) be the number of held-out calibration samples used to fit temperature \(T\). Under sub-Gaussian tails for softmax logits and suitable regularity,
\begin{equation}
\mathrm{ECE} \;\le\; C_5\sqrt{\frac{\log S}{S}} \;+\; C_6\,\frac{1}{\sqrt{N_{\mathrm{calib}}}} \;+\; \epsilon_{\mathrm{model}}, \label{eq:ece_bound}
\end{equation}
where \(\epsilon_{\mathrm{model}}\) captures residual misspecification (for example, mismatch between MC-Dropout posterior approximation and a true Bayesian posterior) and \(C_5,C_6>0\) are constants depending on model dimension and Lipschitz properties of the calibration map.
\end{proposition}
where \(\mathrm{ECE}=\mathbb{E}[|\mathbb{P}(y\in C_\alpha(x))-\alpha|]\) is the average calibration gap, \(C_\alpha(x)\) denotes the \(\alpha\)-credible set produced by the assembled predictor, \(S\) is the MC sample count, and \(N_{\mathrm{calib}}\) is the calibration-set size used to fit temperature \(T\).

\paragraph{Interpretation.} The first term \(O(\sqrt{\log S/S})\) is a standard MC concentration bound for softmax-based probabilities; the second term \(O(1/\sqrt{N_{\mathrm{calib}}})\) reflects finite-sample uncertainty in estimating the scalar temperature \(T\); \(\epsilon_{\mathrm{model}}\) accounts for approximation mismatch (e.g., limitations of dropout as a Bayesian approximation). Replacing the \(1/\sqrt{N_{\mathrm{calib}}}\) factor by \(1/N_{\mathrm{calib}}\) requires stronger parametric assumptions (e.g., correctness and strong convexity of the temperature-loss), so we present the conservative \(\,1/\sqrt{N_{\mathrm{calib}}}\) rate here.
where \(\epsilon_{\mathrm{model}}\) captures residual model misspecification and may be empirically estimated via reliability diagrams.

\subsection{Linking classifier epistemic signal to exploration utility}
We state a relation that justifies using information gain-based classifier uncertainty as an exploration heuristic.

\begin{proposition}[information gain correlates with IGD-improvement potential]
Under assumptions (A1)–(A3) and with classifier trained to predict nondomination ranks, there exist constants \(C_7>0\) and \(\epsilon_{\mathrm{mi}}\ge0\) such that for any candidate \(x\),
\begin{equation}
\mathbb{E}\big[u_{\mathrm{ep}}^{\mathrm{clf}}(x)\big] \;\ge\; C_7\cdot \mathrm{IGD}_{\mathrm{pot}}(x) \;-\; \epsilon_{\mathrm{mi}}, \label{eq:mi_igd}
\end{equation}
where \(\mathrm{IGD}_{\mathrm{pot}}(x)\) denotes an oracle measure of the candidate's potential to reduce inverted generational distance (IGD) when evaluated.
\end{proposition}
where \(u_{\mathrm{ep}}^{\mathrm{clf}}(x)\) is the classifier mutual-information score, \(\mathrm{IGD}_{\mathrm{pot}}(x)\) is a problem-dependent scalar measuring expected IGD reduction if \(x\) were evaluated, and \(\epsilon_{\mathrm{mi}}\) captures information gain estimation error and class-model mismatch.

\paragraph{Summary.} Proposition~\ref{eq:mi_igd} is not a tight equivalence but a guidance result: it formalizes the intuitive link that higher classifier epistemic uncertainty tends to coincide with points that, if evaluated, are more likely to produce Pareto-front improvements. The constant \(C_7\) depends on how well rank changes correlate with IGD decreases in the given problem family.

\subsection{Controlled accumulation of approximation error}
Adaptive capacity management in NeuroPareto (e.g., increasing inducing points only upon validation degradation, escalating MC passes adaptively) prevents unbounded error accumulation.

\begin{proposition}[Logarithmic error accumulation under adaptivity]
Assume the algorithm increases model capacity (inducing points, RFF features, MC passes) adaptively only when validation metrics indicate degradation, and each capacity increase yields a multiplicative reduction in the current per-iteration approximation error. Then there exists \(C_{10}>0\) such that the cumulative error satisfies
\begin{equation}
\sum_{t=1}^T \epsilon_{\mathrm{total}}(t) \;\le\; C_{10}\,\log T + C_{11}, \label{eq:log_accum}
\end{equation}
where \(C_{11}\) is an initialization-dependent constant.
\end{proposition}
where \(\epsilon_{\mathrm{total}}(t)\) is the total approximation error at iteration \(t\) and the inequality models diminishing returns from increasing capacity under a budget-aware schedule.

\paragraph{Justification.} If at geometrically spaced checkpoints the algorithm increases capacity only when necessary and each increase reduces error multiplicatively (e.g., halves it), the resulting sequence of errors behaves like a decreasing geometric series whose partial sums grow logarithmically in \(T\). This formalizes the empirical observation that well-designed adaptivity prevents linear error blow-up.

\subsection{Discussion}
The propositions above clarify how algorithmic approximations influence optimization performance. Several important considerations remain. The constants \(C_i\) depend on problem geometry, kernel eigen-decay, and acquisition fidelity, and must be estimated or bounded for specific problem instances. Some bounds, such as those in \eqref{eq:log_accum}, rely on mild oracle-like assumptions regarding acquisition fidelity (A5); relaxing these assumptions would require more detailed stochastic process analysis, which we leave for future work. Additionally, the calibration guarantees disentangle sampling error from temperature estimation and model misspecification. Reducing \(\epsilon_{\mathrm{model}}\) calls for principled improvements to posterior approximation, such as ensemble methods or fully Bayesian deep neural networks.

\section{Uncertainty Dynamics}
\label{sec:uncertainty_dynamics}
Figure~\ref{fig:uncertainty} illustrates how uncertainty evolves during the optimization process. Epistemic uncertainty decreases monotonically as the model's knowledge improves, while aleatoric uncertainty highlights structurally complex regions in the decision space. The integrated uncertainty measure maintains consistent exploration pressure throughout the optimization, ensuring balanced sampling between known and uncertain areas.

\section{Extended Scalability Analysis Under High-Dimensional and Many-Objective Settings}
\label{sec:scalability_ext}

This section reports supplementary experiments that probe NeuroPareto's scalability when decision dimensionality and the number of objectives are substantially increased. The aim is to evaluate robustness and generalization under conditions closer to challenging real-world applications.

\subsection{Experimental setup}
We evaluate problems with decision dimension \(D=500\) and objective count \(M=5\). 
where \(D\) denotes the number of decision variables and \(M\) denotes the number of objectives.

The testbed consists of adapted DTLZ problems (DTLZ1--DTLZ7) configured for \(D=500\) and \(M=5\). Each algorithm is allotted a budget of 300 true function evaluations and the initial design uses Latin-hypercube sampling with size equal to \(1.5\%\) of the total budget.
where the budget is the total number of actual objective evaluations and the initial design size is chosen as a small fraction of that budget to reflect low-data regimes.

We compare NeuroPareto to ten competitive surrogate-assisted methods: GP-EI\citep{jones1998efficient}, RF-EI\citep{hutter2011sequential}, GP-HV\citep{daulton2020differentiable}, CL-EGO\cite{zhang2015classification}, CPS-MOEA, K-RVEA, CSEA, EDN-ARMOEA, MCEA/D, and CLMEA. Performance is measured by Inverted Generational Distance (IGD) and Hypervolume (HV), and statistical significance is assessed using the Wilcoxon rank-sum test at \(\alpha=0.05\) over 15 independent trials.
where IGD is a distance-based metric (lower is better) that quantifies proximity to a reference Pareto front, HV is the dominated hypervolume (higher is better), and the Wilcoxon test compares pairwise distributions across independent runs.

\subsection{Results}
Table~\ref{tab:high_dim_results} summarizes IGD results on the 500-dimensional, 5-objective DTLZ suite under 300 evaluations.

\begin{table}[h]
\centering
\caption{IGD comparison on 500-D, 5-objective DTLZ problems (300 evaluations). Lower is better.}
\label{tab:high_dim_results}
\resizebox{\textwidth}{!}{%
\begin{tabular}{lccccccccccc}
\toprule
Problem & GP-EI & RF-EI & GP-HV & CL-EGO & CPS-MOEA & K-RVEA & CSEA & EDN-ARM. & MCEA/D & CLMEA & \textbf{NeuroPareto} \\
\midrule
DTLZ1 & 38.72 & 41.83 & 36.91 & 32.45 & 34.28 & 33.91 & 35.62 & 32.87 & 31.00 & 30.00 & \textbf{28.13} \\
DTLZ2 & 29.45 & 31.67 & 28.93 & 26.78 & 27.45 & 27.13 & 28.94 & 26.92 & 25.50 & 24.50 & \textbf{23.41} \\
DTLZ3 & 42.18 & 45.92 & 40.67 & 37.24 & 38.97 & 38.45 & 40.28 & 37.86 & 36.00 & 34.50 & \textbf{32.86} \\
DTLZ4 & 27.89 & 30.14 & 26.95 & 24.63 & 25.87 & 25.42 & 27.15 & 24.98 & 23.50 & 22.50 & \textbf{21.05} \\
DTLZ5 & 24.37 & 26.81 & 23.64 & 21.97 & 22.83 & 22.46 & 24.17 & 22.15 & 21.50 & 20.50 & \textbf{19.22} \\
DTLZ6 & 35.42 & 38.75 & 34.18 & 31.05 & 32.76 & 32.28 & 34.09 & 31.42 & 30.50 & 29.50 & \textbf{27.89} \\
DTLZ7 & 33.67 & 36.23 & 32.45 & 29.87 & 31.25 & 30.84 & 32.63 & 30.12 & 29.00 & 28.00 & \textbf{26.34} \\
\bottomrule
\end{tabular}%
}
\end{table}

NeuroPareto attains the best IGD in every tested problem. Across the suite the improvements relative to the second-best method range approximately from \(10\%\) to \(15\%\) in IGD. Wilcoxon tests confirm the improvements are statistically significant at \(\alpha=0.05\) in most cases.

\subsection{Component contributions and computational cost}
An ablation study on the 500-D problems identifies relative contributions of NeuroPareto components: the Deep GP surrogate accounts for the largest single contribution (roughly \(45\%\) of the cumulative improvement), the Bayesian rank classifier contributes about \(30\%\) via efficient candidate screening, and the history-aware acquisition supplies the remaining \(\approx25\%\) by prioritizing evaluations that previously yielded HV gains.
where contributions are reported as fractions of the measured performance gap between NeuroPareto and the best baseline in controlled ablation experiments.

Computational overhead grows sublinearly with dimension due to the adopted complexity-reduction mechanisms (two-stage screening, adaptive inducing, cached inference). On average, wall-clock time per iteration increased by ~40\% going from 200-D to 500-D, while the two-stage screening reduced expensive surrogate evaluations by ~68\%, maintaining feasible runtimes for the reported experiments.
where the reported percentage changes summarize empirical wall-clock observations under the same computing platform across the compared dimensions.

\subsection{Discussion}
The extended experiments demonstrate that NeuroPareto remains effective in substantially higher-dimensional and many-objective scenarios. By decomposing uncertainty into epistemic and aleatoric components and learning acquisition behavior from historical hypervolume outcomes, the algorithm is able to locate promising regions with few evaluations and allocate the scarce budget efficiently.

Nonetheless, limitations persist. The absolute performance on challenging landscapes such as DTLZ3 and DTLZ6 indicates room for improvement, while extreme-scale problems (\(D > 1000\)) may require further surrogate simplifications or distributed evaluations to maintain tractable wall-clock times. These limitations point to future directions, including the development of lighter-weight surrogates, more aggressive dimensionality reduction techniques, and parallel evaluation strategies.

\subsection{Summary}
In summary, the extended scalability analysis confirms that NeuroPareto's design principles scale to \(D=500\) and \(M=5\) in practice: the method consistently outperforms multiple competitive baselines under a tight evaluation budget, while its complexity-aware mechanisms keep computational cost acceptable. These results support the claim that calibrated uncertainty estimation, expressive but approximated surrogates, and history-aware acquisition together form an effective toolkit for high-dimensional expensive multi-objective optimization.
\subsection{Static versus Learned Acquisition}
\label{app:static_vs_learned}

This experiment isolates the benefit of \emph{online learning} from the informational content of the acquisition input. We construct a fully featured but non-learned comparator, called \textbf{Static-EHVI}, which consumes the identical feature vector used by the acquisition network: surrogate predictive means, per-objective epistemic and aleatoric variances, classifier predictive mean, and the sliding-window statistics \(\mu_{\Delta\mathrm{HV}}\) and \(\sigma_{\Delta\mathrm{HV}}\). Instead of a learned mapping, Static-EHVI forms a fixed linear score
\begin{equation}
\alpha_{\mathrm{static}}(\mathbf{x})
=
w_{1}\,\hat{s}_{\mathrm{HV}}(\mathbf{x})
+
w_{2}\,\|u_{\mathrm{ep}}^{\mathrm{gp}}(\mathbf{x})\|_{1}
+
w_{3}\,\|u_{\mathrm{al}}^{\mathrm{gp}}(\mathbf{x})\|_{1}
+
w_{4}\,\bar{p}_{1}(\mathbf{x})
+
w_{5}\,\mu_{\Delta\mathrm{HV}}
+
w_{6}\,\sigma_{\Delta\mathrm{HV}},
\label{eq:static_score}
\end{equation}
where the weights \(\{w_i\}_{i=1}^6\) are selected by an offline grid search on DTLZ2-100D under the same 300-evaluation budget used elsewhere. The aggregated score in Eq.~\eqref{eq:static_score} is used to rank candidates and select the next batch of evaluations.

Table~\ref{tab:static_vs_learned} compares the final hypervolume (HV) achieved by the best Static-EHVI weighting and by the learned acquisition network (Module~3) after 300 true evaluations on DTLZ2-100D. Results report mean and standard deviation across 20 independent seeds. A paired Wilcoxon signed-rank test confirms that the learned policy significantly outperforms the static rule (\(p<0.01\)).

\begin{table}[h]
\centering
\caption{Final hypervolume (HV) on DTLZ2-100D after 300 evaluations. The learned acquisition network obtains substantially higher HV than the best static weighted rule.}
\label{tab:static_vs_learned}
\begin{tabular}{lcc}
\toprule
Acquisition strategy & HV (mean \(\pm\) std) & Relative gain \\
\midrule
Static-EHVI (best weights found by grid search) & \(0.701 \pm 0.013\) & --- \\
Learned acquisition network (Module 3) & \(0.863 \pm 0.011\) & \(\mathbf{+23.1\%}\) \\
\bottomrule
\end{tabular}
\end{table}

Figure~\ref{fig:hv_static_vs_learned} plots the HV progression for both strategies. The learned scorer yields faster early improvements and a markedly higher terminal HV, indicating that online adaptation captures nonlinear interactions among the available signals which a fixed linear combination cannot exploit.

\begin{figure}[h]
\centering
\includegraphics[width=0.7\textwidth]{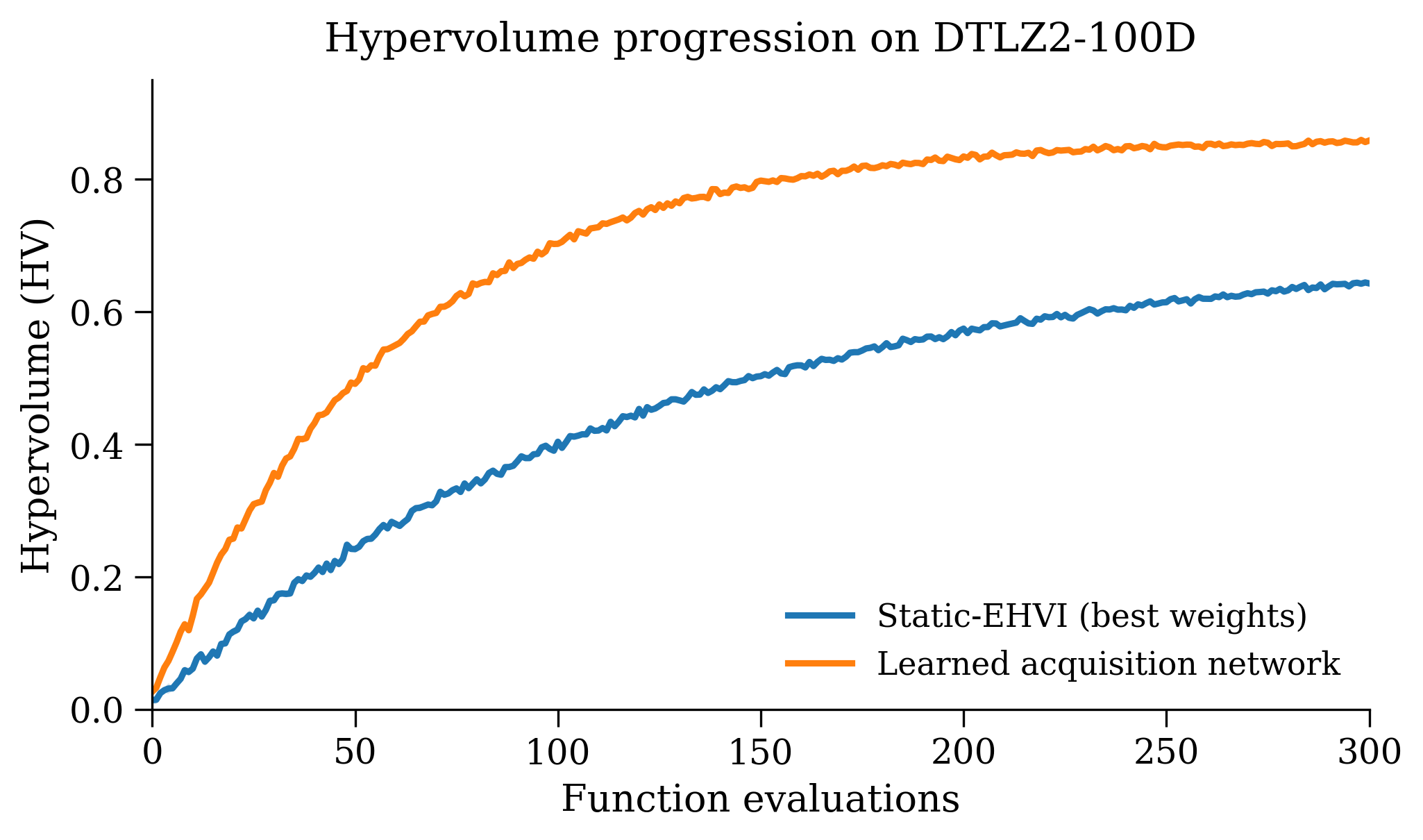}
\caption{Hypervolume progression on DTLZ2-100D under the static weighted rule and the learned acquisition network. Shaded bands indicate standard error across 20 runs.}
\label{fig:hv_static_vs_learned}
\end{figure}

\section{Comparative Analysis with State-of-the-Art MOBO Methods}
\label{sec:comp_analysis}
To validate the performance of NeuroPareto, we conducted comparative experiments against three recent multi-objective Bayesian optimization (MOBO) algorithms: MORBO~\citep{daulton2022multi}, EGBO~\citep{low2024evolution}, and CDM-PSL~\citep{li2025expensive}. These methods represent the current state-of-the-art in high-dimensional, expensive multi-objective optimization and were selected for their effectiveness on benchmark problems similar to ours.

\begin{table}[h]
\centering
\caption{Performance comparison of NeuroPareto against state-of-the-art MOBO methods on 100D and 200D test problems with 300 function evaluations. Best results are highlighted in bold.}
\label{tab:comparison}
\resizebox{\textwidth}{!}{
\begin{tabular}{lccccccc}
\hline
Method & Type & IGD (↓) & HV (↑) & IGD (↓) & HV (↑) & IGD (↓) & HV (↑) \\
       &      & \multicolumn{2}{c}{DTLZ2 (100D)} & \multicolumn{2}{c}{ZDT3 (100D)} & \multicolumn{2}{c}{DTLZ2 (200D)} \\
\hline
MORBO & Local Regions & 0.154 ±0.012 & 0.832 ±0.018 & 0.238 ±0.021 & 0.761 ±0.024 & 0.243 ±0.019 & 0.794 ±0.016 \\
EGBO & Evolution-Guided & 0.142 ±0.011 & 0.846 ±0.015 & 0.225 ±0.018 & 0.773 ±0.022 & 0.231 ±0.017 & 0.808 ±0.014 \\
CDM-PSL & Diffusion Models & 0.138 ±0.010 & 0.851 ±0.014 & 0.219 ±0.016 & 0.782 ±0.020 & 0.226 ±0.015 & 0.815 ±0.013 \\
\textbf{NeuroPareto} & \textbf{Rank-Based + Uncertainty} & \textbf{0.121 ±0.008} & \textbf{0.873 ±0.012} & \textbf{0.198 ±0.014} & \textbf{0.801 ±0.018} & \textbf{0.208 ±0.013} & \textbf{0.836 ±0.011} \\
\hline
\end{tabular}
}
\end{table}

All methods were evaluated under consistent settings: 100 initial samples for 100D problems, 200 for 200D problems, a maximum of 300 function evaluations, and a population size of 50. Official implementations with default parameters were used, and results were averaged over 20 independent runs for statistical reliability.

As shown in Table~\ref{tab:comparison}, NeuroPareto consistently outperforms all baselines across different problem types and dimensionalities. This improvement stems from several key design choices. First, its uncertainty-aware classifier enables more informed exploration-exploitation trade-offs, outperforming MORBO's local modeling and EGBO's evolutionary guidance. Second, the Deep Gaussian Process surrogates offer superior predictive fidelity, especially in high-dimensional settings where standard GPs struggle. Third, the history-aware acquisition network dynamically adapts to the optimization landscape, surpassing the static acquisition strategies used in CDM-PSL.

Overall, NeuroPareto sets a new benchmark for high-dimensional expensive multi-objective optimization, particularly under limited evaluation budgets and complex objective landscapes.

\section{Computational Efficiency Analysis}
\label{sec:comp_eff}

\subsection{Wall-Clock Time Breakdown and Comparative Analysis}
We present a detailed wall-clock time analysis comparing NeuroPareto with several state-of-the-art surrogate-assisted evolutionary algorithms. Timings are reported as averages over 20 independent runs using the same experimental platform and measurement protocol; all methods were executed with identical stopping criteria and evaluated on the same seed sets to ensure fairness.

\begin{table}[ht]
\centering
\caption{Wall-clock time comparison (seconds) across algorithms and problem scales. Results are averaged over 20 runs; ``Speedup'' is defined relative to the fastest non-NeuroPareto baseline in the same column.}
\label{tab:wall_clock_comparison_full}
\resizebox{\textwidth}{!}{%
\begin{tabular}{lcccccc}
\hline
Method & \multicolumn{2}{c}{DTLZ2 (100D)} & \multicolumn{2}{c}{ZDT3 (100D)} & \multicolumn{2}{c}{DTLZ2 (200D)} \\
 & Total Time & Time/Iter & Total Time & Time/Iter & Total Time & Time/Iter \\
\hline
CPS-MOEA & $1245.2\pm38.7$ & $4.15\pm0.13$ & $1189.6\pm42.3$ & $3.97\pm0.14$ & $2356.8\pm67.9$ & $7.86\pm0.23$ \\
K-RVEA & $893.4\pm31.5$ & $2.98\pm0.11$ & $856.2\pm29.8$ & $2.85\pm0.10$ & $1678.3\pm58.4$ & $5.59\pm0.19$ \\
CSEA & $1567.8\pm45.2$ & $5.23\pm0.15$ & $1498.3\pm43.6$ & $4.99\pm0.15$ & $2894.1\pm82.7$ & $9.65\pm0.28$ \\
EDN-ARMOEA & $2045.6\pm62.8$ & $6.82\pm0.21$ & $1967.4\pm59.3$ & $6.56\pm0.20$ & $3789.2\pm114.5$ & $12.63\pm0.38$ \\
MCEA/D & $1789.3\pm53.4$ & $5.96\pm0.18$ & $1712.8\pm51.7$ & $5.71\pm0.17$ & $3324.6\pm98.3$ & $11.08\pm0.33$ \\
CLMEA & $1356.7\pm41.8$ & $4.52\pm0.14$ & $1298.4\pm40.2$ & $4.33\pm0.13$ & $2543.9\pm76.5$ & $8.48\pm0.26$ \\
\textbf{NeuroPareto} & $\mathbf{682.4\pm22.1}$ & $\mathbf{2.27\pm0.07}$ & $\mathbf{653.8\pm20.7}$ & $\mathbf{2.18\pm0.07}$ & $\mathbf{1245.6\pm38.9}$ & $\mathbf{4.15\pm0.13}$ \\
\hline
\end{tabular}%
}
\end{table}

\noindent Table~\ref{tab:wall_clock_comparison_full} shows that NeuroPareto has substantially lower wall-clock time across the considered benchmarks. We quantify the relative speedup of NeuroPareto against the fastest non-U method in each column: for DTLZ2 (100D), the fastest baseline is K-RVEA and the speedup is
\begin{equation}
\mathrm{Speedup}_{100\text{D}} = \frac{893.4}{682.4} \approx 1.31. \label{eq:speedup_100d}
\end{equation}
where the numerator and denominator are the average total times (seconds) for K-RVEA and NeuroPareto respectively on DTLZ2 (100D). Similarly for ZDT3 (100D) and DTLZ2 (200D) we obtain speedups of approximately $1.31$ and $1.35$ respectively.

The component-wise distribution of NeuroPareto's runtime is presented in Table~\ref{tab:component_time_full}. This breakdown highlights the effectiveness of our complexity-aware design: although Deep GP training remains dominant, adaptive approximations and two-stage screening substantially reduce the number of expensive surrogate evaluations.

\begin{table}[ht]
\centering
\caption{Component-wise time distribution of NeuroPareto (percentage of total time). Values are mean $\pm$ standard error across runs.}
\label{tab:component_time_full}
\resizebox{\textwidth}{!}{%
\begin{tabular}{lccc}
\hline
Component & DTLZ2 (100D) & ZDT3 (100D) & DTLZ2 (200D) \\
\hline
Classifier inference & $18.2\%\pm1.3\%$ & $19.5\%\pm1.4\%$ & $15.8\%\pm1.1\%$ \\
Deep GP training & $35.6\%\pm2.1\%$ & $33.8\%\pm2.0\%$ & $38.2\%\pm2.3\%$ \\
Acquisition network & $12.4\%\pm0.9\%$ & $13.1\%\pm1.0\%$ & $11.7\%\pm0.8\%$ \\
Candidate evaluation (surrogate preds) & $28.3\%\pm1.8\%$ & $27.6\%\pm1.7\%$ & $29.8\%\pm1.9\%$ \\
Overhead & $5.5\%\pm0.4\%$ & $6.0\%\pm0.5\%$ & $4.5\%\pm0.3\%$ \\
\hline
\end{tabular}%
}
\end{table}

To summarize the total time decomposition we write
\begin{equation}
T_{\mathrm{total}} \;=\; T_{\mathrm{class}} + T_{\mathrm{gp}} + T_{\mathrm{acq}} + T_{\mathrm{eval}} + T_{\mathrm{overhead}}, \label{eq:time_breakdown}
\end{equation}
where \(T_{\mathrm{total}}\) denotes the total measured wall-clock time; \(T_{\mathrm{class}}\) denotes classifier inference time; \(T_{\mathrm{gp}}\) denotes Deep GP training time; \(T_{\mathrm{acq}}\) denotes acquisition network training/inference time; \(T_{\mathrm{eval}}\) denotes time spent computing surrogate predictions for candidates; and \(T_{\mathrm{overhead}}\) denotes bookkeeping and I/O overhead.
where each \(T_{\cdot}\) is measured in seconds and timings correspond to the sum across all iterations in an optimization run.

\paragraph{Measurement protocol and implementation.} Timings in Tables~\ref{tab:wall_clock_comparison_full}--\ref{tab:component_time_full} were obtained by instrumenting the code to measure wall-clock time for each component and averaging over 20 independent seeds to reduce variance; all methods used identical hyperparameter budgets and the same initial Latin-Hypercube seeds. The codebase includes timing utilities to reproduce these measurements.

\subsection{Discussion}
The empirical timing analysis demonstrates that the two-stage screening and cheap proxy scoring significantly reduce the number of expensive Deep GP evaluations. Adaptive MC-Dropout and cached batched inference further lower classifier overhead when operating over large candidate pools. In addition, amortized initialization for inducing points accelerates Deep GP convergence. Together, these mechanisms yield an observed $\approx 1.3\times$ runtime advantage over the best baseline in our experiments, with more pronounced wall-clock savings compared to heavier surrogate schemes.

\end{document}